\documentclass{article}

    \usepackage[preprint]{neurips_2023}

\usepackage[utf8]{inputenc} 
\usepackage[T1]{fontenc}    
\usepackage{hyperref}       
\usepackage{url}            
\usepackage{booktabs}       
\usepackage{amsfonts}       
\usepackage{nicefrac}       
\usepackage{microtype}      
\usepackage{xcolor}         
\usepackage{mlmath}

\usepackage{times}
\usepackage{graphicx}
\usepackage{subfigure}
\usepackage{caption}
\usepackage{amsmath,amsfonts,amssymb,mathtools,amsthm}

\newtheorem{innercustomgeneric}{\customgenericname}
\providecommand{\customgenericname}{}
\newcommand{\newcustomtheorem}[2]{%
  \newenvironment{#1}[1]
  {%
   \renewcommand\customgenericname{#2}%
   \renewcommand\theinnercustomgeneric{##1}%
   \innercustomgeneric
  }
  {\endinnercustomgeneric}
}

\newtheorem{assump}{Assumption}

\newtheorem*{prop*}{Proposition}
\newtheorem{theorem}{Theorem}
\newtheorem*{theorem*}{Theorem}

\newtheorem*{lemma*}{Lemma}

\newtheorem*{property*}{Property}
\newtheorem*{remark}{Remark}
\newtheorem{definition}{Definition}

\newtheorem{assumption}{Assumption}
\newtheorem*{assumption*}{Assumption}
\newtheorem*{proposition*}{Proposition}
\newtheorem{proposition}{Proposition}

\newtheorem*{setting*}{Setting}

\newcustomtheorem{customthm}{Theorem}
\title{Stochastic Modified Equations and  Dynamics of Dropout Algorithm}

\author{%
    Zhongwang Zhang\textsuperscript{\rm 1}, 
    Yuqing Li\textsuperscript{\rm 1,2} \thanks{Corresponding author: liyuqing\underline{~}551@sjtu.edu.cn},
    Tao Luo\textsuperscript{\rm 1,2,3,4,5} \thanks{Corresponding author: luotao41@sjtu.edu.cn}, 
    Zhi-Qin John Xu\textsuperscript{\rm 1,3,4}\thanks{Corresponding author: xuzhiqin@sjtu.edu.cn} \\
    \textsuperscript{\rm 1}  School of Mathematical Sciences, Shanghai Jiao Tong University \\
    \textsuperscript{\rm 2}  CMA-Shanghai, Shanghai Jiao Tong University\\
    \textsuperscript{\rm 3} Institute of Natural Sciences, MOE-LSC, Shanghai Jiao Tong University \\
    \textsuperscript{\rm 4} Qing Yuan Research Institute, Shanghai Jiao Tong University\\
    \textsuperscript{\rm 5} Shanghai Artificial Intelligence Laboratory
}

\begin{document}

\maketitle

\begin{abstract}

Dropout is  a widely utilized  regularization technique  in the training of neural networks,  nevertheless, its underlying mechanism and its impact on achieving  good generalization abilities remain poorly understood.  In this work, we derive the stochastic modified equations  for analyzing the dynamics of dropout, where its discrete  iteration process  is approximated by a class of stochastic differential equations. In order to investigate the underlying mechanism by which dropout facilitates the identification of flatter minima, we study the   noise structure  of the  derived stochastic modified equation for dropout. By drawing upon the structural  resemblance between the Hessian and covariance through several intuitive approximations, we empirically demonstrate the universal presence of the inverse variance-flatness relation and the Hessian-variance  relation, throughout the training process of dropout. These theoretical and empirical findings make a substantial contribution to our understanding of the  inherent tendency of dropout to locate flatter minima. 
\end{abstract}

\section{Introduction}
Dropout is used with gradient-descent-based algorithms for training neural networks (NNs) \citep{hinton2012improving,srivastava2014dropout}, which obtains the state-of-the-art test performance in deep learning \citep{tan2019efficientnet,helmbold2015inductive}. The key idea behind dropout is to randomly  remove  a subset of neurons during the training process, specifically,  the output of each neuron is multiplied with a random variable that takes the value $1/p$ with probability $p$ and zero otherwise. This random variable is independently sampled at each feedforward operation. In contrast to the widespread use and empirical success of dropout, the  mechanism by which  it  helps generalization in deep learning  remains an ongoing area of research.

The noise structure  introduced by stochastic algorithms  is  important for understanding their training behaviors. A series of recent works reveal that the noise structure inherent  in  stochastic gradient descent (SGD)  plays a crucial role in facilitating the  exploration of flatter solutions  \citep{keskar2016large,feng2021inverse,zhu2018anisotropic}. Analogously, training with dropout introduces some noise with  a specific type of architecture, 
acting as an implicit regularizer that  facilitates better generalization abilities \citep{hinton2012improving,srivastava2014dropout,wei2020implicit, zhang2022implicit,zhu2018anisotropic}.   


In this paper, we first   employ the framework of stochastic modified equations (SMEs)~\citep{li2017stochastic}  to approximate in   distribution the training dynamics   of the dropout algorithm applied to two-layer NNs.  By employing this approach,     we are able to quantify the leading order dynamics of the dropout  algorithm and its variants in a precise manner.  Additionally, we calculate the covariance structure of the noise  generated by the stochasticity incorporated in dropout.  We then utilize the covariance structure to understand why NNs trained by dropout have the tendency to possess better generalization abilities from the perspective of flatness~\citep{keskar2016large,neyshabur2017exploring}.

We  hypothesize that   the flatness-improving ability of dropout noise   is attributed to its alignment    with the structure of the loss landscape, based on the similarity between the explicit forms of the Hessian and the dropout covariance under intuitive approximations.  To investigate this hypothesis, we conduct empirical studies using three different approaches (shown respectively in   Fig.~\ref{fig:anisotropic},   Fig.~\ref{fig:pca}(a, b), and   Fig.~\ref{fig:pca}(c, d))   to assess the similarity between the flatness of the loss landscape and the noise structure induced by dropout at the obtained minima, and    all of them consistently demonstrate two important relationships: i) Inverse variance-flatness relation: The noise is larger at the sharper direction of the loss landscape; ii) Hessian-variance alignment relation: The Hessian of the loss landscape at the found minima aligns with the noise covariance matrix. 
These two relations are compatible with each other in that  they collectively contribute to the ability of the training algorithm    to effectively identify flatter minima.  Our experiments are conducted on several representative datasets, i.e., MNIST \citep{lecun1998gradient}, CIFAR-100 \citep{krizhevsky2009learning} and Multi30k \citep{elliott2016multi30k}, and also on distinct NN structures, i.e., fully-connected neural networks (FNNs), ResNet-20 \citep{he2016deep} and transformer \citep{vaswani2017attention} to demonstrate the universality of our findings.

\section{Related works}
A flurry of recent works aims to  shed light on the regularization effect   conferred by dropout. \citet{wager2013dropout} show  that     dropout  performs a form of adaptive regularization in the context of linear regression and logistic problems. \citet{mcallester2013pac} propose  a  PAC-Bayesian bound, whereas \citet{wan2013regularization,mou2018dropout} derive   some Rademacher-complexity-type error bounds specifically tailored for dropout. \citet{mianjy2020convergence} demonstrate that dropout training with logistic loss achieves $\varepsilon$-suboptimality in test error within $O(1/\varepsilon)$ iterations. Finally, \citet{zhang2022implicit} establish that dropout enhances the flatness of the loss landscape and facilitates condensation through an additional regularization term endowed by dropout.

Continuous formulations have been   extensively utilized to   study the dynamical behavior  of  stochastic algorithms. \citet{li2017stochastic, li2019stochastic} present an entirely
rigorous and self-contained mathematical formulation of the SME framework that applies to a wide class of stochastic algorithms.  Furthermore, \citet{feng2017semi} adopt a semigroup approach to investigate the dynamics of SGD and online PCA.  \citet{malladi2022SDE} derive  the SME approximations for the adaptive stochastic algorithms including RMSprop and
Adam, additionally, they  provide efficient experimental verification of the validity of  square root scaling rules   arising from the  SMEs.

One noteworthy observation is the association between the flatness of minima and improved generalization ability~\citep {li2017visualizing, jastrzkebski2017three, jastrzkebski2018relation}. Specifically,  SGD is shown to preferentially select flat minima, especially under conditions of large learning rates and small batch sizes~\citep{jastrzkebski2017three, jastrzkebski2018relation, wu2018sgd}.  \citet{papyan2018full,papyan2019measurements}  attribute  such enhancement of flatness by SGD to the similarity between   covariance of the noise  and  Hessian  of the loss function. 
Furthermore, \citet{feng2021inverse}    reveal  an   inverse   variance–flatness relation   within the dynamics of SGD.  
Additionally, \citet{zhu2018anisotropic,wu2022alignment}    unveil the Hessian-variance alignment property of SGD noise, shedding light on the role of SGD in escaping from sharper minima and locating flatter minima. 

\section{Preliminary}

In this section,  we present the notations and definitions that are utilized in our theoretical analysis. {\emph{We remark that our experimental settings are more general than the counterparts in the  theoretical analysis.}}

\subsection{Notations}\label{subsection...Notations...Main}
We set a special vector $(1,1,1,\dots,1)^\T$ by $\vone:=(1,1,1,\dots,1)^\T$ whose dimension varies.  We set $n$ for the number of input samples and $m$ for the width of the NN.
We let $[n]=\{1,2, \ldots, n\}$. 
We denote $\otimes$ as the Kronecker tensor product, and $\left<\cdot,\cdot\right>$ for standard inner product between two vectors.
We denote vector $L^2$ norm as $\Norm{\cdot}_2$, vector or function $L_{\infty}$ norm as $\Norm{\cdot}_{\infty}$. 
Finally,  we denote the set of continuous functions $f(\cdot):\sR^D\to\sR$ possessing   continuous derivatives of order up to and including $r$ by $\fC^{r}(\sR^D)$,     the space of bounded measurable functions by $\fB_b(\sR^D)$, and  the space of bounded continuous functions by $\fC_b(\sR^D)$.  
\subsection{Two-layer neural networks and loss function}
We consider the empirical risk minimization problem given by the quadratic loss:
\begin{equation}
\min_{\vtheta}R_{\fS}(\vtheta)=\frac{1}{2n}\sum_{i=1}^n\left({f_{\vtheta}(\vx_i)-y_i}\right)^2,
\end{equation}
where $\fS:=\{ (\vx_i, y_i)\}_{i=1}^n$ is  the training sample,    $f_{\vtheta}(\vx)$ is the prediction function,  $\vtheta$ are the parameters, and their dependence is modeled by   a two-layer NN with $m$  hidden neurons
\begin{equation}
    f_{\vtheta}(\vx) := \sum_{r=1}^{m}a_r\sigma(\vw_r^{\T}\vx),
\end{equation}
where $\vx\in\sR^{d}$,   $\vtheta=\mathrm{vec}(\vtheta_a,\vtheta_{\vw})\in\sR^D$, where   $D:=m(d+1)$ throughout this paper. We remark that $\vtheta $ is the set of parameters with   $\vtheta_a=\mathrm{vec}(\{a_r\}_{r=1}^{m})$, $\vtheta_{\vw}=\mathrm{vec}(\{\vw_r\}_{r=1}^{m})$, and $\sigma(\cdot)$ is the activation function.  More precisely, $\vtheta=\mathrm{vec}(\{\vq_r\}_{r=1}^m)$, where  for each $r\in[m]$, $\vq_r:=(a_r,\vw_r^{\T})^\T$, and   the bias term $b_r$ can be incorporated by expanding $\vx$ and $\vw_r$ to $(\vx^\T,1)^\T$ and $\left(\vw_r^\T,b_r\right)^\T$. 

\subsection{Dropout}
Given fixed learning rate $\eps>0$, then  at the $N$-th iteration  where  $t_N:=N\eps$, a scaling vector $\veta_N \in \sR^{m}$ is sampled  with independent random coordinates: For each $k\in[m]$,
\begin{equation}
    (\veta_N)_{k}= \begin{cases}\frac{1}{p} & \text { with probability } p, \\ 0 & \text { with probability } 1-p, \end{cases}
\end{equation}
and we observe that $\{\veta_N\}_{N\geq 1}$ is an i.i.d.\ Bernoulli sequence with $\Exp \veta_N=\vone$. With slight abuse of notations, the $\sigma$-fields $\fF_N:=\left\{\sigma(\veta_1, \veta_2, \cdots \veta_N)\right\}$ forms a natural filtration. 
We then apply dropout to the two-layer NNs by computing 
\begin{equation}
f_{\vtheta}(\vx;\veta):=\sum_{r=1}^m (\veta)_{r}a_r\sigma(\vw_r^{\T}\vx),
\end{equation}
and   we denote the empirical risk associated with dropout   by  
\begin{equation}
\begin{aligned}
\RS^\mathrm{drop}\left(\vtheta;\veta\right)&:= \frac{1}{2n}\sum_{i=1}^n\left(f_{\vtheta}(\vx_i;\veta)-y_i\right)^2=\frac{1}{2n}\sum_{i=1}^n\left(\sum_{r=1}^m (\veta)_{r}a_r\sigma(\vw_r^{\T}\vx_i)-y_i\right)^2.   
\end{aligned}
\end{equation}
We observe that the parameters at the $N$-th step are updated   as follows: 
\begin{equation}\label{eq...text...SME...ModifiedLoss...ThetaUpdate...Abstract}
 \vtheta_N= \vtheta_{N-1}-\eps \nabla_{\vtheta}\RS^\mathrm{drop}\left(\vtheta_{N-1}; \veta_{N}\right),
\end{equation}
where $\vtheta_0:=\vtheta(0)$.  Finally,  we denote hereafter that  for all $i\in[n]$, 
\[
e_{i}^N :=e_i(\vtheta_{N-1};\veta_N) := f_{\vtheta_{N-1}}(\vx_i; \veta_N) - y_i.   
\]

\section{Stochastic modified equations for dropout}

In this section, we approximate the iterative process  of dropout \eqref{eq...text...SME...ModifiedLoss...ThetaUpdate...Abstract} in the weak sense~(Definition \ref{defi...one}).  

\subsection{Modified loss}\label{subsection...ModifiedLoss}
As the dropout iteration \eqref{eq...text...SME...ModifiedLoss...ThetaUpdate...Abstract} can be written into
\begin{align*}
\vtheta_N- \vtheta_{N-1}&= -\eps \nabla_{\vtheta}\RS^\mathrm{drop}\left(\vtheta_{N-1}; \veta_{N}\right)=-\frac{\eps}{n}\sum_{i=1}^ne_{i}^N \nabla_{\vtheta}e_{i}^N.
\end{align*}    
Since $\vtheta=\mathrm{vec}(\{\vq_r\}_{r=1}^m)=\mathrm{vec}\left(\{(a_r, \vw_r)\}_{r=1}^m\right)$, then given $\vtheta_{N-1}$, for each $k\in [m]$,   the expectation of the  increment  restricted to $\vq_k$ reads
\begin{align*}
&\Exp_{\vtheta_{N-1}}\left[\sum_{i=1}^ne_{i}^N \nabla_{\vq_k}e_{i}^N \right]  
= \Exp_{\vtheta_{N-1}}\left[\sum_{i=1}^ne_{i}^N(\veta_N)_{k}\nabla_{\vq_k}\left(a_k\sigma(\vw_k^{\T}\vx_i)\right) \right]  
\\
=& \sum_{i=1}^ne_i\nabla_{\vq_k}\left(a_k\sigma(\vw_k^{\T}\vx_i)\right) +\frac{1-p}{p}\sum_{i=1}^n a_{k}\sigma(\vw_{k}^{\T}\vx_i) \nabla_{\vq_k}\left(a_k\sigma(\vw_k^{\T}\vx_i)\right),
\end{align*}
where we denote for simplicity   that 
$
e_{i}:=e_{i}(\vtheta):=\sum_{r=1}^m a_{r}\sigma(\vw_{r}^{\T}\vx_i)-y_i,
$
and compared with  $e_{i}^N$, $e_i$  does not depend on the random variable $\veta_N$. 
Hence,  the {\emph{modified loss}} $L_S(\cdot):\sR^{D}\to\sR$  for dropout can be defined as:
\begin{equation}\label{eq...text...SME...ModifiedLoss...ModifiedLoss}
\begin{aligned}
L_S(\vtheta)&:=\frac{1}{2n}\sum_{i=1}^ne_i^2 +\frac{1-p}{2np}\sum_{i=1}^n \sum_{r=1}^m a_{r}^2\sigma(\vw_{r}^{\T}\vx_i)^2,
\end{aligned}
\end{equation}
in that  as $\vtheta_{N-1}$ is  given, by taking conditional expectation, its increment reads
\begin{align*}
\vtheta_N- \vtheta_{N-1}&= -\eps \Exp_{\vtheta_{N-1}}\left[\nabla_{\vtheta}\RS^\mathrm{drop}\left(\vtheta_{N-1}; \veta_{N}\right)\right]=-\eps\nabla_{\vtheta}L_S(\vtheta)\big|_{\vtheta=\vtheta_{N-1}},
\end{align*}    
then in the sense of expectations, $\{\vtheta_N\}_{N\geq 0}$ follows close to the gradient descent~(GD) trajectory of  $L_S(\vtheta)$ with fixed learning rate $\eps$.

\subsection{Stochastic modified equations}\label{subsection...ModifiedEquaitons}
Firstly, from the results in Section \ref{subsection...ModifiedLoss}, we observe that given $\vtheta_{N-1}$,
\begin{equation}\label{eq...text...SME...SME...DiscreteUpdate+Covariance}
\vtheta_N-\vtheta_{N-1} = -\eps\nabla_{\vtheta}L_S(\vtheta)\big|_{\vtheta=\vtheta_{N-1}}+\sqrt{\eps}\vV(\vtheta_{N-1}),    
\end{equation}
where $L_S(\cdot):\sR^{D}\to\sR$ is the modified loss defined in \eqref{eq...text...SME...ModifiedLoss...ModifiedLoss}, and $\vV(\cdot):\sR^{D}\to\sR^{D}$ is a $D$-dimensional random vector, and when  given $\vtheta_{N-1}$, $\vV(\vtheta_{N-1})$ has mean $\vzero$ and covariance  $\eps \mSigma(\vtheta_{N-1})$, where $\mSigma (\cdot):\sR^{D}\to \sR^{D \times {D}}$, whose expression  is deferred to Section \ref{subsection...Explicit}. 

Consider the stochastic differential equation~(SDE),  
\begin{equation}\label{eq...text...SME...SME...SDE...Abstract}
\D \vTheta_t=\vb \left(\vTheta_t\right)\D t+\vsigma \left(\vTheta_t\right) \D \vW_t, \quad  \vTheta_0=\vTheta(0),
\end{equation}
where $\vW_t$ is a standard $D$-dimensional Brownian motion, and its Euler–Maruyama discretization  with step size $\eps>0$  at the $N$-th step reads
\begin{equation*}
  \vTheta_{\eps N}=  \vTheta_{\eps(N-1)}+\eps\vb \left(\vTheta_{\eps(N-1)}\right)  +\sqrt{\eps}\vsigma \left(\vTheta_{\eps(N-1)}\right) \vZ_{N}, 
\end{equation*}
where $\vZ_N\sim\fN(\vzero, \mI_{D})$ and $\vTheta_0=\vTheta(0)$. Thus, if we set 
\begin{equation}\label{eq...text...SME...SME...theChoiceofB+Sigma}
\begin{aligned}
\vb\left(\vTheta\right)&:=   -\nabla_{\vTheta}L_S(\vTheta),\\     
\vsigma\left(\vTheta\right)&:=\sqrt{\eps}\left(\mSigma\left(\vTheta\right)\right)^{\frac{1}{2}},\\
\vTheta_0&:=\vtheta_0,
\end{aligned}
\end{equation}
then we would expect \eqref{eq...text...SME...SME...SDE...Abstract} to be a `good' approximation of \eqref{eq...text...SME...SME...DiscreteUpdate+Covariance} with time identification $t=\eps N$. Based on the previous work \citep{li2017stochastic}, we use approximations in the {\emph{weak}} sense~\cite[Section 9.7]{kloeden2011numerical} since the path of dropout and the corresponding SDE are driven by noises sampled in different spaces.

To compare different discrete  time approximations, we need to take the rate of weak convergence into consideration, and we also  need to choose  an appropriate  class of functions as the space of test functions.
We introduce the following set of smooth functions:
\begin{equation}
\fC_b^M\left(\sR^D\right)=\left\{ f \in \fC^M\left(\sR^D\right) \Bigg|\Norm{f}_{\fC^M}:=\sum_{|\beta| \leq M}  \Norm{\mathrm{D}^\beta f }_{\infty}<\infty \right\},
\end{equation}
where $\mathrm{D}$ is the usual differential operator.
We remark that $\fC_b^M(\sR^D)$ is a subset of $\fG(\sR^D)$, the class of functions with polynomial growth, which is chosen to be the space of test functions in previous works \citep{li2017stochastic,kloeden2011numerical,malladi2022SDE}. 
Before we proceed to the definition of weak approximation, to ensure the rigor and validity of our analysis, we   assume that
\begin{assumption}\label{theOnlyAssumption}
There exists $T^{\ast}>0$, such that for any $t\in\left[0, T^{\ast} \right]$, there exists a unique $t$-continuous solution $\vTheta_t$ to SDE \eqref{eq...text...SME...SME...SDE...Abstract}.
Furthermore,   for each $l \in [3]$, there exists $C(T^{\ast},\vTheta_0)>0$, such that 
\begin{equation}\label{eq...assump...UniformBDD...Cts}
\sup_{{0}\leq s\leq T^{\ast}}\Exp\left(\Norm{\vTheta_s(\cdot)}_2^{2l}\right)\leq  C(T^{\ast},\vTheta_0).
\end{equation}
Moreover,  for the dropout iterations \eqref{eq...text...SME...ModifiedLoss...ThetaUpdate...Abstract}, let $0<\eps<1$, $T>0$ and set $N_{T,\eps}:=\lfloor \frac{T}{\eps} \rfloor$.    There exists    $\eps_0>0$, such that    given any learning rate  $\eps\leq \eps_0$,  then for  all $N\in[0:N_{T^{\ast},\eps}]$ and     for each $l \in [3]$, there exists $C(T^{\ast},\vtheta_0,\eps_0)>0$,  such that 
\begin{equation}\label{eq...assump...UniformBDD...Discrete}
\sup_{{0}\leq N\leq [N_{T^{\ast},\eps}]}\Exp\left(\Norm{\vtheta_N}_2^{2l}\right)\leq  C(T^{\ast},\vtheta_0,\eps_0).
\end{equation}
\end{assumption}
\noindent
We remark that if  $\fG(\sR^D)$ is chosen to be  the test functions in~\cite{li2019stochastic},  then   similar     relations to  \eqref{eq...assump...UniformBDD...Cts} and  \eqref{eq...assump...UniformBDD...Discrete} shall be imposed, except that in our cases, we only require the second, fourth and sixth moments to be uniformly bounded, while  in their cases,  all  $2l$-moments are required for $l\geq 1$.
\begin{definition}\label{defi...one}
The SDE \eqref{eq...text...SME...SME...SDE...Abstract} is an order $\alpha$ weak approximation to the dropout \eqref{eq...text...SME...ModifiedLoss...ThetaUpdate...Abstract}, if for every $g\in \fC_b^M\left(\sR^{D}\right)$, there exists $C>0$ and $\eps_0>0$,  such that     given any $\eps\leq\eps_0$ and $T\leq T^{\ast}$, then for all $N\in[N_{T,\eps}]$,
\begin{equation}\label{eq...definition...WeakApproximationAlpha,,,Main}
\Abs{\Exp g(\vTheta_{\eps N}) -\Exp g(\vtheta_N)}    \leq C(T^{\ast}, g, \eps_0)\eps^{\alpha}. 
\end{equation}
\end{definition}
\noindent
We now state informally our approximation theorem.
\begin{customthm}{1*}
Fix time $T\leq T^{\ast}$ and learning rate $\eps>0$, then if we choose  
\begin{align*}
\vb(\vTheta)&=-\nabla_{\vTheta}L_S(\vTheta),\\
\vsigma(\vTheta)&=\sqrt{\eps}\left(\mSigma\left(\vTheta\right)\right)^{\frac{1}{2}},
\end{align*}
then for all $t\in[0, T]$, the stochastic processes $\vTheta_t $ satisfying
\begin{equation*}
\D \vTheta_t=\vb\left(\vTheta_t\right)\D t+\vsigma \left(\vTheta_t\right) \D \vW_t, 
\end{equation*}
is an order-$1$ approximation of dropout \eqref{eq...text...SME...ModifiedLoss...ThetaUpdate...Abstract}. If we choose instead
\begin{align*}
\vb(\vTheta)&=-\nabla_{\vTheta}\left(L_S(\vTheta)+\frac{\eps}{4}\Norm{\nabla_{\vTheta}L_S(\vTheta)}_2^2\right),\\
\vsigma(\vTheta)&=\sqrt{\eps}\left(\mSigma\left(\vTheta\right)\right)^{\frac{1}{2}},
\end{align*}
then $\vTheta_t$  
is an order-$2$ approximation.
\end{customthm}
It is noteworthy that our   findings reproduce the explicit  regularization effect attributed to dropout~\citep{wei2020implicit,zhang2022implicit}.  This regularization effect    modifies the expected training objective     from $R_{\fS}(\theta)$ to $L_{\fS}(\theta)$.  The regularization effect stems from the stochasticity of dropout. Unlike SGD, where the noise arises from the stochasticity involved in the selection of training samples, dropout introduces noise through the stochastic removal of parameters. In the sequel, we  focus on how such stochasticity exerts an impact on our learning results. 
\section{The effect of the noise structure on flatness} \label{sec:noise}

We begin this section by examining the   expression of the noise structure arising from dropout. 

\subsection{Explicit form of the dropout noise structure}\label{subsection...Explicit}

In this subsection, we present the expression for  $\mSigma$.
Once again, as $\vtheta=\mathrm{vec}(\{\vq_r\}_{r=1}^m)=\mathrm{vec}\left(\{(a_r, \vw_r)\}_{r=1}^m\right)$, then  covariance of $\nabla_{\vtheta}\RS^\mathrm{drop}\left(\vtheta_{N-1}; \veta_{N}\right)$ equals to   $ \mSigma(\vtheta_{N-1})$.   We denote 
 \[\mSigma_{kr}(\vtheta_{N-1}):=\mathrm{Cov}\left( \nabla_{\vq_k}\RS^\mathrm{drop}\left(\vtheta_{N-1}; \veta_{N}\right), \nabla_{\vq_r}\RS^\mathrm{drop}\left(\vtheta_{N-1}; \veta_{N}\right)\right), \]
then 
 \[
\mSigma=\left[\begin{array}{cccc}
 \mSigma_{11} &  \mSigma_{12} &  \cdots & \mSigma_{1m}  \\
\mSigma_{21} &  \mSigma_{22} &  \cdots & \mSigma_{2m}\\ 
\vdots& \vdots&\vdots&\vdots\\
\mSigma_{m1} &  \mSigma_{m2} &  \cdots & \mSigma_{mm}
\end{array}\right].
\]
 For each  $  k \in [m]$, we obtain that 
\begin{align*}
&\mSigma_{kk}(\vtheta_{N-1})=\mathrm{Cov}\left( \nabla_{\vq_k}\RS^\mathrm{drop}\left(\vtheta_{N-1}; \veta_{N}\right), \nabla_{\vq_k}\RS^\mathrm{drop}\left(\vtheta_{N-1}; \veta_{N}\right)\right) \\
=&\left(\frac{1}{p}-1\right)\left(\frac{1}{n}\sum_{i=1}^n\left( e_{i,\backslash k}+\frac{1}{p}a_{k}\sigma(\vw_{k}^{\T}\vx_i)\right)\nabla_{\vq_k}\left(a_k\sigma(\vw_k^{\T}\vx_i)\right) \right)\\
&~~~~~~~~~~~~~~~~~~~~~~~~~~~~~~{\otimes}\left(\frac{1}{n}\sum_{i=1}^n\left( e_{i,\backslash k}+\frac{1}{p}a_{k}\sigma(\vw_{k}^{\T}\vx_i)\right)\nabla_{\vq_k}\left(a_k\sigma(\vw_k^{\T}\vx_i)\right) \right)\\
&~+\left(\frac{1}{p^2}-\frac{1}{p}\right)\sum_{k'=1, k'\neq k}^m\left(\frac{1}{n}\sum_{i=1}^na_{k'}\sigma(\vw_{k'}^{\T}\vx_i)\nabla_{\vq_k}\left(a_k\sigma(\vw_k^{\T}\vx_i)\right)\right)\\
&~~~~~~~~~~~~~~~~~~~~~~~~~~~~~~{\otimes}\left(\frac{1}{n}\sum_{i=1}^na_{k'}\sigma(\vw_{k'}^{\T}\vx_i)\nabla_{\vq_k}\left(a_k\sigma(\vw_k^{\T}\vx_i)\right)\right),
\end{align*}
where $
 e_{i,\backslash k}:=e_{i,\backslash k}(\vtheta):=\sum_{l=1, l\neq k
 }^m a_{l}\sigma(\vw_{l}^{\T}\vx_i)-y_i,$
and for each $  k, r \in [m]$ with $k \neq r$,  
\begin{align*}
 \mSigma_{kr}(\vtheta_{N-1}) =&\mathrm{Cov}\left( \nabla_{\vq_k}\RS^\mathrm{drop}\left(\vtheta_{N-1}; \veta_{N}\right), \nabla_{\vq_r}\RS^\mathrm{drop}\left(\vtheta_{N-1}; \veta_{N}\right)\right) \\
=&\left(\frac{1}{p}-1\right)\sum_{k'=1, k'\neq k, k'\neq r}^m\left(\frac{1}{n}\sum_{i=1}^na_{k'}\sigma(\vw_{k'}^{\T}\vx_i)\nabla_{\vq_k}\left(a_k\sigma(\vw_k^{\T}\vx_i)\right)\right)\\
&~~~~~~~~~~~~~~~~~~~~~~~~~~~~~~{\otimes}\left(\frac{1}{n}\sum_{i=1}^na_{k'}\sigma(\vw_{k'}^{\T}\vx_i)\nabla_{\vq_r}\left(a_r\sigma(\vw_r^{\T}\vx_i)\right)\right)\\
&~+\left(\frac{1}{p}-1\right)\left(\frac{1}{n}\sum_{i=1}^n\left(e_{i,\backslash k, \backslash r}+\frac{1}{p}a_k\sigma(\vw_k^\T\vx_i)+\frac{1}{p}a_r\sigma(\vw_r^\T\vx_i)\right)\nabla_{\vq_k}\left(a_k\sigma(\vw_k^{\T}\vx_i)\right)\right)\\
&~~~~~~~~~~~~~~~~~~~~~~~~~~~~~~ {\otimes}\left(\frac{1}{n}\sum_{i=1}^na_k\sigma(\vw_k^\T\vx_i)\nabla_{\vq_r}\left(a_r\sigma(\vw_r^{\T}\vx_i)\right)\right)\\
&~+\left(\frac{1}{p}-1\right)\left(\frac{1}{n}\sum_{i=1}^na_r\sigma(\vw_r^\T\vx_i)\nabla_{\vq_k}\left(a_k\sigma(\vw_k^{\T}\vx_i)\right)\right) \\
&~~~~~~~~~~~~~~~~~~~~~~~~~~~~~~{\otimes}\left(\frac{1}{n}\sum_{i=1}^n\left(e_{i,\backslash k, \backslash r}+a_k\sigma(\vw_k^\T\vx_i)+\frac{1}{p}a_r\sigma(\vw_r^\T\vx_i)\right)\nabla_{\vq_r}\left(a_r\sigma(\vw_r^{\T}\vx_i)\right)\right),
\end{align*}
where $
 e_{i,\backslash k, \backslash r}:=e_{i,\backslash k,\backslash r}(\vtheta):=\sum_{l=1, l\neq k, l\neq r
 }^m a_{l}\sigma(\vw_{l}^{\T}\vx_i)-y_i$. We remark that such expression is consistent in that for the extreme case where $p=1$, dropout `degenerates' to GD, hence the covariance matrix degenerates to a zero matrix, i.e., $\mSigma=\mzero_{D\times D}$.

\subsection{Experimental results on the dropout noise structure}

In this subsection, we endeavor to show the structural similarity between  the  covariance   and the Hessian  in terms of both  Hessian-variance alignment relations and  Inverse variance-flatness relations. Intuitively,  the structural similarity between  the Hessian   and   covariance matrix is shown below:
\begin{equation}\label{eq...text...similar}
 \begin{aligned}
    \mH(\vtheta)&\approx\frac{1}{n} \sum_{i=1}^{n}\left[ \nabla_{\vtheta} f_{\vtheta}\left(\vx_{i}\right){\otimes} \nabla_{\vtheta} f_{\vtheta}\left(\vx_{i}\right)+ \frac{1-p}{p} \sum_{r=1}^{m} \nabla_{\vq_r}\left(a_r\sigma(\vw_r^{\T}\vx_i)\right)   {\otimes}\nabla_{\vq_r}\left(a_r\sigma(\vw_r^{\T}\vx_i)\right) \right],\\  
    \mSigma(\vtheta)&\approx\frac{1}{n}\sum_{i=1}^{n} \left[ l_{i,1} \nabla_{\vtheta}f_{\vtheta}(\vx_i){\otimes}\nabla_{\vtheta}f_{\vtheta}(\vx_i)+ l_{i,2} \frac{1-p}{p}\sum_{r=1}^{m} \nabla_{\vq_r}\left(a_r\sigma(\vw_r^{\T}\vx_i)\right)   {\otimes}\nabla_{\vq_r}\left(a_r\sigma(\vw_r^{\T}\vx_i)\right)\right],
\end{aligned}
\end{equation}
where  $\mH(\vtheta):=\nabla^2_{\vtheta}L_{\fS}(\vtheta)$ , and $l_{i,1}:=  (e_{i})^2+\frac{1-p}{p}\sum_{r=1}^m a_r^2\sigma(\vw_r^\T\vx_i)^2$, $l_{i,2}:=  (e_{i})^2 $ , and the  detailed derivation for \eqref{eq...text...similar} is deferred  to the Appendix. We remark that the expression for the covariance matrix in \eqref{eq...text...similar} differs from the counterpart in Section \ref{subsection...Explicit} since some certain assumptions, as outlined in~\cite{zhu2018anisotropic}, have been imposed. 
With the established structural similarity through the aforementioned   intuitive approximations shown in \eqref{eq...text...similar},  we proceed to the empirical  investigation concerning the intricate relationship between the Hessian   and the covariance.

\subsubsection{Random data collection methods} \label{sec:randomness}
We first introduce two types of dynamical datasets collected during dropout training to study the noise structure of dropout. These datasets are different from the training sample $\fS$. 

\textbf{Random trajectory data. }The training process of NNs usually consists of two phases: the fast convergence phase and the exploration phase \citep{shwartz2017opening}.  In the exploration phase, the network is often considered to be near a minimum, and the movement of parameters is largely affected by the noise structure. Based on the previous work \citep{feng2021inverse}, 
we collect parameter sets $\fD_{\rm para}:=\{\vtheta_{i}\}_{i=1}^{N}$ from $N$ consecutive training steps in the exploration phase, where $\vtheta_{i}$ is the network parameter set at $i$-th sample step. 
This sampling method requires a large number of training steps, so model parameters often have large fluctuations during the sampling process. To improve the sampling accuracy, we propose another type of random data to characterize the noise structure of dropout as follows.


\textbf{Random gradient data. }We   train the network until the loss is near  zero and then we freeze the training process, then we sample $N$ realizations of the dropout variable to get the random gradient dataset, i.e., $\fD_{\rm grad}:=\{\vg_{i}\}_{i=1}^{N}$. The $i$-th sample point $\vg_{i}$ is obtained as follows: i) Firstly, we generate a realization of the dropout variable $\veta_i$ under a given dropout rate; ii) Then, we compute the gradient of the loss function with respect to the parameters, denoted by $\vg_i(\cdot):=\nabla  \RS^\mathrm{drop}\left(\cdot;\veta_i\right)$.
Each element in $\fD_{\rm grad}$ represents an evolution direction of network parameters, determined by the dropout variable. Therefore, studying the structure of $\fD_{\rm grad}$ can help us understand how the dropout noise exerts an impact throughout the training process.

\subsubsection{Hessian-Variance alignment}



 In this subsection, we employ a  metric  $ \operatorname{Tr}(\mH_i \mSigma_i)  $ established to be valuable~\citep{zhu2018anisotropic} in the assessment of  the degree of alignment between the noise structure and  curvature of the loss landscape, 
where $ \operatorname{Tr}(\cdot)$  stands for the trace of a square matrix,  $\mSigma_i$ is the   covariance matrix of $\fD_{\rm grad}$  sampled at the  $i$th-step, whose   definition can be found in Section \ref{sec:randomness}, and $\mH_i$ is the Hessian   of the loss function at the $i$th-step. 

To investigate  the Hessian-Variance alignment relation, we construct an isotropic noise    termed  $\bar{\mSigma}_i$ by means of averaging, i.e., $\bar{\mSigma}_i=\frac{\operatorname{Tr} (\mSigma_i)}{D} \mI_{D\times D}$, where $D$ is the total number of parameters, $\mI_{D\times D}$ is the identity matrix, and  $\bar{\mSigma}_i$ is employed for comparative purposes. As shown in Fig. \ref{fig:anisotropic},  under different learning rates and dropout rates, $\operatorname{Tr}(\mH_i \mSigma_i)$ significantly exceeds $\operatorname{Tr}(\mH_i \bar{\mSigma}_i)$ throughout the whole training process, thus indicating that dropout-induced  noise possesses an anisotropic structure that aligns well with the Hessian across all directions. It should be acknowledged that due to computational limitations, this experiment limits the trace calculation of $\bar{\mSigma}_i$  to  a subset of parameters, which can be regarded as the projection of the Hessian and the noise into some specific  directions.

\begin{figure}[h]
	\centering
	\includegraphics[width=0.5\textwidth]{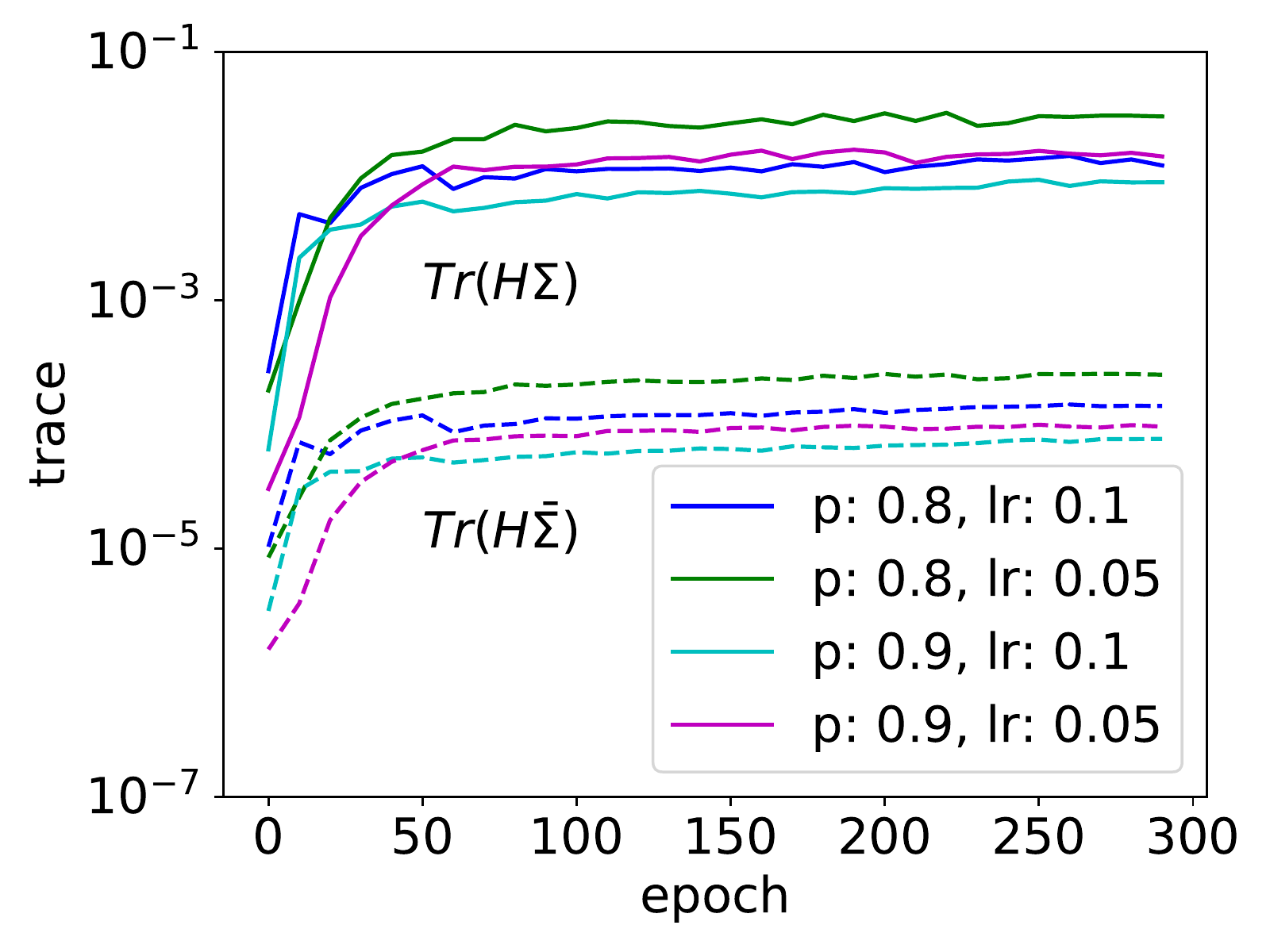}
	\caption{Comparison between $\operatorname{Tr}(\mH_i \mSigma_i)$ and $\operatorname{Tr}(\mH_i \bar{\mSigma}_i)$ in each training epoch $i$ for different choices of $p$ and learning rate $lr$. The FNN is trained on the MNIST dataset using the first 10000 examples as the training dataset. The solid and the dotted lines represent the value of $\operatorname{Tr}(\mH_i \mSigma_i)$ and $\operatorname{Tr}(\mH_i \bar{\mSigma}_i)$, respectively. \label{fig:anisotropic}}
\end{figure}

\subsubsection{Inverse variance-flatness relation} \label{sec:inv_var}
  The alignment relation studied above also implies the inverse variance-flatness relation, i.e.,  the noise variance is large  along the  sharp   direction of the loss landscape, and small along the  flat   direction. In this subsection, we verify this relation by two sets of experiments.  Firstly, we  present two different approaches to characterize the flatness of loss landscape and the covariance of noise from  the random trajectory data $\fD_{\rm para}$ and random gradient data $\fD_{\rm grad}$, then we numerically demonstrate the inverse variance-flatness relation. Due to space limitations, we defer the experiments on ResNet and Transformer to Appendix \ref{app:noise_further}. For convenience,   $\fD$   refers to either the dataset $\fD_{\rm para}$ or  the dataset  $\fD_{\rm para}$ depending on its context, so is the case for their corresponding covariance   $\mSigma$ and Hessian $\mH$. We then proceed to the  definitions of \textbf{noise variance} and \textbf{interval flatness.}
 
\begin{definition}[\textbf{noise variance}]
    For dataset $\fD$ and its covariance $\mSigma$, we denote $\lambda_{i}(\mSigma)$ as the $i$th eigenvalue of $\mSigma$ and its corresponding eigen direction as $\vv_i(\mSigma)$. Then we term $\lambda_{i}(\mSigma)$  \emph{ the noise variance of $\fD$ at the eigen direction  $\vv_i(\mSigma)$.}
\end{definition}

The interval flatness below characterizes the flatness of the landscape around a local minimum.

\begin{definition}[\textbf{interval flatness}\footnote{This definition is also used in \cite{feng2021inverse}}
]
   For a  a local minimum $\vtheta^{*}_{0}$, the loss function profile $R_{\vv}$ along  direction $\vv$ reads:
$$R_{\vv}(\delta)\equiv R_{S}(\vtheta^{*}_{0}+\delta\vv), $$ 
 where $\delta$ represents the distance moved in the $\vv$ direction. The interval flatness $F_{\vv}$ is then defined as the width of the region within which $R_{\vv}(\delta)\leq 2R_{\vv}(0)$. We determine $F_{\vv}$ by finding two closest points $\theta_{\vv}^{l}<0$ and $\theta_{\vv}^{r}>0$ on each side of the minimum that satisfy $R_{\vv}(\theta_{\vv}^{l})=R_{\vv}(\theta_{\vv}^{r})=2R_{\vv}(0)$. The interval flatness is defined as:
\begin{equation}
  F_{\vv}\equiv \theta_{\vv}^{r}-\theta_{\vv}^{l}. \label{eq:Fv}
\end{equation}

\end{definition}

\begin{remark}
    The experiments show that the result is not sensitive to the selection of the pre-factor 2. A larger value of $F_{\vv}$ means a flatter landscape in the direction $\vv$.
\end{remark}

We use PCA to study the weight variations when the training accuracy is nearly $100\%$. 
The networks are trained with full-batch GD for different learning rates and dropout rates under the same random seed. When the loss is small enough, we sample the parameters or gradients of parameters $N$ times ($N=3000$ for this experiment) and study the relationship between $\{\lambda_{i}(\mSigma)\}_{i=1}^N$ and $\{F_{\vv_{i}(\mSigma)}\}_{i=1}^N$ for both weight dataset $\fD_{\rm para}$ and gradient dataset $\fD_{\rm grad}$.
\begin{figure*}[h]
	\centering
	\subfigure[FNN, $\fD=\fD_{\rm para}$]{\includegraphics[width=0.24\textwidth]{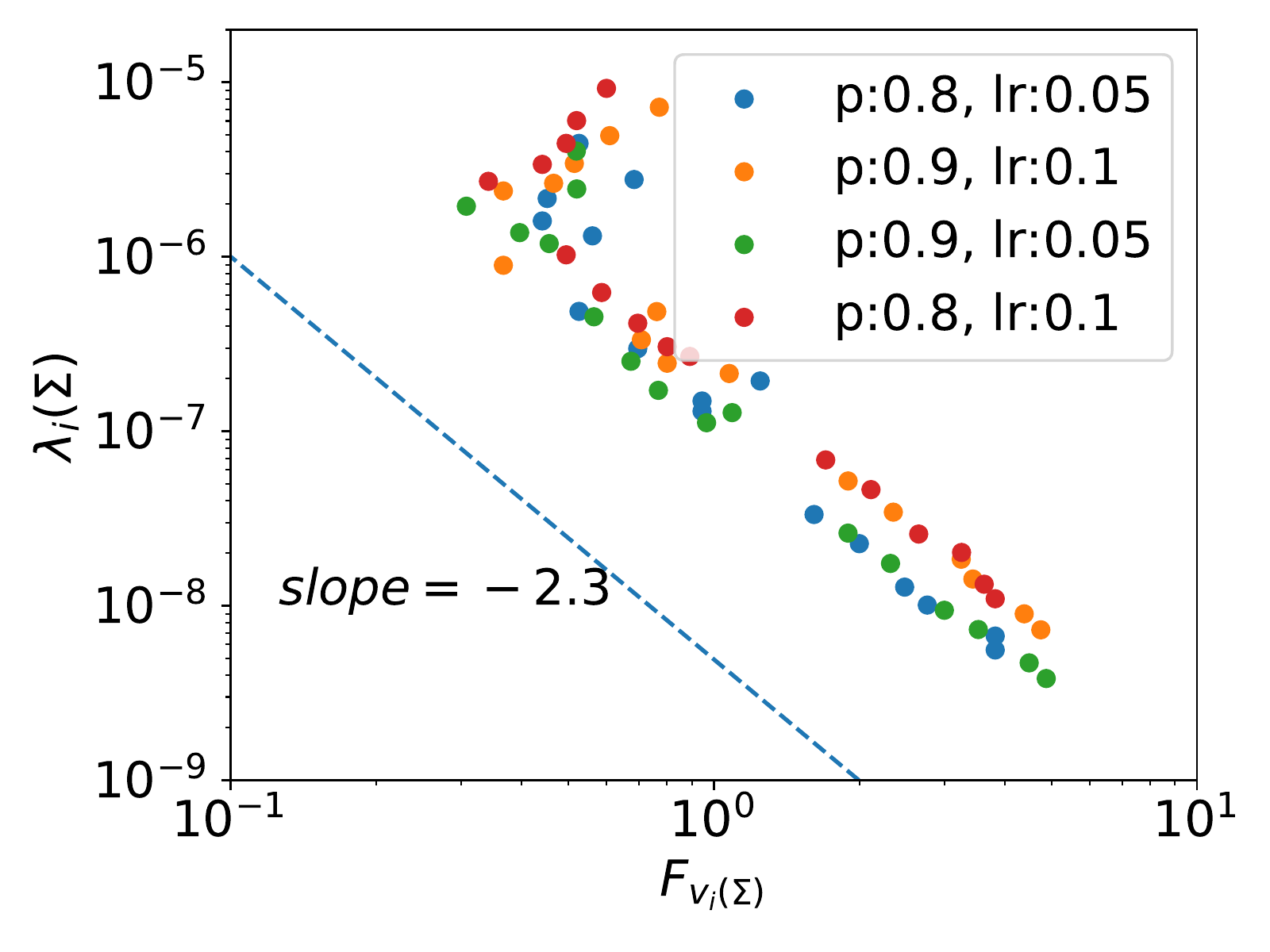}}
 \subfigure[FNN, $\fD=\fD_{\rm grad}$]{\includegraphics[width=0.24\textwidth]{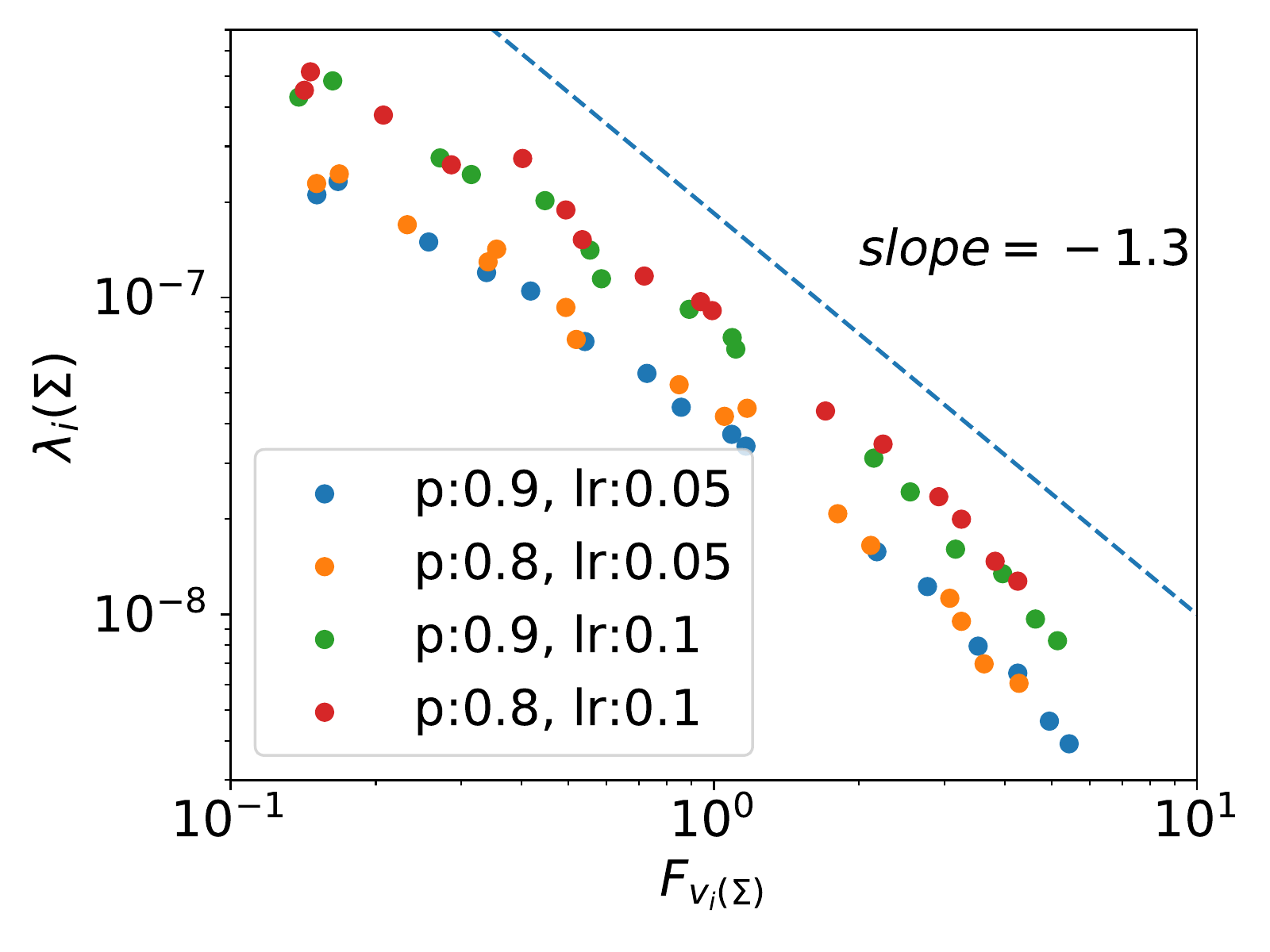}}
\subfigure[FNN, $\fD=\fD_{\rm para}$]{\includegraphics[width=0.24\textwidth]{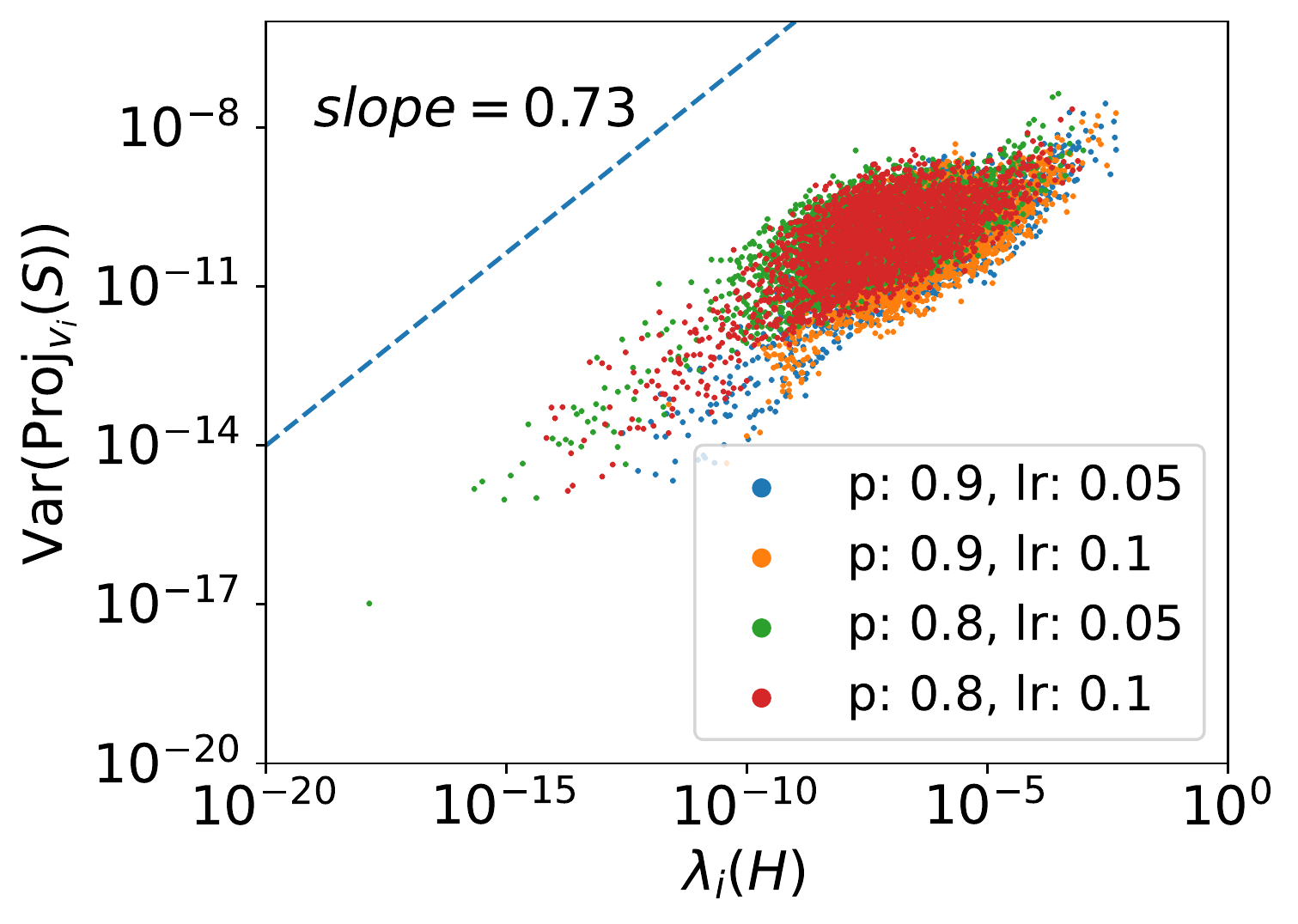}}
 \subfigure[FNN, $\fD=\fD_{\rm grad}$]{\includegraphics[width=0.24\textwidth]{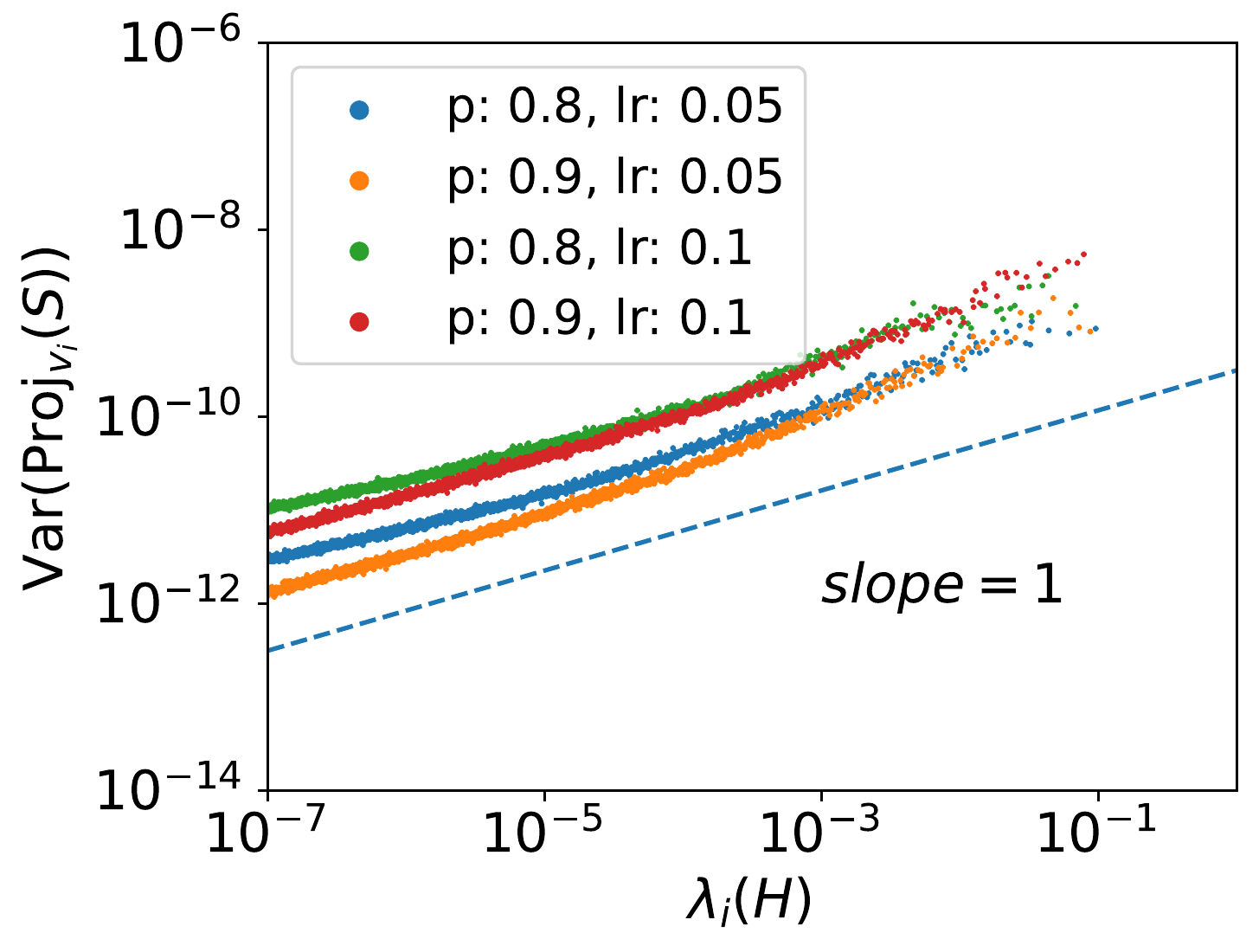}} 

  \caption{(a, b)The inverse relation between the variance $\{\lambda_{i}(\mSigma)\}_{i=1}^N$ and the interval flatness $\{F_{\vv_{i}(\mSigma)}\}_{i=1}^N$ for different choices of $p$ and learning rate $lr$ with different network structures. The PCA is done for different datasets $\fD$ sampled from parameters for the top line and sampled from gradients of parameters for the bottom line. The dashed lines give the approximate slope of the scatter. (c, d)The relation between the variance $\{{\rm Var}({\rm Proj}_{\vv_i(\mH)}(\fD))\}_{i=1}^N$ and the eigenvalue $\{\lambda_{i}(\mH)\}_{i=1}^{N}$ for different choices of $p$ and learning rate $lr$ with different network structures. The projection is done for different datasets $\fD$ sampled from parameters for the top line and sampled from gradients of parameters for the bottom line. The dashed lines give the approximate slope of the scatter. Refer to Appendix \ref{app:noise_further} for further experiments such as ResNet and Transformer.} \label{fig:pca}
\end{figure*} 

 For different learning rates and dropout rates, Fig. \ref{fig:pca}(a, b) reveal an inverse relationship between the interval flatness of the loss landscape denoted as  $\{F_{\vv_{i}(\mSigma)}\}_{i=1}^N$, and the noise variance represented by the    PCA spectrum $\{\lambda_{i}(\mSigma)\}_{i=1}^N$.  Notably, a power-law relationship  can be established    between $\{F_{\vv_{i}(\mSigma)}\}_{i=1}^N$ and $\{\lambda_{i}(\mSigma)\}_{i=1}^N$.  Specifically, in the low flatness region, the dropout-induced noise exhibits a large variance. As the loss landscape transitions into the high flatness regime, the linear relationship between variance and flatness becomes more evident. Overall,  These findings consistently demonstrate the inverse relation between variance and flatness, as exemplified in Fig. \ref{fig:pca}(a, b). Subsequently, we delve into the   definitions of \textbf{Projected variance} and \textbf{Hessian flatness.}


\begin{definition}[\textbf{projected variance}]
    For a given direction $\bm{v}\in \sR^{D}$ and dataset $\fD=\{\vtheta_{i}\}_{i=1}^{N}$, where $\vtheta_{i} \in \sR^{D}$, the inner product of $\bm{v}$ and $\vtheta_{i}$ is denoted by $ {\rm Proj}_{\bm{v}}(\vtheta_{i}):=\left<\vtheta_{i}, \bm{v}\right>$, then we can define the projected variance  for $\fD$  at the direction $\bm{v}$  as follows,
    \begin{equation*}
        {\rm Var}({\rm Proj}_{\bm{v}}(\fD))=\frac{\sum_{i=1}^{N}({\rm Proj}_{\bm{v}}(\vtheta_{i})-\vmu)^2}{N}, 
    \end{equation*}
    where $\vmu$ is the mean value of $\{{\rm Proj}_{\bm{v}}(\vtheta_{i})\}_{i=1}^{N}$.
\end{definition}

\begin{definition}[\textbf{Hessian flatness}]
    For Hessian   $\mH$, as we  denote $\lambda_{i}(\mH)$ by the $i$-th eigenvalue of $\mH$ corresponding to the eigenvector $\vv_i(\mH)$, we term $\lambda_{i}(\mH)$  the Hessian flatness along direction $\vv_i(\mH)$.
\end{definition}

The eigenvalues of the Hessian evaluated at a local minimum often serve as indicators of the flatness of the loss landscape, and larger eigenvalues correspond to sharper directions. In our investigation, we analyze the interplay between the eigenvalues of Hessian $\mH$ at the final stage of the training process and the projected variance of dropout at each of the corresponding eigen directions, i.e., $\lambda_i(\mH)$ v.s. $\{{\rm Var}({\rm Proj}_{\vv_i(\mH)}(\fD))\}_{i=1}^N$.  Specifically, we sample the parameters or gradients of parameters  $N$ times ($N=1000$ for this experiment), and examine the relationship between $\{\lambda_{i}(\mH)\}_{i=1}^{N}$ and $\{{\rm Var}({\rm Proj}_{\vv_i(\mH)}(\fD))\}_{i=1}^N$ for both the parameter dataset $\fD_{\rm para}$ and  the gradient dataset $\fD_{\rm grad}$. 

    Under various dropout rates and learning rates,  Fig. \ref{fig:pca}(c, d) presents establishes   a consistent power-law relationship between $\{\lambda_{i}(\mH)\}_{i=1}^{N}$ and $\{{\rm Var}({\rm Proj}_{\vv_i(\mH)}(\fD))\}_{i=1}^N$, and  this relationship remains robust  irrespective of   the choice between parameter dataset $\fD_{\rm para}$ or the gradient dataset $\fD_{\rm grad}$. The positive correlation observed  between the Hessian flatness and the projection variance provides  insights into the structural characteristics of the dropout-induced  noise. Specifically, these characteristics have the potential to facilitate   the escape from sharp minima and enhance the generalization capabilities of NNs. Additionally,    Fig. \ref{fig:pca} highlights the distinct linear structure exhibited by gradient sampling in comparison to parameter sampling,  which corroborates the discussions outlined  in  Section \ref{sec:randomness}.  For detailed experimental evidence,     including  our investigations involving ResNet and Transformer models, one may refer to Appendix \ref{app:noise_further}. 

\section{Conclusion}
Our main contribution is twofold. First, we derive the SMEs  that provide a weak approximation for the dynamics of the dropout algorithm for two-layer NNs. Second, we demonstrate that  dropout exhibits the inverse variance-flatness relation and the Hessian-variance alignment relation through extensive empirical analysis, which is consistent with SGD. These relations  are widely recognized to be beneficial for finding flatter minima, thus implying that dropout acts as an implicit regularizer that enhances the generalization abilities.

Given the broad applicability of  the methodologies employed in our proof,  we aim to extend the formulations of SMEs to an even wider class of stochastic algorithms applied to NNs with different architectures.  Such an extension could help us better understand the role of stochastic algorithms in NN training. Moreover, the SME framework could offer a promising approach to the examination of  the underlying mechanisms that explain the observed inverse variance-flatness relation and Hessian-variance  relation and beyond.

\section*{Acknowledgments}
This work is sponsored by the National Key R\&D Program of China  Grant No. 2022YFA1008200 (Z. X., T. L.), the Shanghai Sailing Program, the Natural Science Foundation of Shanghai Grant No. 20ZR1429000  (Z. X.), the National Natural Science Foundation of China Grant No. 62002221 (Z. X.), the National Natural Science Foundation of China Grant No. 12101401 (T. L.), Shanghai Municipal Science and Technology Key Project No. 22JC1401500 (T. L.), Shanghai Municipal of Science and Technology Major Project No. 2021SHZDZX0102, and the HPC of School of Mathematical Sciences and the Student Innovation Center, and the Siyuan-1 cluster supported by the Center for High Performance Computing at Shanghai Jiao Tong University.

\bibliographystyle{elsarticle-num-names}
\bibliography{dl}

\newpage

\appendix

\section{Experimental setups}

For Fig. \ref{fig:anisotropic},  Fig. \ref{fig:pca}, we use the FNN with size $784$-$50$-$50$-$10$ for the MNIST classification task. We train the network using GD with the first $10000$ images as the training set. We add a dropout layer behind the second layer. The dropout rate and learning rate are specified and unchanged in each experiment. We only consider the parameter matrix corresponding to the weight and the bias of the fully-connected layer between two hidden layers. Therefore, for experiments in Fig. \ref{fig:anisotropic}, $D=2500$. 

For Fig. \ref{fig:pca_further}(a, c, e, g), we add dropout layers after the convolutional layers, and for each dropout layer, $p=0.8$. We only consider the parameter matrix corresponding to the weight of the first convolutional layer of the first block of the ResNet-20. Models are trained using full-batch GD on the CIFAR100 classification task for $1200$ epochs. The learning rate is initialized at $0.01$. Since the Hessian calculation of ResNet takes much time, we only perform it at a specific dropout rate and learning rate.

For Fig. \ref{fig:pca_further}(b, d, f, h), we use transformer \cite{vaswani2017attention} with $d_{\mathrm{model}}=50, d_k=d_v=20, d_{\mathrm{ff}}=256, h=4, N=3$, the meaning of the parameters is consistent with the original paper. We only consider the parameter matrix corresponding to the weight of the fully-connected layer whose output is queried in the Multi-Head Attention layer of the first block of the decoder. We apply dropout to the output of each sub-layer before it is added to the sub-layer input and normalized. In addition, we apply dropout to the sums of the embeddings and the positional encodings in both the encoder and decoder stacks. For each dropout layer, $p=0.9$. For the English-German translation problem, we use the cross-entropy loss with label smoothing trained by full-batch Adam based on the Multi30k dataset. The learning rate strategy is the same as that in \cite{vaswani2017attention}. The warm-up step is $4000$ epochs, the training step is $10000$ epochs. We only use the first $2048$ examples for training to compromise with the computational burden.

\newpage

\section{Extended experiments on verifying the inverse flatness} \label{app:noise_further}

In this section, we verify the inverse relation between the covariance matrix and the Hessian matrix of dropout through different data collection methods and projection methods on larger network structures, such as ResNet-20 and transformer, and more complex datasets, such as CIFAR-100 and Multi30k, as shown in Fig. \ref{fig:pca_further}. 

\begin{figure}[h]
	\centering
	\subfigure[ResNet-20, $D=D_{\rm para}$]{\includegraphics[width=0.23\textwidth]{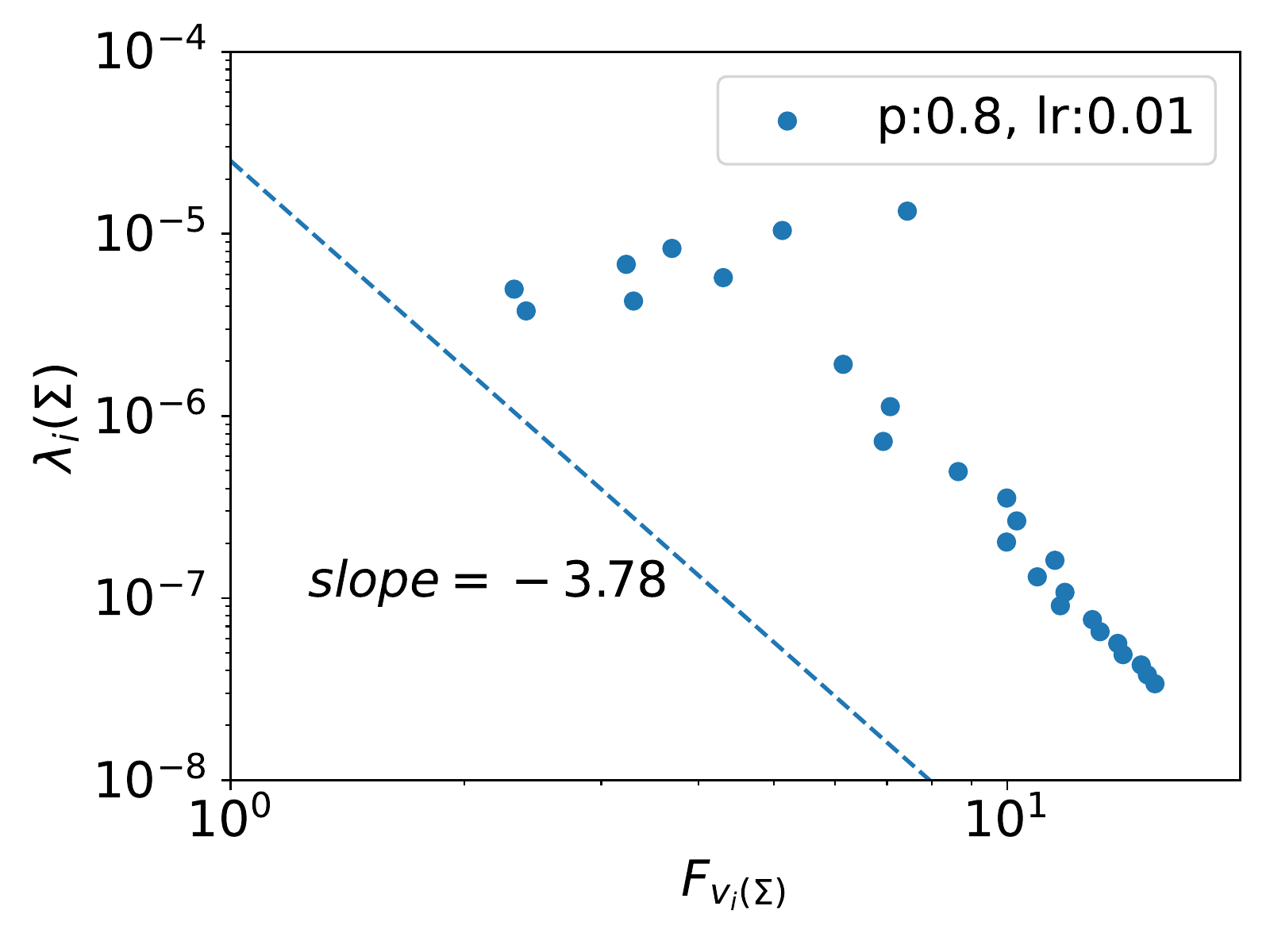}}
	\subfigure[Transformer, $D=D_{\rm para}$]{\includegraphics[width=0.23\textwidth]{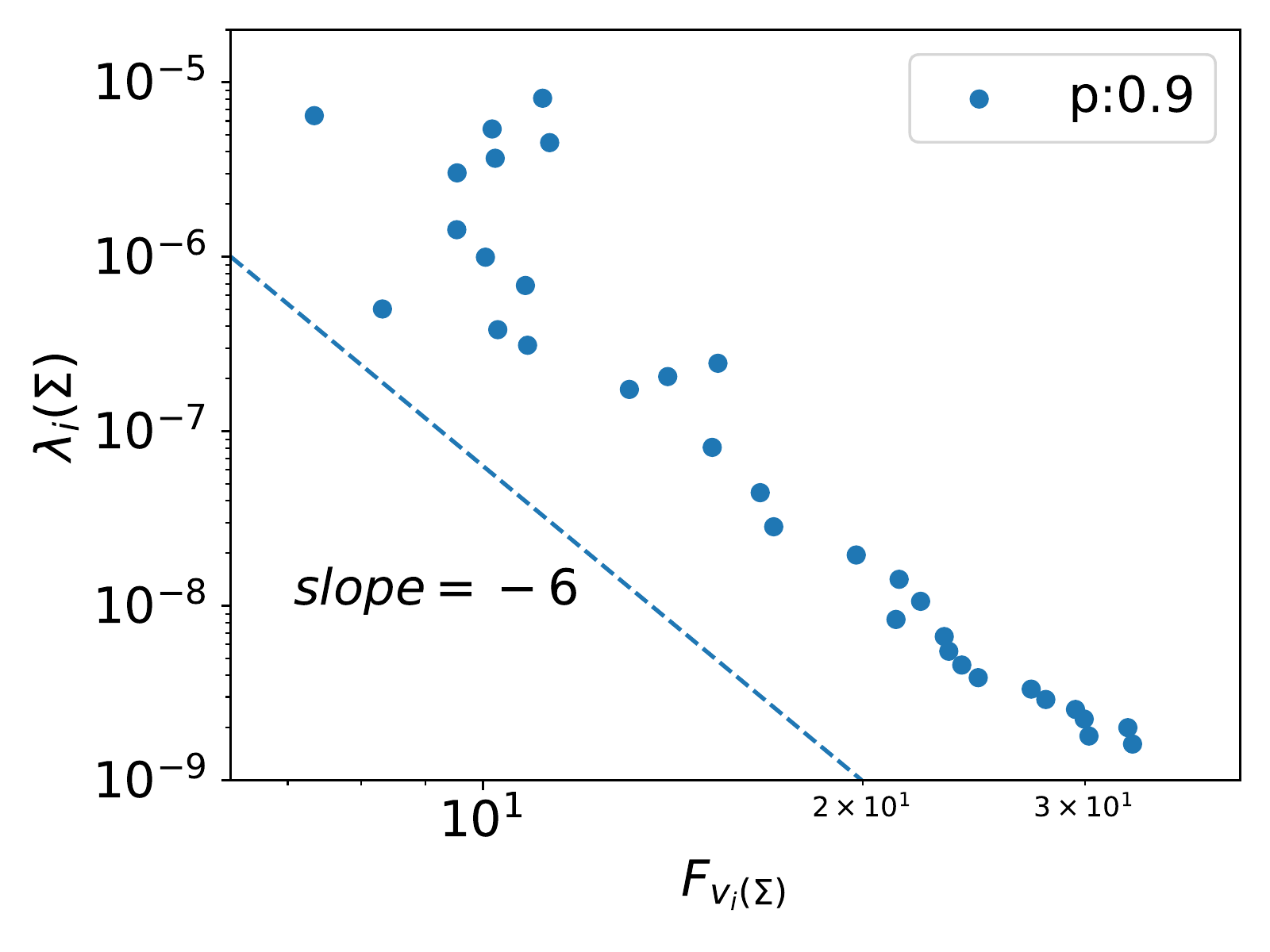}}
 \subfigure[ResNet-20, $D=D_{\rm grad}$]{\includegraphics[width=0.23\textwidth]{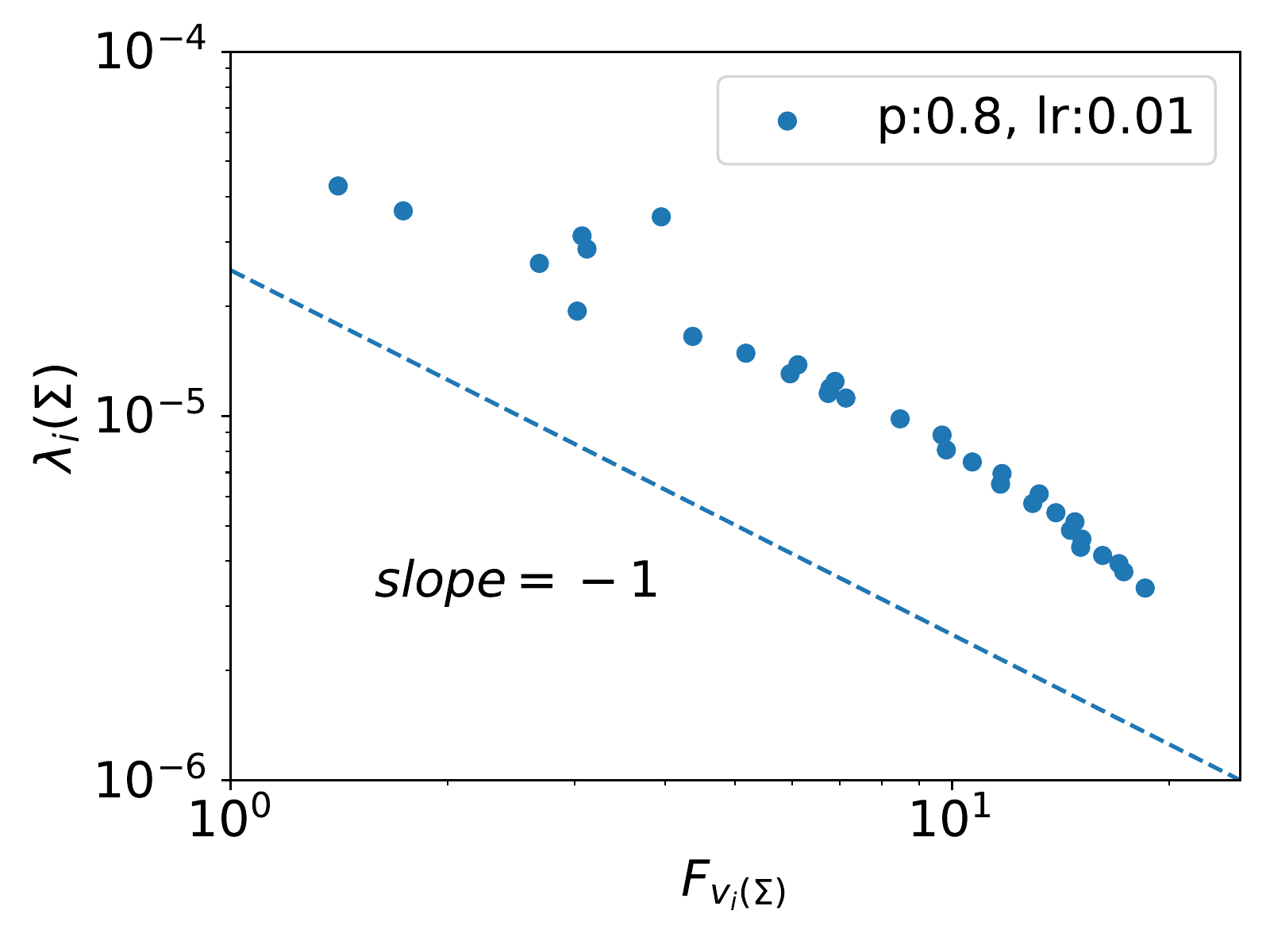}} 
\subfigure[Transformer, $D=D_{\rm grad}$]{\includegraphics[width=0.23\textwidth]{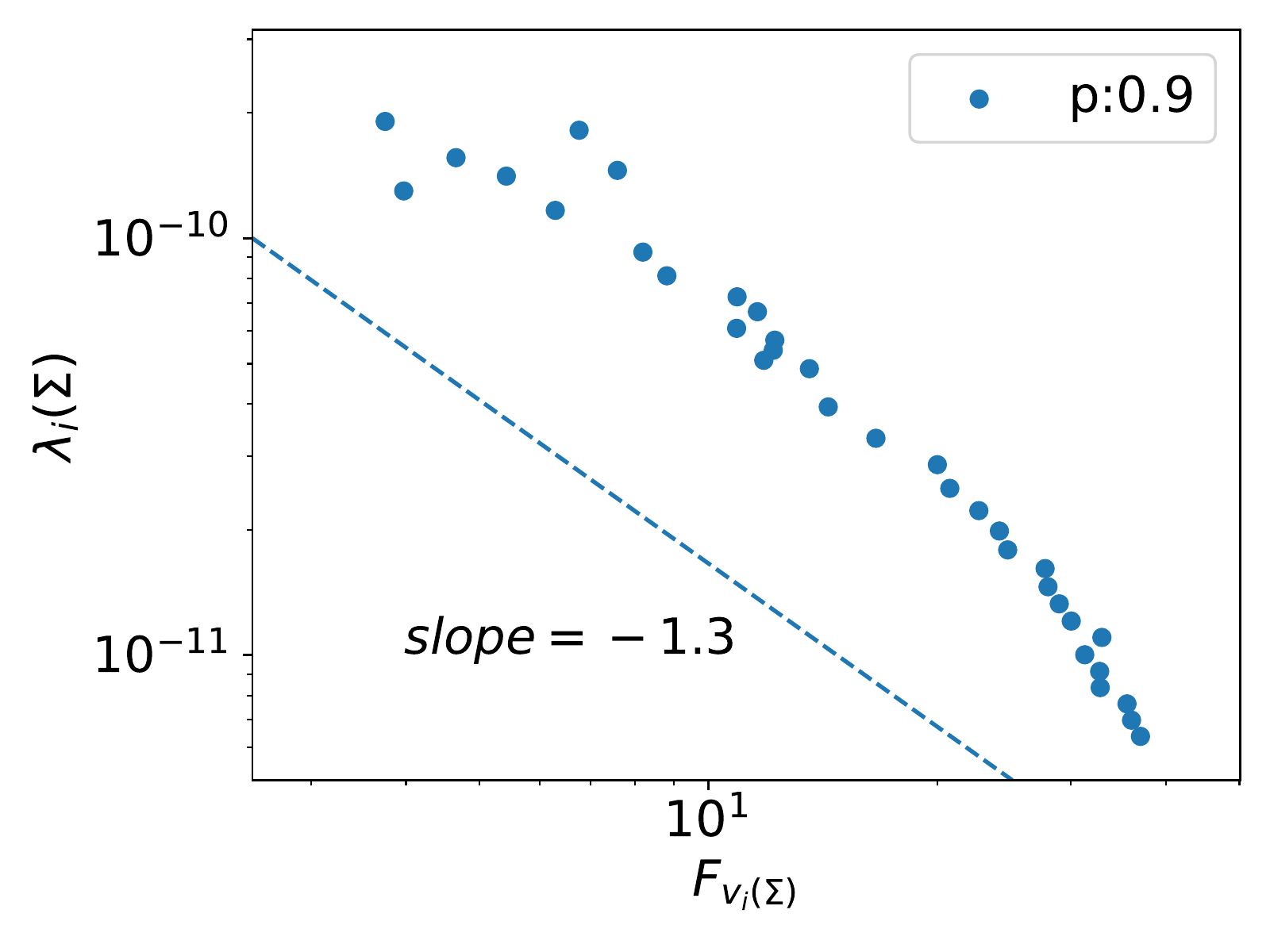}} \\
	\subfigure[ResNet-20, $D=D_{\rm para}$]{\includegraphics[width=0.23\textwidth]{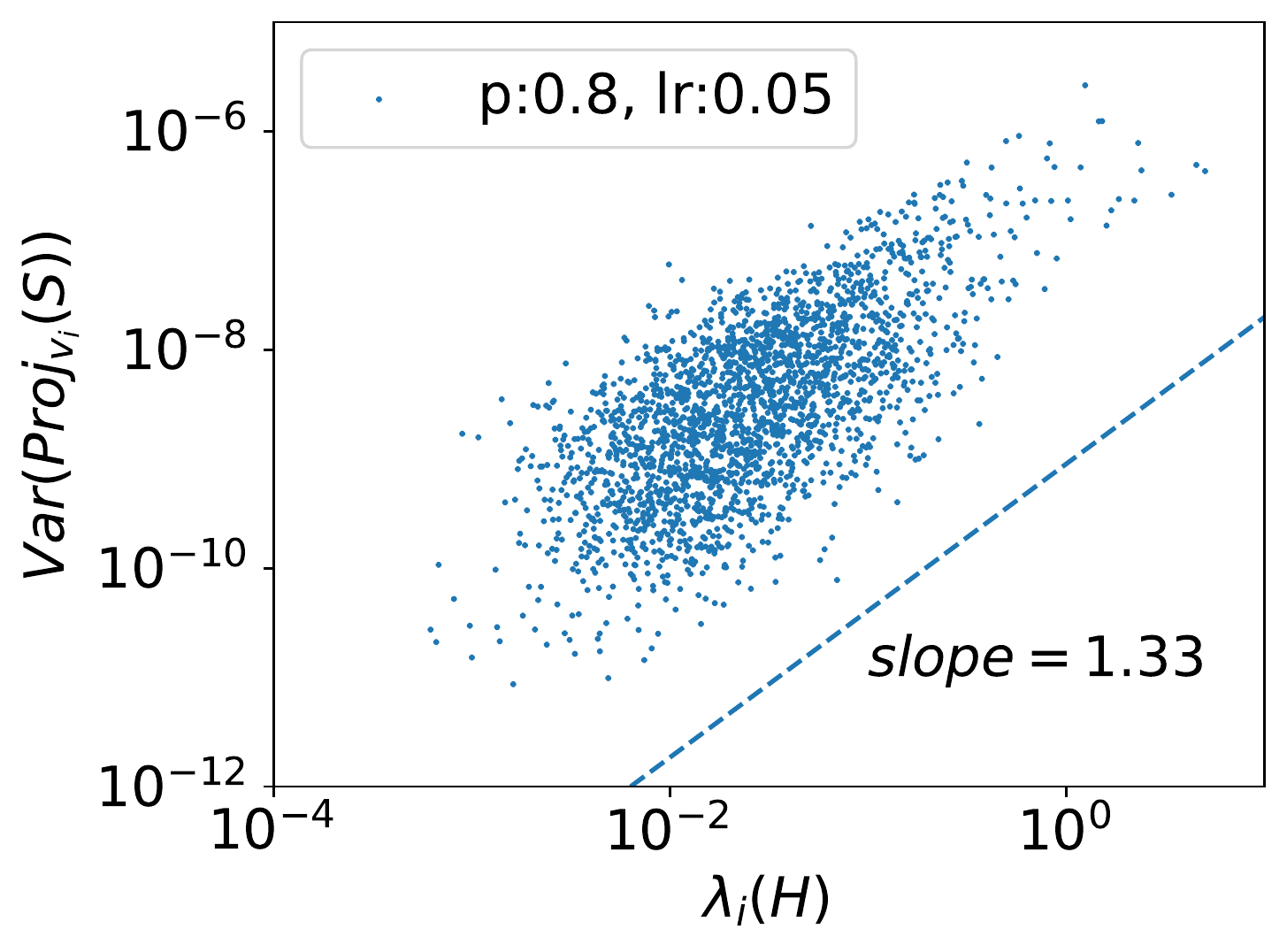}}
	\subfigure[Transformer, $D=D_{\rm para}$]{\includegraphics[width=0.23\textwidth]{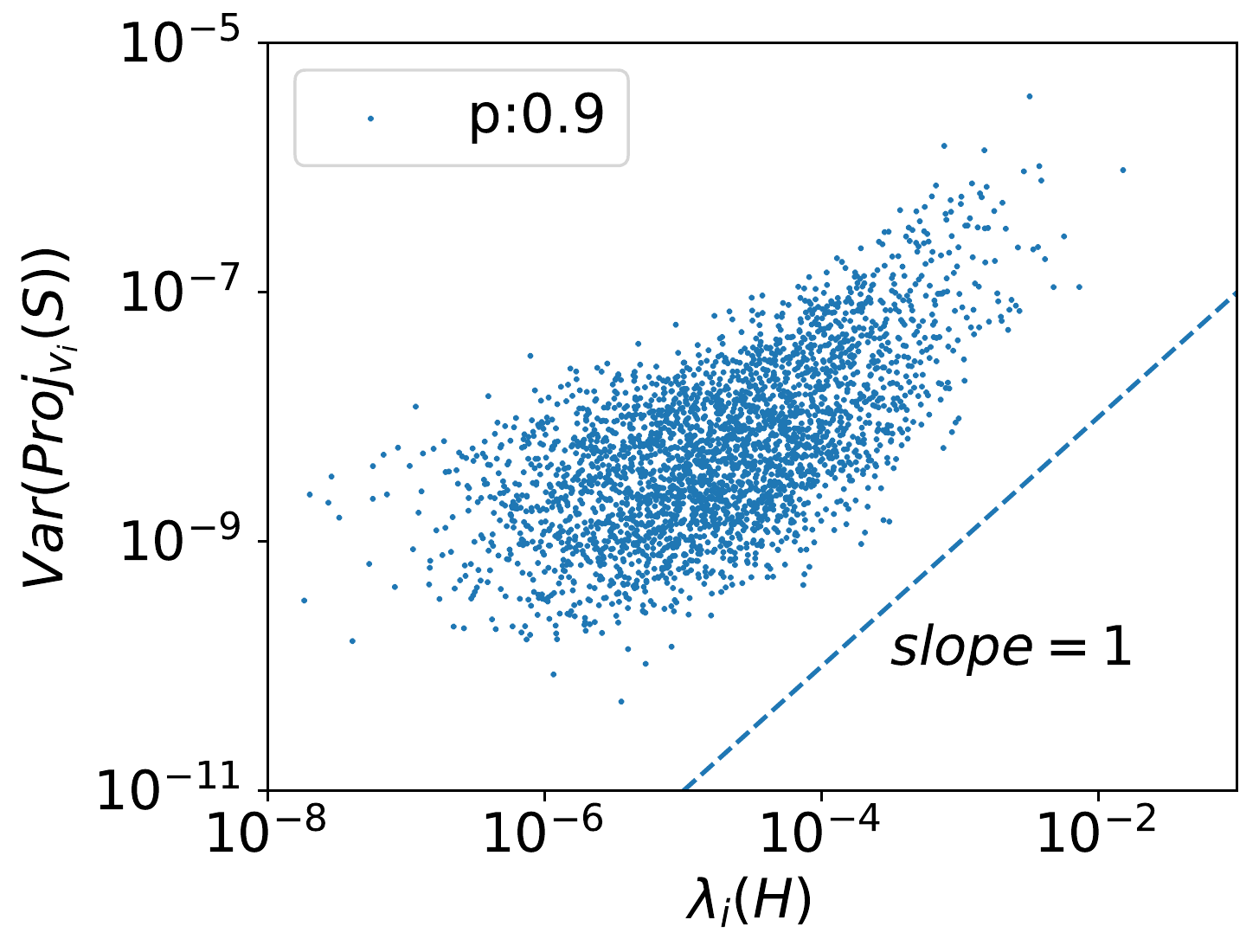}}
 	\subfigure[ResNet-20, $D=D_{\rm grad}$]{\includegraphics[width=0.23\textwidth]{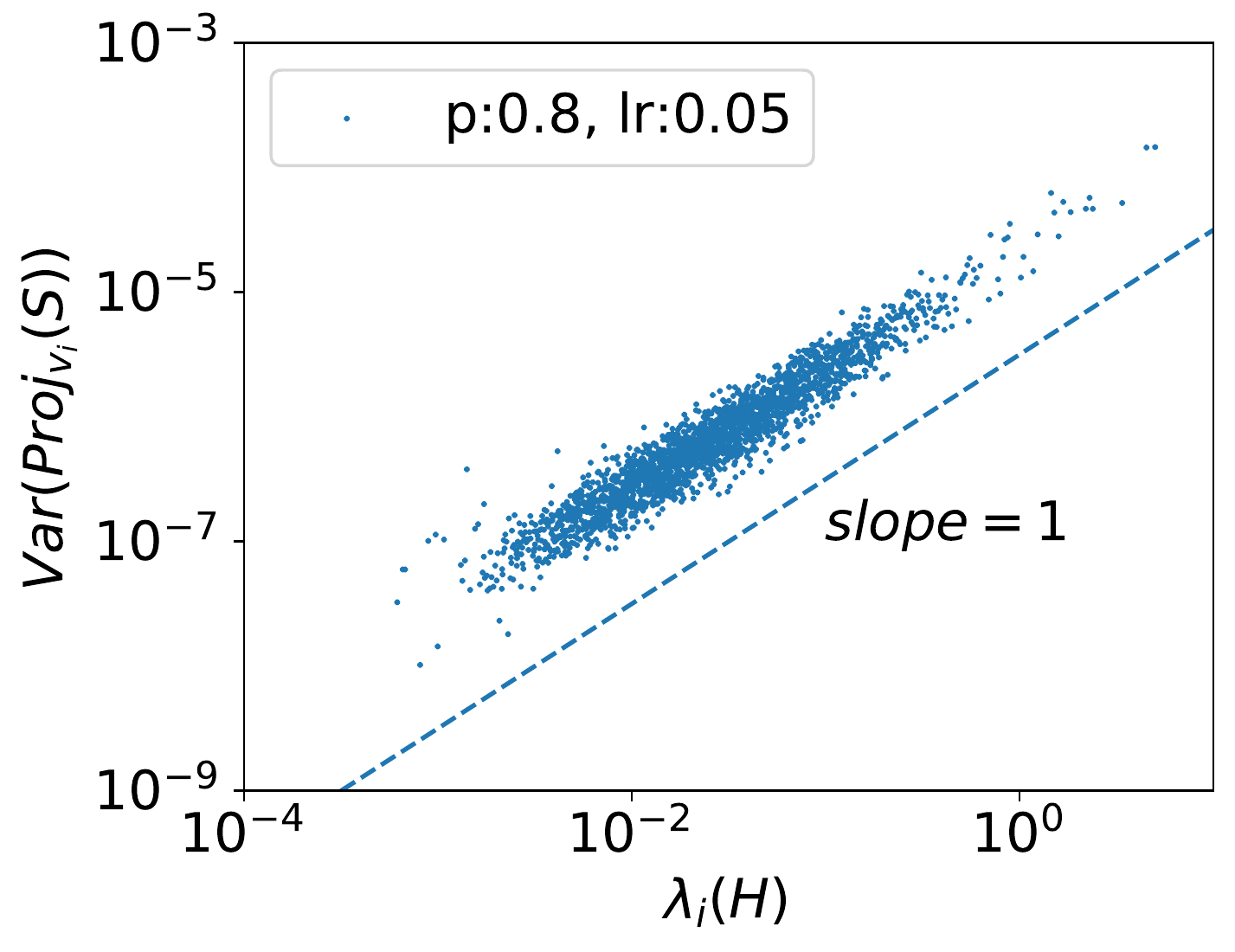}}
	\subfigure[Transformer, $D=D_{\rm grad}$]{\includegraphics[width=0.23\textwidth]{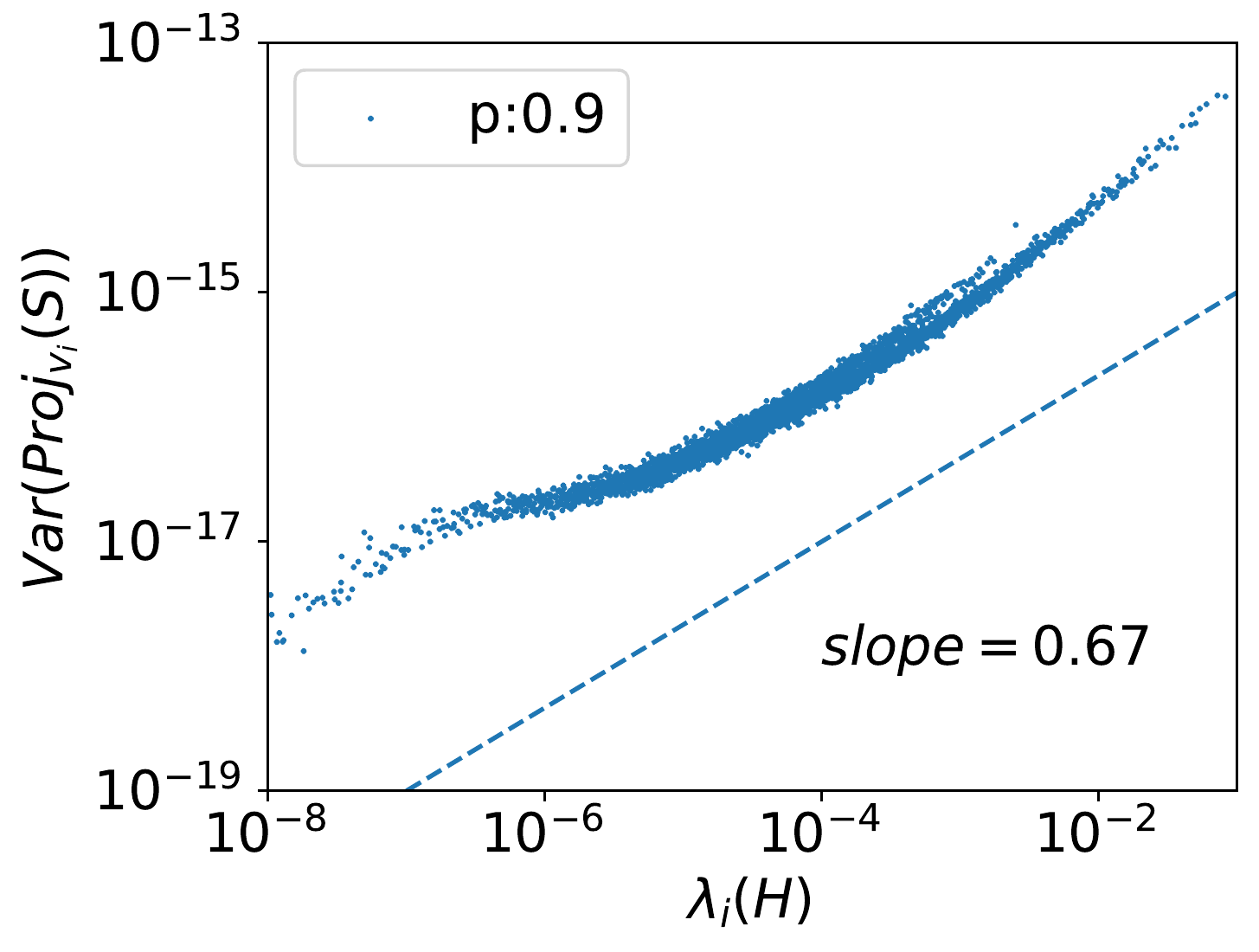}} 

  \caption{(a, b, c, d) The inverse relation between the variance $\{\lambda_{i}(\mSigma)\}_{i=1}^N$ and the interval flatness $\{F_{\vv_{i}(\mSigma)}\}_{i=1}^N$ for different choices of $p$ and learning rate $lr$ with different network structures. The PCA is done for different datasets $D$ sampled from parameters for the top line and sampled from gradients of parameters for the bottom line. The dashed lines give the approximate slope of the scatter. (e, f, g, h) The relation between the variance $\{{\rm Var}({\rm Proj}_{\vv_i(H)}(D))\}_{i=1}^N$ and the eigenvalue $\{\lambda_{i}(H)\}_{i=1}^{N}$ for different choices of $p$ and learning rate $lr$ with different network structures. The projection is done for different datasets $D$ sampled from parameters for the top line and sampled from gradients of parameters for the bottom line. The dashed lines give the approximate slope of the scatter. }
  \label{fig:pca_further}
\end{figure}

\newpage

\section{Preliminaries}\label{section....Preliminaries}
\subsection{Notations}\label{subsection...Notations}
We adhere wherever possible to the following notation.  Dimensional indices are written as subscripts with a bracket to avoid confusion with other sequential indices (e.g. time, iteration number), which do not have brackets.  When more than one indices are present, we separate them with a comma, e.g. $\vx_{k,(i)}$ is the $i$-th coordinate of the vector $\vx_k$, the $k^{\text {th }}$ member of a sequence. 

We set   a special vector $(1,1,1,\dots,1)^\T$ by $\vone:=(1,1,1,\dots,1)^\T$ whose dimension varies.  We set $n$ for the number of input samples, $m$ for the width of the neural network, and $D:=m(d+1)$ hereafter in this paper.
We let $[n]=\{1,2, \ldots, n\}$. We set   $\fN(\vmu, \mSigma)$ as the normal distribution with mean $\vmu$ and covariance $\mSigma$.
We denote $\otimes$ as the Kronecker tensor product,    $\left<\cdot,\cdot\right>$ for standard inner product between two vectors, and $ \mA : \mB $ for the Frobenius inner product between two matrices $\mA$ and $\mB$.
We denote vector $L^2$ norm as $\Norm{\cdot}_2$,  vector or function $L_{\infty}$ norm as $\Norm{\cdot}_{\infty}$, function $L_{1}$ norm as $\Norm{\cdot}_{1}$,  matrix infinity norm as $\Norm{\cdot}_{\infty\to\infty}$,  matrix spectral~(operator) norm as $\Norm{\cdot}_{2\to 2}$, and  matrix Frobenius norm as $\Norm{\cdot}_{\mathrm{F}}.$    
Finally,  we denote the set of continuous functions $f(\cdot):\sR^D\to\sR$ possessing   continuous derivatives of order up to and including $r$ by $\fC^{r}(\sR^D)$, and  for a Polish space $\fX$,  we denote     the space of bounded measurable functions by $\fB_b(\fX)$, and  the space of bounded continuous functions by $\fC_b(\fX)$. In the mathematical discipline of general topology, a Polish space is  a separable complete metric space. 
\subsection{Problem Setup}\label{subsection....ProblemSetup}
For the empirical risk minimization problem given by the quadratic loss:
\begin{equation}\label{eq...text...Prelim...LossFunction}
\min_{\vtheta}R_{\fS}(\vtheta)=\frac{1}{2n}\sum_{i=1}^n\left({f_{\vtheta}(\vx_i)-y_i}\right)^2,
\end{equation}
where $\fS:=\{ (\vx_i, y_i)\}_{i=1}^n$ is  the training sample,    $f_{\vtheta}(\vx)$ is the prediction function,  $\vtheta$ are the parameters to be optimized over, and their dependence is modeled by   a two-layer neural network~(NN) with $m$  hidden neurons
\begin{equation}\label{eq...text...Prelim...TwolayerNN}
    f_{\vtheta}(\vx) := \sum_{r=1}^{m}a_r\sigma(\vw_r^{\T}\vx),
\end{equation}
where $\vx\in\sR^{d}$,   $\vtheta=\mathrm{vec}(\vtheta_a,\vtheta_{\vw})$ with $\vtheta_a=\mathrm{vec}(\{a_r\}_{r=1}^{m})$, $\vtheta_{\vw}=\mathrm{vec}(\{\vw_r\}_{r=1}^{m})$ is the set of parameters,  $\sigma(\cdot)$ is the activation function applied coordinate-wisely to its input, and $\sigma$ is $1$-Lipschitz with $\sigma\in\fC^{\infty}(\sR)$. More precisely, $\vtheta=\mathrm{vec}(\{\vq_r\}_{r=1}^m)$ whereas for each $r\in[m]$, $\vq_r:=(a_r,\vw_r^{\T})^\T$. We remark that the bias term $b_r$ can be incorporated by expanding $\vx$ and $\vw_r$ to $(\vx^\T,1)^\T$ and $\left(\vw_r^\T,b_r\right)^\T$.

Given fixed learning rate $\eps>0$, then  at the $N$-th iteration, where   \[t_N:=N\eps,\] and a scaling vector $\veta_N \in \sR^{m}$ is sampled  with independent random coordinates: For each $k\in[m]$,
\begin{equation}\label{eq...text...Prelim...RandomVectorforDropout}
    (\veta_N)_{k}= \begin{cases}\frac{1}{p} & \text { with probability } p, \\ 0 & \text { with probability } 1-p, \end{cases}
\end{equation}
and we observe that $\{\veta_N\}_{N\geq 1}$ is an i.i.d.\ Bernulli sequence with $\Exp \veta_1=\vone$, and naturally, with slight abuse of notations, the $\sigma$-fields $\fF_N:=\left\{\sigma(\veta_1, \veta_2, \cdots \veta_N)\right\}$ forms a filtration.

We then apply dropout to   two-layer NNs by computing 
\begin{equation}\label{eq...text...Prelim...Dropout_Output}
\vf_{\vtheta}(\vx;\veta ):=\sum_{r=1}^m (\veta )_{r}a_r\sigma(\vw_r^{\T}\vx),
\end{equation}
and   we denote the empirical risk associated with dropout   by  
\begin{equation}\label{eq...text...Prelim...RS^drop} 
\begin{aligned}
R_{\fS}^\mathrm{drop}\left(\vtheta;\veta\right) :&= \frac{1}{2n}\sum_{i=1}^n\left(\vf_{\vtheta}(\vx_i;\veta)-y_i\right)^2\\
&=\frac{1}{2n}\sum_{i=1}^n\left(\sum_{r=1}^m (\veta)_{r}a_r\sigma(\vw_r^{\T}\vx_i)-y_i\right)^2.   
\end{aligned}
\end{equation}
We observe that  the parameters   at the $N$-th step are updated via back propagation as follows: 
\begin{equation}\label{eq...text...Prelim...Setup...ThetaUpdate...Abstract}
 \vtheta_N= \vtheta_{N-1}-\eps \nabla_{\vtheta}R_{\fS}^\mathrm{drop}\left(\vtheta_{N-1}; \veta_{N}\right),
\end{equation}
where $\vtheta_0:=\vtheta(0)$. Finally,  we denote hereafter that  for all $i\in[n]$, 
\[
e_{i}^N :=e_i(\vtheta_{N-1};\veta_N) := \vf_{\vtheta_{N-1}}(\vx_i; \veta_N) - y_i,   
\]
hence the empirical risk associated with dropout 
 $R_{\fS}^\mathrm{drop}\left(\vtheta_{N-1}; \veta_N\right)$ can be written into 
\[R_{\fS}^\mathrm{drop}\left(\vtheta_{N-1}; \veta_N\right)=\frac{1}{2n}\sum_{i=1}^n\left(e_{i}^N\right)^2,\] 
thus the dropout iteration  
\eqref{eq...text...Prelim...Setup...ThetaUpdate...Abstract}  reads
\begin{align*}
\vtheta_N- \vtheta_{N-1}&= -\eps \nabla_{\vtheta}R_{\fS}^\mathrm{drop}\left(\vtheta_{N-1}; \veta_{N}\right)=-\frac{\eps}{n}\sum_{i=1}^ne_{i}^N \nabla_{\vtheta}e_{i}^N,
\end{align*}    
and we may proceed to the  introduction of the stochastic modified equation~(SME) approximation.  

\newpage

\section{Stochastic Modified Equations for Dropout}\label{section...ModifiedEquations}
\subsection{Modified Loss}\label{subsection...SME...ModifiedLoss}
Recall that the parameters at the $N$-th step are updated   as follows: 
\begin{equation}\label{eq...text...SME...ModifiedLoss...ThetaUpdate...Abstract...appendix}
 \vtheta_N= \vtheta_{N-1}-\frac{\eps}{n}\sum_{i=1}^ne_{i}^N \nabla_{\vtheta}e_{i}^N,
\end{equation}
and since $\{\veta_N\}_{N\geq 1}$ is an i.i.d.\  sequence, then the dropout iteration  \eqref{eq...text...SME...ModifiedLoss...ThetaUpdate...Abstract...appendix} updates  the  parameters   in a recursion form of 
\begin{equation}\label{eq...text...SME...ModifiedLoss...ThetaUpdate...RecursionForm}
\vtheta_{N}=\vF(\vtheta_{N-1},\veta_N),
\end{equation}
where $\vF(\cdot,\cdot):\sR^D\times \sR^m\to \sR^D $ is a smooth~($\fC^{\infty}$) function, and   $\{\veta_N\}_{N\geq 1}$ is a disturbance sequence on $\sR^m$, whose marginal distribution possesses a density supported on an open subset of   $\sR^m$.
Then, based on the results in~\cite{meyn2012markov},  the dropout iterations \eqref{eq...text...SME...ModifiedLoss...ThetaUpdate...Abstract...appendix} forms a  time-homogeneous Markov chain.
Thus, we may misuse $\Exp[\cdot \mid \fF_{N}]$, the conditional expectation given $\fF_{N}$,  with  $\Exp_{\vtheta_{N-1}}[\cdot]$,  the  conditional expectation given $\vtheta_{N-1}$.  Then, for each $k\in [m]$,   the  conditional expectation    of the  increment  restricted to $\vq_k$ reads
\begin{align*}
\Exp_{\vtheta_{N-1}}\left[\sum_{i=1}^ne_{i}^N \nabla_{\vq_k}e_{i}^N \right]&=\Exp_{\vtheta_{N-1}}\left[\sum_{i=1}^ne_{i}^N(\veta_N)_{k}\nabla_{\vq_k}\left(a_k\sigma(\vw_k^{\T}\vx_i)\right) \right],
\end{align*}
and since   
\begin{align*}
\Exp_{\vtheta_{N-1}}\left[ e_{i}^N(\veta_N)_{k} \right] &=  \Exp_{\vtheta_{N-1}}\left [\sum_{r=1, r\neq k
}^m (\veta_N)_{r} a_{r}\sigma(\vw_{r}^{\T}\vx_i)-y_i\right] \Exp_{\vtheta_{N-1}}\left[  (\veta_N)_{k} \right]\\
&~~+\Exp_{\vtheta_{N-1}}\left[(\veta_N)_k^2\right] a_{k}\sigma(\vw_{k}^{\T}\vx_i)\\
&= \left(\sum_{r=1, r\neq k}^m a_{r}\sigma(\vw_{r}^{\T}\vx_i)-y_i\right)+\frac{1}{p}a_{k}\sigma(\vw_{k}^{\T}\vx_i)\\
&= \left(\sum_{r=1}^m a_{r}\sigma(\vw_{r}^{\T}\vx_i)-y_i\right)+\left(\frac{1}{p}-1\right)a_{k}\sigma(\vw_{k}^{\T}\vx_i).
\end{align*}
For simplicity,     given fixed $k\in[m]$,  for any $i\in[n]$, we denote hereafter that
\begin{align*}
e_{i}&:=e_{i}(\vtheta):=\sum_{r=1}^m a_{r}\sigma(\vw_{r}^{\T}\vx_i)-y_i,\\
e_{i, \backslash k}&:=e_{i, \backslash k}(\vtheta):=\sum_{r=1, r\neq k}^m a_{r}\sigma(\vw_{r}^{\T}\vx_i)-y_i,
\end{align*}
we remark that  compared with  $e_{i}^N$, $e_i$ and  $e_{i, \backslash k}$  do not depend on the random variable $\veta_N$. Then $\Exp_{\vtheta_{N-1}}\left( e_{i}^N(\veta_N)_{k}\right)$ can be written in short by
\begin{equation}\label{eq...text...SME...ModifiedLoss...FirstMoment}
\begin{aligned}
\Exp_{\vtheta_{N-1}}\left [ e_{i}^N(\veta_N)_{k}\right]&=e_{i, \backslash k}+  \frac{1}{p} a_{k}\sigma(\vw_{k}^{\T}\vx_i)\\
&= e_i+  \left(\frac{1}{p}-1\right)a_{k}\sigma(\vw_{k}^{\T}\vx_i).
\end{aligned}
\end{equation}
Hence for each $k\in [m]$, expectation of the  increment  restricted to $\vq_k$ reads
\begin{align*}
&\Exp_{\vtheta_{N-1}}\left[\sum_{i=1}^ne_{i}^N(\veta_N)_{k}\nabla_{\vq_k}\left(a_k\sigma(\vw_k^{\T}\vx_i)\right) \right] \\
=&\sum_{i=1}^ne_i\nabla_{\vq_k}\left(a_k\sigma(\vw_k^{\T}\vx_i)\right) +\sum_{i=1}^n \left(\frac{1}{p}-1\right)a_{k}\sigma(\vw_{k}^{\T}\vx_i) \nabla_{\vq_k}\left(a_k\sigma(\vw_k^{\T}\vx_i)\right),
\end{align*}    
then we define the {\emph{modified loss}} $L_{\fS}(\cdot):\sR^{m(d+1)}\to\sR$ for dropout:
\begin{equation}\label{eq...text...SME...ModifiedLoss...ModifiedLoss...appendix}
\begin{aligned}
L_{\fS}(\vtheta)&:=\frac{1}{2n}\sum_{i=1}^ne_i^2 +\frac{1-p}{2np}\sum_{i=1}^n \sum_{r=1}^m a_{r}^2\sigma(\vw_{r}^{\T}\vx_i)^2,
\end{aligned}
\end{equation}
since as $\vtheta_{N-1}$ is  given, then by taking the conditional expectation, increment of  the dropout iteration \eqref{eq...text...SME...ModifiedLoss...ThetaUpdate...Abstract...appendix} reads
\begin{align*}
\vtheta_N- \vtheta_{N-1}&= -\eps \Exp_{\vtheta_{N-1}}\left[\nabla_{\vtheta}R_{\fS}^\mathrm{drop}\left(\vtheta_{N-1}; \veta_{N}\right)\right]=-\eps\nabla_{\vtheta}L_{\fS}(\vtheta)\big|_{\vtheta=\vtheta_{N-1}},
\end{align*}    
which implies that  in the sense of expectations, $\{\vtheta_N\}_{N\geq 0}$ follows close to the gradient descent trajectory of  $L_{\fS}(\vtheta)$ with fixed learning rate $\eps$.
\subsection{Stochastic Modified Equations}\label{subsection...SME...ModifiedEquaitons}
We then follow the strategy of \cite{li2017stochastic} to derive the stochastic modified equations~(SME) for dropout.  Firstly, from the results in Section \ref{subsection...SME...ModifiedLoss}, we observe that given $\vtheta_{N-1}$,
\begin{equation}\label{eq...text...SME...SME...DiscreteUpdate+Covariance...appendix}
\vtheta_N-\vtheta_{N-1} = -\eps\nabla_{\vtheta}L_{\fS}(\vtheta)\big|_{\vtheta=\vtheta_{N-1}}+\sqrt{\eps}\vV(\vtheta_{N-1}),    
\end{equation}
where $L_{\fS}(\cdot):\sR^{m(d+1)}\to\sR$ is the modified loss defined in \eqref{eq...text...SME...ModifiedLoss...ModifiedLoss...appendix}, and $\vV(\cdot):\sR^{m(d+1)}\to\sR^{m(d+1)}$ is a $m(d+1)$-dimensional random vector, and when  given $\vtheta_{N-1}$, $\vV(\vtheta_{N-1})$ has mean $\vzero$ and covariance  $\eps \mSigma(\vtheta_{N-1})$, where $\mSigma (\cdot):\sR^{m(d+1)}\to \sR^{m(d+1) \times {m(d+1)}}$ is the  covariance of $\nabla_{\vtheta}R_{\fS}^\mathrm{drop}\left(\vtheta_{N-1}; \veta_{N}\right)$. Recall that  $\vtheta=\mathrm{vec}(\{\vq_r\}_{r=1}^m)=\mathrm{vec}\left(\{(a_r, \vw_r)\}_{r=1}^m\right)$, and for any $k, r\in[m]$, we denote  that
 \[\mSigma_{kr}(\vtheta_{N-1}):=\mathrm{Cov}\left( \nabla_{\vq_k}R_{\fS}^\mathrm{drop}\left(\vtheta_{N-1}; \veta_{N}\right), \nabla_{\vq_r}R_{\fS}^\mathrm{drop}\left(\vtheta_{N-1}; \veta_{N}\right)\right), \]
then 
 \[
\mSigma=\left[\begin{array}{cccc}
 \mSigma_{11} &  \mSigma_{12} &  \cdots & \mSigma_{1m}  \\
\mSigma_{21} &  \mSigma_{22} &  \cdots & \mSigma_{2m}\\ 
\vdots& \vdots&\vdots&\vdots\\
\mSigma_{m1} &  \mSigma_{m2} &  \cdots & \mSigma_{mm}
\end{array}\right].
\]
 For each  $  k \in [m]$, we obtain that 
\begin{align*}
&\mSigma_{kk}(\vtheta_{N-1})=\mathrm{Cov}\left( \nabla_{\vq_k}R_{\fS}^\mathrm{drop}\left(\vtheta_{N-1}; \veta_{N}\right), \nabla_{\vq_k}R_{\fS}^\mathrm{drop}\left(\vtheta_{N-1}; \veta_{N}\right)\right) \\
=&\left(\frac{1}{p}-1\right)\left(\frac{1}{n}\sum_{i=1}^n\left( e_{i,\backslash k}+\frac{1}{p}a_{k}\sigma(\vw_{k}^{\T}\vx_i)\right)\nabla_{\vq_k}\left(a_k\sigma(\vw_k^{\T}\vx_i)\right) \right)\\
&~~~~~~~~~~~~~~~~~~~~~~~~~~~~~~{\otimes}\left(\frac{1}{n}\sum_{i=1}^n\left( e_{i,\backslash k}+\frac{1}{p}a_{k}\sigma(\vw_{k}^{\T}\vx_i)\right)\nabla_{\vq_k}\left(a_k\sigma(\vw_k^{\T}\vx_i)\right) \right)\\
&~+\left(\frac{1}{p^2}-\frac{1}{p}\right)\sum_{l=1, l\neq k}^m\left(\frac{1}{n}\sum_{i=1}^na_{l}\sigma(\vw_{l}^{\T}\vx_i)\nabla_{\vq_k}\left(a_k\sigma(\vw_k^{\T}\vx_i)\right)\right)\\
&~~~~~~~~~~~~~~~~~~~~~~~~~~~~~~{\otimes}\left(\frac{1}{n}\sum_{i=1}^na_{l}\sigma(\vw_{l}^{\T}\vx_i)\nabla_{\vq_k}\left(a_k\sigma(\vw_k^{\T}\vx_i)\right)\right),
\end{align*}
and for each $  k, r \in [m]$ with $k \neq r$,  
\begin{align*}
 &\mSigma_{kr}(\vtheta_{N-1}) = \mathrm{Cov}\left( \nabla_{\vq_k}R_{\fS}^\mathrm{drop}\left(\vtheta_{N-1}; \veta_{N}\right), \nabla_{\vq_r}R_{\fS}^\mathrm{drop}\left(\vtheta_{N-1}; \veta_{N}\right)\right) \\
=&\left(\frac{1}{p}-1\right)\left(\frac{1}{n}\sum_{i=1}^n\left(e_{i,\backslash k, \backslash r}+\frac{1}{p}a_k\sigma(\vw_k^\T\vx_i)+\frac{1}{p}a_r\sigma(\vw_r^\T\vx_i)\right)\nabla_{\vq_k}\left(a_k\sigma(\vw_k^{\T}\vx_i)\right)\right)\\
&~~~~~~~~~~~~~~~~{\otimes}\left(\frac{1}{n}\sum_{i=1}^na_k\sigma(\vw_k^\T\vx_i)\nabla_{\vq_r}\left(a_r\sigma(\vw_r^{\T}\vx_i)\right)\right)\\
&+\left(\frac{1}{p}-1\right)\left(\frac{1}{n}\sum_{i=1}^na_r\sigma(\vw_r^\T\vx_i)\nabla_{\vq_k}\left(a_k\sigma(\vw_k^{\T}\vx_i)\right)\right) \\
&~~~~~~~~~~~~~~~~{\otimes}\left(\frac{1}{n}\sum_{i=1}^n\left(e_{i,\backslash k, \backslash r}+a_k\sigma(\vw_k^\T\vx_i)+\frac{1}{p}a_r\sigma(\vw_r^\T\vx_i)\right)\nabla_{\vq_r}\left(a_r\sigma(\vw_r^{\T}\vx_i)\right)\right),
\end{align*}
where we denote hereafter that
\begin{align*}
e_{i,\backslash k, \backslash r}:=e_{i,\backslash k,\backslash r}(\vtheta):=\sum_{l=1, l\neq k, l\neq r
 }^m a_{l}\sigma(\vw_{l}^{\T}\vx_i)-y_i,
\end{align*}
and  compared with  $e_{i}^N$, $e_{i,\backslash k, \backslash r} $  still does not depend on the random variable $\veta_N$. 
  We remark that the expression above is consistent in that for the extreme case where $p=1$, dropout `degenerates' to gradient descent~(GD), hence the covariance matrix degenerates to a zero matrix, i.e., $\mSigma=\mzero_{D\times D}$.  We remark that details for the derivation of   $\mSigma$  is deferred   to Section \ref{section...ComputationalDetails}.

Now, as we consider the stochastic differential equation~(SDE),  
\begin{equation}\label{eq...text...SME...SME...SDE...Abstract...appendix}
\D \vTheta_t=\vb \left(\vTheta_t\right)\D t+\vsigma \left(\vTheta_t\right) \D \vW_t, \quad  \vTheta_0=\vTheta(0),
\end{equation}
where $\vW_t$ is a standard $m(d+1)$-dimensional standard Wiener process, whose Euler–Maruyama discretization  with step size $\eps>0$  at the $N$-th step reads
\begin{equation*}
  \vTheta_{\eps N}=  \vTheta_{\eps(N-1)}+\eps\vb \left(\vTheta_{\eps(N-1)}\right)  +\sqrt{\eps}\vsigma \left(\vTheta_{\eps(N-1)}\right) \vZ_{N}, 
\end{equation*}
where $\vZ_N\sim\fN(\vzero, \mI_{m(d+1)})$ and $\vTheta_0=\vTheta(0)$. Thus, if we set 
\begin{equation}
\begin{aligned}
\vb\left(\vTheta\right)&:=   -\nabla_{\vTheta}L_{\fS}(\vTheta),\\     
\vsigma\left(\vTheta\right)&:=\sqrt{\eps}\left(\mSigma\left(\vTheta\right)\right)^{\frac{1}{2}},\\
\vTheta_0&:=\vtheta_0,
\end{aligned}
\end{equation}
then we would expect \eqref{eq...text...SME...SME...SDE...Abstract...appendix} to be a `good' approximation of \eqref{eq...text...SME...SME...DiscreteUpdate+Covariance...appendix} with the time identification $t=\eps N$. Based on the earlier work of \cite{li2017stochastic}, since the path  of dropout and the counterpart of SDE are driven by   noises sampled in different spaces. Firstly, notice that the stochastic process $\left\{\vtheta_N\right\}_{N\geq 0}$ induces a probability measure on the product space $\sR^D \times \sR^D \times \cdots\times\sR^D\times \cdots$, whereas $\left\{\vTheta_t\right\}_{t\geq 0}$ induces a probability measure on $\fC\left([0, \infty), \mathbb{R}^D\right)$. To compare them, one can form a piece-wise linear interpolation of the former. Alternatively, as we do in this work, we sample a discrete number of points from the latter. Secondly, the process $\left\{\vtheta_N\right\}_{N\geq 0}$ is adapted to the filtration generated by  $\fF_N$ whereas the process $\left\{\vTheta_t\right\}_{t\geq 0}$  is adapted to an independent  Wiener filtration $\fF_t$. Hence, it is not appropriate to compare individual sample paths. Rather, we define below a sense of {\emph{weak}} approximations~\citep[Section 9.7]{kloeden2011numerical} by comparing the distributions of the two processes.

To compare different discrete  time approximations, we need to take the rate of weak convergence into consideration, and we also  need to choose  an appropriate  class of functions as the space of test functions.
We introduce the following set of smooth functions:
\begin{equation*}
\fC_b^M\left(\sR^{m(d+1)}\right)=\left\{ f \in \fC^M\left(\sR^{m(d+1)}\right) \Bigg|\Norm{f}_{\fC^M}:=\sum_{|\beta| \leq M}  \Norm{\mathrm{D}^\beta f }_{\infty}<\infty \right\},
\end{equation*}
where $\mathrm{D}$ is the usual differential operator.
We remark that $\fC_b^M(\sR^D)$ is a subset of $\fG(\sR^D)$, the class of functions with polynomial growth, which is chosen to be the space of test functions in previous works \citep{li2017stochastic,kloeden2011numerical,malladi2022SDE}.

Before we proceed to the definition of weak approximation, to ensure the rigor and validity of our analysis, we shall assert an assumption regarding the existence and uniqueness   of solutions to the SDE \eqref{eq...text...SME...SME...SDE...Abstract...appendix}.
\begin{assumption}\label{Append...theOnlyAssumption}
There exists $T^{\ast}>0$, such that for any time $t\in\left[0, T^{\ast} \right]$, there exists a unique $t$-continuous solution $\vTheta_t $ of the initial value problem:  
\[
\D \vTheta_t=\vb\left(\vTheta_t\right)\D t+\vsigma \left(\vTheta_t\right) \D \vW_t, \quad  \vTheta_0=\vTheta(0),
\]
with the property that $\vTheta_t$ is adapted to the filtration $\mathcal{F}_t$ generated by   $ \vW_s $ for  all time $ s \leq t$. 
Furthermore,  for any $t\in[0, T^{\ast}]$,
\[
\Exp\int_{0}^t\Norm{\vTheta_s(\cdot)}_2^2\D s< \infty.
\]
Moreover, we assume 
that the second, fourth and sixth moments   of  the solution  to SDE \eqref{eq...text...SME...SME...SDE...Abstract...appendix}  are uniformly bounded with respect to time $t$, i.e.,  for each $l\in[3]$, there exists $C(T^{\ast},\vTheta_0)>0$, such that  
\begin{equation}\label{eq...assump...UniformBDD...Cts...appendix}
\sup_{{0}\leq s\leq T^{\ast}}\Exp \Norm{\vTheta_s(\cdot)}_2^{2l} \leq  C(T^{\ast},\vTheta_0).
\end{equation}
As  for the dropout iterations \eqref{eq...text...SME...ModifiedLoss...ThetaUpdate...Abstract...appendix},   we assume  further
that the second, fourth and sixth moments   of  the dropout iterations \eqref{eq...text...SME...ModifiedLoss...ThetaUpdate...Abstract...appendix}  are uniformly bounded with respect to the number of  iterations $N$, i.e.,   let $0<\eps<1$, $T>0$ and set $N_{T,\eps}:=\lfloor \frac{T}{\eps} \rfloor$,    then for each $l\in[3]$, there exists $T^{\ast}>0$ and $\eps_0>0$, such that for  any given learning rate  $\eps\leq \eps_0$ and   all $N\in[0:N_{T^{\ast},\eps}]$, there exists $C(T^{\ast},\vtheta_0,\eps_0)>0$, such that 
\begin{equation}\label{eq...assump...UniformBDD...Discrete...appendix}
\sup_{{0}\leq N\leq [N_{T^{\ast},\eps}]}\Exp \Norm{\vtheta_N}_2^{2l} \leq  C(T^{\ast},\vtheta_0,\eps_0).
\end{equation}
\end{assumption}
\noindent
We remark that if  $\fG(\sR^D)$ is chosen to be  the test functions in~\cite{li2019stochastic},  then   similar     relations to  \eqref{eq...assump...UniformBDD...Cts...appendix} and  \eqref{eq...assump...UniformBDD...Discrete...appendix} shall be imposed, except that in our cases, we only require the second, fourth and sixth moments to be uniformly bounded, while  in their cases,  all  $2l$-moments are required for $l\geq 1$.Establishments of  the validity of Assumption \ref{Append...theOnlyAssumption} regarding the existence and uniqueness of the SDE  will be exhibited in Section \ref{section...ValidationforAssump}.

The definition of weak approximation is stated out as follows.
\begin{definition}
The SDE \eqref{eq...text...SME...SME...SDE...Abstract...appendix} is an order $\alpha$ weak approximation to the dropout \eqref{eq...text...SME...ModifiedLoss...ThetaUpdate...Abstract...appendix}, if for every $g\in \fC_b^M\left(\sR^{m(d+1)}\right)$, there exists $C>0$ and $\eps_0>0$,  such that     given any $\eps\leq\eps_0$ and $T\leq T^{\ast}$, then for all $N\in[N_{T,\eps}]$,
\begin{equation}\label{eq...definition...WeakApproximationAlpha}
\Abs{\Exp  g(\vTheta_{\eps N})  -\Exp  g(\vtheta_N)   }  \leq C(T, g, \eps_0)\eps^{\alpha}. 
\end{equation}
\end{definition}

\newpage

\section{Semigroup and  Proof Details for the Main Theorem}\label{section...Semigroup}
In this section, we use a semigroup approach~\citep{feng2018semigroups}  to study the time-homogeneous Markov chains~(processes) formed by dropout.  
\subsection{Discrete and Continuous Semigroup}\label{subsection...Semigroup...DisCtsSemigroup}
\begin{definition}\label{defi...MarkovOperator}
A {\emph{Markov operator}} over a Polish space $\fX$    is a bounded linear operator $\fP:\fB_b(\fX)\to\fB_b(\fX)$ satisfying
\begin{itemize}
\item $\fP \mathbf{1}=\mathbf{1}$;  
\item  $\mathcal{P} \varphi$ is positive whenever $\varphi$ is positive;
\item  If a sequence $\left\{\varphi_n\right\} \subset \mathcal{B}_b(\mathcal{X})$ converges pointwise to an element $\varphi \in \mathcal{B}_b(\mathcal{X})$, then $\mathcal{P} \varphi_n$ converges pointwise to $\mathcal{P} \varphi$;
\end{itemize}
\end{definition}
\noindent To demonstrate further inequalities that Markov operators satisfy, we
offer the following proposition
\begin{proposition}\label{prop...Contractive}
A  {Markov operator}  $\fP:\fB_b(\fX)\to\fB_b(\fX)$ over a Polish space $\fX$       satisfies
\begin{itemize}
\item  $(\fP f(\vx))^{+} \leq \fP f^{+}(\vx)$;  
\item  $(\fP f(\vx))^{-} \leq \fP f^{-}(\vx)$;
\item  $|\fP f(\vx)| \leq \fP|f(\vx)|$.
\end{itemize}
Moreover, if the Polish space $\fX$ is equipped with a measure $\mu$,    a function $f: \fX \rightarrow \sR$   is said to be an element of $\fL^1(\fX)$ if
$$
\int_\fX|f| \D \mu<\infty.
$$  
Then for every 
$f\in \fL^1(\fX)$, the following holds
\begin{itemize}
\item $\Norm{\fP f}_1 \leq \Norm{  f}_1$. 
\end{itemize}
\end{proposition}
\noindent
In mathematics, the positive part of a real   function is defined by the formula
$$
f^{+}(\vx)=\max (f(\vx), 0)= \begin{cases}f(\vx) & \text { if } f(\vx)>0, \\ 0 & \text { otherwise. }\end{cases}
$$
Similarly, the negative part of $f$ is defined as
$$
f^{-}(\vx)=\max (-f(\vx), 0)=-\min (f(\vx), 0)= \begin{cases}-f(\vx) & \text { if } f(\vx)<0, \\ 0 & \text { otherwise. }\end{cases}
$$
We proceed to the proof for Proposition \ref{prop...Contractive}
\begin{proof}
From the definition of $f^{+}$and $f^{-}$, it follows that
\begin{align*}
(\fP f)^{+}=\left(\fP f^{+}-\fP f^{-}\right)^{+} & =\max \left(0, \fP f^{+}-\fP f^{-}\right) \\
& \leq \max \left(0, \fP f^{+}\right)=\fP f^{+}.
\end{align*}   
Similarly, we obtain that 
\begin{align*}
(\fP f)^{-}=\left(\fP f^{+}-\fP f^{-}\right)^{-} & =\max \left(0, \fP f^{-}-\fP f^{+}\right) \\
& \leq \max \left(0, \fP f^{-}\right)=\fP f^{-}.
\end{align*}   
Hence for the last inequality
\begin{align*}
|\fP f| & =(\fP f)^{+}+(\fP f)^{-} \\
&\leq \fP f^{+}+\fP f^{-} \\
& =\fP\left(f^{+}+f^{-}\right)=\fP|f|. 
\end{align*}   

Finally, by integrating the above relation over $\fX$, we obtain that 
\begin{equation}\label{eq...prop...proof...contract}
\begin{aligned}
\Norm{\fP f}_1 & =\int_{\fX}\Abs{\fP f} \D \mu\\
&\leq \int_{\fX}\fP \Abs{ f} \D \mu=\int_{\fX}\Abs{ f} \D \mu=\Norm{f}_1. 
\end{aligned}   
\end{equation}
 
\end{proof}
\noindent
 Inequality \eqref{eq...prop...proof...contract} is extremely important, and any operator $\fP$ that satisfies it is called a contraction. This relation  is known as the contractive property of $\fP$. To illustrate its power, note that for any $f \in \fL^1(\fX)$, we have
$$
 \Norm{\fP^n f}_1=\Norm{\fP\circ \fP^{n-1} f}_1 \leq   \Norm{\fP^{n-1} f}_1.
$$ 
As we  consider  Markov processes with continuous time, it is natural to consider a family of Markov operators indexed by time. We call such a family a Markov semigroup~\citep{hairer2008ergodic}, provided that it satisfies the relation 
\begin{equation}\label{eq...text...Semigroup...SemigroupComposition}
\mathcal{P}_{t+s}=\mathcal{P}_t \circ \mathcal{P}_s,\quad \text{for~any~time}~s,t >0.
\end{equation}
And if given $A\in\fB(\fX)$, where  $\fB(\fX)$ is the Borel $\sigma$-algebra on $\fX$,   and given any two times $s<t$, if the following holds  almost surely
$$
\Prob\left(\vX_t \in A \mid \vX_s\right)=\left(\mathcal{P}_{t-s}\mathbf{1}_{A}\right)\left(\vX_s\right),
$$
then we call $\vX_t$ a time-homogeneous Markov process with semigroup $\left\{\mathcal{P}_t\right\}_{t\geq 0}$.

In our case for dropout, we set the Polish space $\fX=\sR^D$,  and since $\fC_b^M(\sR^D)\subset \fB_b(\sR^D)$, then WLOG  we fix $g\in\fC_b^M(\sR^D)$ and define
\begin{equation}\label{eq...text...Semigroup...DisCtsSemigroup...Discrete-Semi-Group}
\begin{aligned}
\fP_\eps g(\tilde{\vtheta}) :=\Exp \left[ g\left(\tilde{\vtheta}-\eps\nabla_{\vtheta}R_{\fS}^\mathrm{drop}\left(\vtheta; \veta\right)\mid_{\vtheta=\tilde{\vtheta}}\right)  \right].
\end{aligned}
\end{equation}
We conclude 
that the dropout iterations \eqref{eq...text...SME...ModifiedLoss...ThetaUpdate...Abstract...appendix} forms a  time-homogeneous Markov chain with discrete Markov semigroup $\left\{\fP_\eps^n\right\}_{n\geq 0}$.

As for the SDE \eqref{eq...text...SME...SME...SDE...Abstract...appendix}, based on Assumption \ref{Append...theOnlyAssumption} and combined with the results in~\cite[Example 2.11]{hairer2008ergodic},   the Markov semigroup $\left\{\fP_t\right\}_{t\geq 0}$ associated to the solutions of the SDE reads: For any $g\in\fB_b(\sR^D)$,
\[ \partial_t\fP_t g=\fL \fP_t g,\]
where $\fL$ is termed the {\emph{generator}} of the diffusion process \eqref{eq...text...SME...SME...SDE...Abstract...appendix}, which reads
\begin{equation}\label{eq...text...Semigroup...Generator_of_Continuous-Semi-Group}
\fL g:=\left<\vb, \nabla_{\vTheta}g \right>+\frac{1}{2}   \vsigma   \vsigma^\T :\nabla_{\vTheta}^2 g. 
\end{equation}
Moreover, for a fixed test function $g\in \fC_b^M(\sR^D)$,  then for   any two times $s, t \geq 0$,
\begin{equation}\label{eq...text...Semigroup...Continuous-Semi-Group}
\fP_t g(\vTheta_s):=\exp(t\fL)g(\vTheta_s):=\Exp_{\vTheta_s} \left[g(\vTheta_{t+s})\right],
\end{equation}
and $\{\fP_t\}_{t\geq 0}$ forms a continuous Markov semigroup for the  SDE \eqref{eq...text...SME...SME...SDE...Abstract...appendix}.
\subsection{Semigroup Expansion with Accuracy of Order One} \label{subsection...Semigroup...SemigroupExpansionOrderOne}
Our results are essentially based on  It\^{o}-Taylor expansions~\citep{kloeden2011numerical} or  Taylor's theorem with the Lagrange form of the remainder~\citep[Lemma 27]{li2019stochastic}.
\begin{theorem}[Order-$1$ accuracy]\label{thm...OrderOne}
Fix time $T\leq T^{\ast}$,   if we choose  
\begin{align*}
\vb\left(\vTheta\right)&:=   -\nabla_{\vTheta}L_{\fS}(\vTheta),\\     
\vsigma\left(\vTheta\right)&:=\sqrt{\eps}\left(\mSigma\left(\vTheta\right)\right)^{\frac{1}{2}},
\end{align*}
then for all $t\in[0, T]$, the stochastic processes $\vTheta_t $ satisfying
\begin{equation}\label{eq...thm...SDEfororderOneapproximation}
\D \vTheta_t=\vb \left(\vTheta_t\right)\D t+\vsigma\left(\vTheta_t\right) \D \vW_t, \quad \vTheta_0=\vTheta(0), 
\end{equation}
is an order-$1$ approximation of dropout \eqref{eq...text...SME...ModifiedLoss...ThetaUpdate...Abstract...appendix}, i.e.,  given any test function   $g\in\fC_b^4(\sR^D)$, there exists $\eps_0>0$ and $C(T, \Norm{g}_{C^4}, \eps_0)>0$, such that  for any  $\eps\leq\eps_0$ and $T\leq T^{\ast}$, and for all  $N\in[N_{T,\eps}]$, the following holds:
\begin{equation}\label{eq...thm...orderOneapproximation}
\Abs{\Exp g(\vtheta_{N})-\Exp g(\vTheta_{\eps N})} 
\leq C(T, \Norm{g}_{C^4}, \vtheta_0,  \eps_0)\eta,  
\end{equation}
where $\vtheta_0=\vTheta_0$.
\end{theorem}
\begin{proof}
By application of      Taylor's theorem with the Lagrange form of the remainder,   we have that for some $\alpha\geq 1,$
\begin{align*}
g(\bm{\vartheta})-g(\tilde{\bm{\vartheta}})&=\sum_{s=1}^\alpha \frac{1}{s !} \sum_{i_1, \ldots, i_j=1}^D\prod_{j=1}^s\left[\bm{\vartheta}_{\left(i_j\right)}-\tilde{\bm{\vartheta}}_{\left(i_j\right)} \right] \frac{\partial^s g}{\partial \bm{\vartheta}_{\left(i_1\right)} \ldots \partial \bm{\vartheta}_{\left(i_j\right)}}(\tilde{\bm{\vartheta}}) \\    
&~~+\frac{1}{(\alpha+1) !} \sum_{i_1, \ldots, i_j=1}^D \prod_{j=1}^{\alpha+1}\left[\bm{\vartheta}_{\left(i_j\right)}-\tilde{\bm{\vartheta}}_{\left(i_j\right)} \right]  \frac{\partial^{\alpha+1}  g}{\partial \bm{\vartheta}_{\left(i_1\right)} \ldots \partial \bm{\vartheta}_{\left(i_j\right)}}(\gamma\bm{\vartheta}+(1-\gamma)\tilde{\bm{\vartheta}}),
\end{align*}
for some $\gamma\in(0,1)$. We adopt the Einstein's summation convention, where repeated (spatial) indices are summed, i.e.,  \[\vx_{(i)} \vx_{(i)}:=\sum_{i=1}^D \vx_{(i)} \vx_{(i)}.\] 

As we choose $\bm{\vartheta}:=\vtheta_1$, $\tilde{\bm{\vartheta}}:=\vtheta_0$ and $\alpha=1$, then we obtain that 
\begin{align*}
g(\vtheta_1)-g(\vtheta_0)&=\left<\nabla_{\vtheta} g(\vtheta_0), \vtheta_1-\vtheta_0\right>\\
&~~+\frac{1}{2}\nabla_{\vtheta}^2 g(\gamma\vtheta_1+(1-\gamma)\vtheta_0):(\vtheta_1-\vtheta_0)\otimes (\vtheta_1-\vtheta_0)\\
&=\left<\nabla_{\vtheta} g(\vtheta_0), \vtheta_1-\vtheta_0\right> +\frac{1}{2}\nabla_{\vtheta}^2 g( \tilde{\vtheta}_0):(\vtheta_1-\vtheta_0)\otimes (\vtheta_1-\vtheta_0),
\end{align*}
where $\tilde{\vtheta}_0:=\gamma\vtheta_1+(1-\gamma)\vtheta_0$, and 
we observe that  since
\begin{equation*}
\vtheta_1-\vtheta_{0} = -\eps\nabla_{\vtheta}L_{\fS}(\vtheta)\big|_{\vtheta=\vtheta_0}+\sqrt{\eps}\vV(\vtheta_0),    
\end{equation*}
then 
\begin{align*}
\Exp g(\vtheta_1)-\Exp g(\vtheta_0)&=\left<\nabla_{\vtheta} g(\vtheta_0), \Exp\vtheta_1-\Exp\vtheta_0\right> +\frac{1}{2}\Exp\left[\nabla_{\vtheta}^2 g(\tilde{\vtheta}_0):(\vtheta_1-\vtheta_0)\otimes (\vtheta_1-\vtheta_0)\right]\\
&=-\eps\left<\nabla_{\vtheta} g(\vtheta_0),  \nabla_{\vtheta}L_{\fS}(\vtheta)\big|_{\vtheta=\vtheta_0}\right>+E_{\eps}^1({\vtheta}_0),
\end{align*}
where the remainder term $E_{\eps}^1(\cdot):\sR^D\to\sR$, whose expression reads
\begin{equation}
E_{\eps}^1({\vtheta}_0):=   \frac{1}{2}\Exp\left[\nabla_{\vtheta}^2 g(\tilde{\vtheta}_0):(\vtheta_1-\vtheta_0)\otimes (\vtheta_1-\vtheta_0) \right],
\end{equation}
and we remark that $\tilde{\vtheta}_0$ and $\vtheta_1$ are implicitly defined by $\vtheta_0$.
Then, directly from Assumption \ref{Append...theOnlyAssumption}, we obtain that 
\begin{align*}
E_{\eps}^1({\vtheta}_0)&=\frac{1}{2}\Exp\left[\nabla_{\vtheta}^2 g(\tilde{\vtheta}_0):(\vtheta_1-\vtheta_0)\otimes (\vtheta_1-\vtheta_0) \right] \\
&\leq \frac{1}{2} \Norm{g}_{C^4} \Exp \Norm{\vtheta_1-\vtheta_0}_2^2 = \eps^2 \Norm{g}_{C^4} \Exp \left[\Norm{\nabla_{\vtheta}R_{\fS}^\mathrm{drop}\left(\vtheta_{0}; \veta_{1}\right)}_2^2\right]  \\
& \leq \eps^2\Norm{g}_{C^4} C(T^{\ast},\vtheta_0,\eps_0),
\end{align*}
since $\nabla_{\vtheta}L_{\fS}(\vtheta)$ and $\mSigma\left(\vtheta\right)$ can be  bounded above by the second and  fourth  moments of the dropout  iteration \eqref{eq...text...SME...ModifiedLoss...ThetaUpdate...Abstract...appendix}.

We observe that  
\begin{align*}
  \vTheta_\eps -   \vTheta_{0} &=  \int_0^\eps \vb (\vTheta_s)\D s+\int_0^\eps\vsigma(\vTheta_s)\D \mW_s. 
\end{align*}
As we choose $\bm{\vartheta}:=\vTheta_{\eps}$, $\tilde{\bm{\vartheta}}:=\vTheta_{0}$ and $\alpha=1$, then we obtain that 
\begin{align*}  
g(\vTheta_{\eps})-g(\vTheta_{0})&=\left<\nabla_{\vTheta} g(\vTheta_{0}), \vTheta_\eps-\vTheta_{0}\right>\\
&~~+\frac{1}{2}\nabla_{\vTheta}^2 g(\widetilde{\vTheta}_{0}):(\vTheta_\eps-\vTheta_{0})\otimes(\vTheta_\eps-\vTheta_{0}),  
\end{align*}
where 
\[\widetilde{\vTheta}_{0}:=\gamma\vTheta_\eps+(1-\gamma)\vTheta_{0},\]
for some $\gamma\in(0,1)$.
Then 
\begin{align*}
&\Exp g(\vTheta_{\eps})-\Exp g(\vTheta_0)\\=&\left<\nabla_{\vTheta} g(\vTheta_0), \Exp\vTheta_\eps -  \Exp \vTheta_{0}\right> +\frac{1}{2}\Exp\left[\nabla_{\vTheta}^2 g(\widetilde{\vTheta}_{0}):(\vTheta_\eps-\vTheta_{0})\otimes(\vTheta_\eps-\vTheta_{0})\right]\\
=&\left<\nabla_{\vTheta} g(\vTheta_0),  \int_0^\eps \Exp [\vb (\vTheta_s)]\D s\right> +\frac{1}{2}\Exp\left[\nabla_{\vTheta}^2 g(\widetilde{\vTheta}_{0}):(\vTheta_\eps-\vTheta_{0})\otimes(\vTheta_\eps-\vTheta_{0})\right],
\end{align*}
and  since 
\begin{align*}
\left<\nabla_{\vTheta} g(\vTheta_0), \Exp [\vb (\vTheta_s)]\right>&=\left<\nabla_{\vTheta} g(\vTheta_0),  \Exp [\vb (\vTheta_0)] \right> +\int_{0}^s \fL \left<\nabla_{\vTheta} g(\vTheta_0), \vb  \right>(\vTheta_v)\D v,
\end{align*}
then we obtain that 
\begin{align*}
\Exp g(\vTheta_{\eps})-\Exp g(\vTheta_0)&=\eps\left<\nabla_{\vTheta} g(\vTheta_0),  \Exp [\vb (\vTheta_0)] \right>+\int_0^\eps\int_{0}^s \fL \left<\nabla_{\vTheta} g(\vTheta_0), \vb  \right>(\vTheta_v)\D v \D s\\
&~~+\frac{1}{2}\Exp\left[\nabla_{\vTheta}^2 g(\widetilde{\vTheta}_{0}):(\vTheta_\eps-\vTheta_{0})\otimes(\vTheta_\eps-\vTheta_{0})\right]\\
&=\eps\left<\nabla_{\vTheta} g(\vTheta_0),  \vb (\vTheta_0)  \right>+\eps^2\Bar{E}_{\eps}^1({\vTheta}_0),
\end{align*}
where the remainder term $\Bar{E}_{\eps}^1(\cdot):\sR^D\to\sR$, whose expression reads
\begin{equation}\label{eq...text...Semigroup...OrderOne...BarR_1}
\begin{aligned}
\Bar{E}_{\eps}^1({\vTheta}_0)&:=  \int_0^\eps\int_{0}^s \fL \left<\nabla_{\vTheta} g(\vTheta_0), \vb  \right>(\vTheta_v)\D v \D s\\
&~~+\frac{1}{2}\Exp\left[\nabla_{\vTheta}^2 g(\widetilde{\vTheta}_{0}):(\vTheta_\eps-\vTheta_{0})\otimes(\vTheta_\eps-\vTheta_{0})\right],
\end{aligned}
\end{equation}
and we remark that $\widetilde{\vTheta}_0$ and $\vTheta_\eps$ are implicitly defined by $\vTheta_0$.
As we choose 
\begin{align*}
\vb\left(\vTheta\right)& =   -\nabla_{\vTheta}L_{\fS}(\vTheta),\\     
\vsigma\left(\vTheta\right)& =\sqrt{\eps}\left(\mSigma\left(\vTheta\right)\right)^{\frac{1}{2}},
\end{align*}
then we carry out the computation for $\fL \left<\nabla_{\vTheta} g(\vTheta_0), \vb  \right>(\vTheta_v)$,
\begin{align*}
\fL \left<\nabla_{\vTheta} g(\vTheta_0), \vb \right>(\vTheta_v)&=\left< \nabla_{\vTheta}L_{\fS}(\vTheta_v) , \nabla_{\vTheta} \left<\nabla_{\vTheta} g(\vTheta_0),  \nabla_{\vTheta}L_{\fS}(\vTheta)   \right>\mid_{\vTheta=\vTheta_v}\right>\\
&~~+\frac{\eps}{2}    \mSigma\left(\vTheta_v\right):\nabla^2_{\vTheta}\left(\left<\nabla_{\vTheta} g(\vTheta_0),  \nabla_{\vTheta}L_{\fS}(\vTheta)   \right>\right)\mid_{\vTheta=\vTheta_v},
\end{align*}
since $\nabla_{\vTheta}L_{\fS}(\vTheta)$,  $\nabla^2_{\vTheta}L_{\fS}(\vTheta)$, $\nabla^3_{\vTheta}L_{\fS}(\vTheta)$ and $\mSigma\left(\vTheta\right)$ can be  bounded above by the second, fourth and sixth moments of the solution to SDE \eqref{eq...text...SME...SME...SDE...Abstract...appendix}, hence we may apply the mean value theorem to \eqref{eq...text...Semigroup...OrderOne...BarR_1} and obtain that  
\begin{align*}
\Abs{\Bar{E}_{\eps}^1({\vTheta}_0)}& = \Abs{ \int_0^\eps s \fL \left<\nabla_{\vTheta} g(\vTheta_0), \vb  \right>({\widetilde{\vTheta}}_s)\D s    +\frac{1}{2}\Exp\left[\nabla_{\vTheta}^2 g(\widetilde{\vTheta}_0):(\vTheta_\eps-\vTheta_{0})\otimes(\vTheta_\eps-\vTheta_{0})\right] }   \\
&\leq \int_0^\eps s \Norm{g}_{C^4} C(T^{\ast},\vTheta_0)\D s+\frac{1}{2} \Norm{g}_{C^4} \Exp \Norm{\vTheta_\eps-\vTheta_0}_2^2 \\
&\leq \frac{\eps^2}{2}  \Norm{g}_{C^4} C(T^{\ast},\vTheta_0)+\Norm{g}_{C^4} \Exp\Norm{\int_0^\eps \vb (\vTheta_s)\D s+\int_0^\eps\vsigma(\vTheta_s)\D \mW_s}_2^2  \\
&\leq \frac{\eps^2}{2}  \Norm{g}_{C^4} C(T^{\ast},\vTheta_0)+2\Norm{g}_{C^4} \Exp \Norm{\int_0^\eps \vb (\vTheta_s)\D s}_2^2\\&~~+2\Norm{g}_{C^4} \Exp\Norm{\int_0^\eps\vsigma(\vTheta_s)\D \mW_s}_2^2\\
&\leq \frac{\eps^2}{2}  \Norm{g}_{C^4} C(T^{\ast},\vTheta_0)+2\Norm{g}_{C^4} \eps^2 \Exp\Norm{ \nabla_{\vTheta}L_{\fS} (\widetilde{\vTheta}_0) }_2^2\\&~~+2\Norm{g}_{C^4} \Exp \int_0^\eps\Norm{\vsigma(\vTheta_s)}_{\mathrm{F}}^2\D s \\
&\leq \frac{\eps^2}{2}  \Norm{g}_{C^4} C(T^{\ast},\vTheta_0)+2\Norm{g}_{C^4} \eps^2 \Exp\Norm{ \nabla_{\vTheta}L_{\fS} (\widetilde{\vTheta}_0) }_2^2\\&~~+ 2\Norm{g}_{C^4} \eps \Exp \left[\eps \Norm{\vSigma(\widetilde{\vTheta}_0)}_{\mathrm{F}} \right] \leq \eps^2\Norm{g}_{C^4}C(T^{\ast},\vTheta_0).
\end{align*}
To sum up for now,
\begin{align*}
\Abs{\Exp g(\vtheta_1)-\Exp g(\vTheta_\eps)}&=\Big|\Exp g(\vtheta_0)-\eps\left<\nabla_{\vtheta} g(\vtheta_0),  \nabla_{\vtheta}L_{\fS}(\vtheta)\big|_{\vtheta=\vtheta_0}\right>+E_{\eps}^1({\vtheta}_0)\\   
&~~-\Exp g(\vTheta_0)-\eps\left<\nabla_{\vTheta} g(\vTheta_0),  \vb (\vTheta_0)  \right>+ \Bar{E}_{\eps}^1({\vTheta}_0)\Big|,
\end{align*}
since $\vtheta_0=\vTheta_0$ and $\vb\left(\vTheta_0\right)  =   -\nabla_{\vTheta}L_{\fS}(\vTheta)\big|_{\vtheta=\vtheta_0}$, thus
\begin{equation}\label{eq...proof...OrderOne...Crutial}
\begin{aligned}
\Abs{ \fP_\eps^1 g(\vtheta_0)-  \fP_\eps g(\vTheta_0) }&=\Abs{\Exp g(\vtheta_1)-\Exp g(\vTheta_\eps)}\\
&\leq \Abs{E_{\eps}^1({\vtheta}_0)}+\Abs{\Bar{E}_{\eps}^1({\vTheta}_0)}\\
&\leq  \eps^2\Norm{g}_{C^4} C(T^{\ast},\vtheta_0,\eps_0)+\eps^2\Norm{g}_{C^4}C(T^{\ast},\vTheta_0)\\
&=\fO(\eps^2).
\end{aligned}
\end{equation}

For the $N$-th step iteration,  since
\begin{align*}
\Abs{\Exp g(\vtheta_{N})-\Exp g(\vTheta_{\eps N})} &=\Abs{\fP_\eps^N g(\vtheta_0)-\fP_{ \eps N } g(\vTheta_0)},  
\end{align*}
and the RHS of the above equation can be written into a telescoping sum as 
\begin{align*}
\fP_\eps^N g(\vtheta_0)-\fP_{ \eps N } g(\vTheta_0)&=\sum_{l=1}^{N}\left(\fP_\eps^{N-l+1}\circ\fP_{(l-1)\eps}  g(\vtheta_0)-\fP_\eps^{N-l}\circ\fP_{l\eps} g(\vTheta_0)  \right),
\end{align*}
hence by application of Proposition \ref{prop...Contractive}, we obtain that 
\begin{align*}
\Abs{\Exp g(\vtheta_{N})-\Exp g(\vTheta_{\eps N})}  &\leq  \sum_{l=1}^{N}\Abs{\fP_\eps^{N-l+1}\circ\fP_{(l-1)\eps}  g(\vtheta_0)-\fP_\eps^{N-l}\circ\fP_{l\eps} g(\vTheta_0)  }\\
&\leq \sum_{l=1}^{N}\Abs{\fP_\eps^{N-l}\circ\left(\fP_\eps^{1}\circ\fP_{(l-1)\eps}  - \fP_\eps\circ\fP_{(l-1)\eps}   \right) g(\vTheta_0)},
\end{align*}
since $\left(\fP_\eps^{1}\circ\fP_{(l-1)\eps}  - \fP_\eps\circ\fP_{(l-1)\eps}   \right) g(\vTheta_0)$ can be regarded as $\fL^1(\sR^D)$ if we choose measure $\mu$  to be the delta measure concentrated on $\vTheta_0$. i.e., \[\mu:=\delta_{\vTheta_0},\] hence by the conctration property of Markov operators, we obtain further that 
\begin{align*}
\Abs{\Exp g(\vtheta_{N})-\Exp g(\vTheta_{\eps N})} 
&\leq\sum_{l=1}^{N}\Abs{ \left(\fP_\eps^{1}\circ\fP_{(l-1)\eps}  - \fP_\eps\circ\fP_{(l-1)\eps}   \right) g(\vTheta_0)}\\
&\leq   \sum_{l=1}^{N}\Abs{\fP_\eps^{1}   g(\vTheta_{(l-1)\eps} )-\fP_\eps g(\vTheta_{(l-1)\eps} )}.
\end{align*}
By taking   expectation conditioned on $\vTheta_{(l-1)\eps}$, then similar to the relation \eqref{eq...proof...OrderOne...Crutial}, the following holds 
\begin{equation*}
\begin{aligned}
\Abs{ \fP_\eps^1 g(\vTheta_{(l-1)\eps})-  \fP_\eps g(\vTheta_{(l-1)\eps}) }&=\Exp\left[\left[\Abs{\Exp g(\vtheta_l)-\Exp g(\vTheta_\eps l)}\Big| \vTheta_{(l-1)\eps}\right]\right]\\
&\leq \Exp\Abs{E_{\eps}^1(\vTheta_{(l-1)\eps})}+\Exp\Abs{\Bar{E}_{\eps}^1(\vTheta_{(l-1)\eps})}\\
&\leq  \eps^2\Norm{g}_{C^4} C(T^{\ast},\vtheta_0,\eps_0)+\eps^2\Norm{g}_{C^4}C(T^{\ast},\vTheta_0)\\
&=\fO(\eps^2).
\end{aligned}
\end{equation*}
We remark that the last line of the above relation is essentially based on Assumption \ref{Append...theOnlyAssumption}, since  $ \Exp\Abs{E_{\eps}^1(\vTheta_{(l-1)\eps})}$ and $\Exp\Abs{\Bar{E}_{\eps}^1(\vTheta_{(l-1)\eps})}$ can be  bounded above by the second, fourth and sixth moments of the solution to SDE \eqref{eq...text...SME...SME...SDE...Abstract...appendix}, hence we may apply dominated convergence theorem to obtain the  last line of the above relation.

To sum up, as
\begin{align*}
 \Abs{\fP_\eps^N g(\vtheta_0)-\fP_{ \eps N } g(\vTheta_0)}  &\leq  \sum_{l=1}^{N}\Abs{\fP_\eps^{N-l+1}\circ\fP_{(l-1)\eps}  g(\vtheta_0)-\fP_\eps^{N-l}\circ\fP_{l\eps} g(\vTheta_0)  } =N\fO(\eps^2),
\end{align*}
hence for  $N=N_{T,\eps}$, 
\[
 \Abs{\fP_\eps^N g(\vtheta_0)-\fP_{ \eps N } g(\vTheta_0)}   =N\fO(\eps^2)=N\eps \fO(\eps)\leq T\fO(\eps)=\fO(\eps).
\]
\end{proof}
\subsection{Semigroup Expansion with Accuracy of Order Two} \label{subsection...Semigroup...SemigroupExpansionOrderTwo}
\begin{theorem}[Order-$2$ accuracy]\label{thm...OrderTwo}
Fix time $T\leq T^{\ast}$,   if we choose  
\begin{align*}
\vb(\vTheta)&=-\nabla_{\vTheta}\left(L_{\fS}(\vTheta)+\frac{\eps}{4}\Norm{\nabla_{\vTheta}L_S(\vTheta)}_2^2\right),\\
\vsigma(\vTheta)&=\sqrt{\eps}\left(\mSigma\left(\vTheta\right)\right)^{\frac{1}{2}},
\end{align*}
then for all $t\in[0, T]$, the stochastic processes $\vTheta_t $ satisfying
\begin{equation}\label{eq...thm...SDEfororderTwoapproximation}
\D \vTheta_t=\vb \left(\vTheta_t\right)\D t+\vsigma\left(\vTheta_t\right) \D \vW_t, \quad \vTheta_0=\vTheta(0), 
\end{equation}
is an order-$2$ approximation of dropout \eqref{eq...text...SME...ModifiedLoss...ThetaUpdate...Abstract...appendix}, i.e.,  given any test function   $g\in\fC_b^6(\sR^D)$, there exists $\eps_0>0$ and $C(T, \Norm{g}_{C^6}, \eps_0)>0$, such that  for any  $\eps\leq\eps_0$ and $T\leq T^{\ast}$, and for all  $N\in[N_{T,\eps}]$, the following holds:
\begin{equation}\label{eq...thm...orderTwoapproximation}
\Abs{\Exp g(\vtheta_{N})-\Exp g(\vTheta_{\eps N})} 
\leq C(T, \Norm{g}_{C^6}, \vtheta_0,  \eps_0)\eta,  
\end{equation}
where $\vtheta_0=\vTheta_0$.
\end{theorem}
\begin{proof}
    By application of      Taylor's theorem with the Lagrange form of the remainder,   we have that for some $\alpha\geq 1,$
\begin{align*}
g(\bm{\vartheta})-g(\tilde{\bm{\vartheta}})&=\sum_{s=1}^\alpha \frac{1}{s !} \sum_{i_1, \ldots, i_j=1}^D\prod_{j=1}^s\left[\bm{\vartheta}_{\left(i_j\right)}-\tilde{\bm{\vartheta}}_{\left(i_j\right)} \right] \frac{\partial^s g}{\partial \bm{\vartheta}_{\left(i_1\right)} \ldots \partial \bm{\vartheta}_{\left(i_j\right)}}(\tilde{\bm{\vartheta}}) \\    
&~~+\frac{1}{(\alpha+1) !} \sum_{i_1, \ldots, i_j=1}^D \prod_{j=1}^{\alpha+1}\left[\bm{\vartheta}_{\left(i_j\right)}-\tilde{\bm{\vartheta}}_{\left(i_j\right)} \right]  \frac{\partial^{\alpha+1}  g}{\partial \bm{\vartheta}_{\left(i_1\right)} \ldots \partial \bm{\vartheta}_{\left(i_j\right)}}(\gamma\bm{\vartheta}+(1-\gamma)\tilde{\bm{\vartheta}}),
\end{align*}
for some $\gamma\in(0,1)$.

As we choose $\bm{\vartheta}:=\vtheta_1$, $\tilde{\bm{\vartheta}}:=\vtheta_0$ and $\alpha=2$, with slight misuse of the Frobenius inner product  notation, we obtain that 
\begin{align*}
g(\vtheta_1)-g(\vtheta_0)&=\left<\nabla_{\vtheta} g(\vtheta_0), \vtheta_1-\vtheta_0\right> +\frac{1}{2}\nabla_{\vtheta}^2 g( \vtheta_0):(\vtheta_1-\vtheta_0)\otimes (\vtheta_1-\vtheta_0)\\
&~~+\frac{1}{6}\nabla_{\vtheta}^3 g( \gamma\vtheta_1+(1-\gamma)\vtheta_0):(\vtheta_1-\vtheta_0)\otimes (\vtheta_1-\vtheta_0)\otimes (\vtheta_1-\vtheta_0)\\
&=\left<\nabla_{\vtheta} g(\vtheta_0), \vtheta_1-\vtheta_0\right> +\frac{1}{2}\nabla_{\vtheta}^2 g( \vtheta_0):(\vtheta_1-\vtheta_0)\otimes (\vtheta_1-\vtheta_0)\\
&~~+\frac{1}{6}\nabla_{\vtheta}^3 g( \tilde{\vtheta}_0):(\vtheta_1-\vtheta_0)\otimes (\vtheta_1-\vtheta_0)\otimes (\vtheta_1-\vtheta_0),
\end{align*}
where $\tilde{\vtheta}_0:=\gamma\vtheta_1+(1-\gamma)\vtheta_0$, and 
we observe that  since
\begin{equation*}
\vtheta_1-\vtheta_{0} = -\eps\nabla_{\vtheta}L_{\fS}(\vtheta)\big|_{\vtheta=\vtheta_0}+\sqrt{\eps}\vV(\vtheta_0),    
\end{equation*}
then 
\begin{align*}
\Exp g(\vtheta_1)-\Exp g(\vtheta_0)&=\left<\nabla_{\vtheta} g(\vtheta_0), \Exp\vtheta_1-\Exp\vtheta_0\right> +\frac{1}{2}\nabla_{\vtheta}^2 g( {\vtheta}_0):\Exp\left[(\vtheta_1-\vtheta_0)\otimes (\vtheta_1-\vtheta_0)\right]\\
&~~+\frac{1}{6}\Exp\left[\nabla_{\vtheta}^3 g( \tilde{\vtheta}_0):(\vtheta_1-\vtheta_0)\otimes (\vtheta_1-\vtheta_0)\otimes (\vtheta_1-\vtheta_0)\right]\\
&=-\eps\left<\nabla_{\vtheta} g(\vtheta_0),  \nabla_{\vtheta}L_{\fS}(\vtheta)\big|_{\vtheta=\vtheta_0}\right>\\
&~~+ \frac{\eps^2}{2}\nabla_{\vtheta}^2 g( {\vtheta}_0):\left( \nabla_{\vtheta}L_{\fS}(\vtheta)\big|_{\vtheta=\vtheta_0} \otimes  \nabla_{\vtheta}L_{\fS}(\vtheta)\big|_{\vtheta=\vtheta_0}+\mSigma(\vtheta_{0})\right)\\
&~~+E_{\eps}^2({\vtheta}_0),
\end{align*}
where the remainder term $E_{\eps}^2(\cdot):\sR^D\to\sR$, whose expression reads
\begin{equation}
E_{\eps}^2({\vtheta}_0):=   \frac{1}{6}\Exp\left[\nabla_{\vtheta}^3 g(\tilde{\vtheta}_0):(\vtheta_1-\vtheta_0)\otimes (\vtheta_1-\vtheta_0)\otimes (\vtheta_1-\vtheta_0)\right],
\end{equation}
and we remark that $\tilde{\vtheta}_0$ and $\vtheta_1$ are implicitly defined by $\vtheta_0$.
Then, directly from Assumption \ref{Append...theOnlyAssumption}, we obtain that 
\begin{align*}
E_{\eps}^2({\vtheta}_0)
&\leq \frac{1}{6} \Norm{g}_{C^6} \Exp \Norm{\vtheta_1-\vtheta_0}_2^3 = \eps^3 \Norm{g}_{C^6} \Exp \left[\Norm{\nabla_{\vtheta}R_{\fS}^\mathrm{drop}\left(\vtheta_{0}; \veta_{1}\right)}_2^3\right]  \\
& \leq \eps^3\Norm{g}_{C^6} C(T^{\ast},\vtheta_0,\eps_0),
\end{align*}
since $\nabla_{\vtheta}L_{\fS}(\vtheta)$ and $\mSigma\left(\vtheta\right)$ can be  bounded above by the second and  fourth  moments of the dropout  iteration \eqref{eq...text...SME...ModifiedLoss...ThetaUpdate...Abstract...appendix}.

We observe that   
\begin{align*}
  \vTheta_\eps -   \vTheta_{0} &=  \int_0^\eps \vb (\vTheta_s)\D s+\int_0^\eps\vsigma(\vTheta_s)\D \mW_s. 
\end{align*}
As we choose $\bm{\vartheta}:=\vTheta_{\eps}$, $\tilde{\bm{\vartheta}}:=\vTheta_{0}$ and $\alpha=3$, then we obtain that 
\begin{align*}  
g(\vTheta_{\eps})-g(\vTheta_{0})&=\left<\nabla_{\vTheta} g(\vTheta_{0}), \vTheta_\eps-\vTheta_{0}\right>\\
&~~+\frac{1}{2}\nabla_{\vTheta}^2 g( {\vTheta}_{0}):(\vTheta_\eps-\vTheta_{0})\otimes(\vTheta_\eps-\vTheta_{0})\\
&~~+\frac{1}{6}\nabla_{\vTheta}^3 g(  {\vTheta}_{0}):(\vTheta_\eps-\vTheta_{0})\otimes(\vTheta_\eps-\vTheta_{0})  \otimes(\vTheta_\eps-\vTheta_{0})\\
&~~+\frac{1}{24}\nabla_{\vTheta}^4 g( \widetilde{\vTheta}_{0}):(\vTheta_\eps-\vTheta_{0})\otimes(\vTheta_\eps-\vTheta_{0})  \otimes(\vTheta_\eps-\vTheta_{0})\otimes(\vTheta_\eps-\vTheta_{0}),  
\end{align*}
where 
\[\widetilde{\vTheta}_{0}:=\gamma\vTheta_\eps+(1-\gamma)\vTheta_{0},\]
for some $\gamma\in(0,1)$. Then 
\begin{align*}
&\Exp g(\vTheta_{\eps})-\Exp g(\vTheta_0)\\=&\left<\nabla_{\vTheta} g(\vTheta_0), \Exp\vTheta_\eps -  \Exp \vTheta_{0}\right> +\frac{1}{2}\nabla_{\vTheta}^2 g({\vTheta}_{0}):\Exp\left[(\vTheta_\eps-\vTheta_{0})\otimes(\vTheta_\eps-\vTheta_{0})\right]\\
&~~+\frac{1}{6}\nabla_{\vTheta}^3 g( {\vTheta}_{0}):\Exp\left[(\vTheta_\eps-\vTheta_{0})\otimes(\vTheta_\eps-\vTheta_{0})  \otimes(\vTheta_\eps-\vTheta_{0})\right]\\
&~~+\frac{1}{24}\Exp\left[\nabla_{\vTheta}^4 g( \widetilde{\vTheta}_{0}):(\vTheta_\eps-\vTheta_{0})\otimes(\vTheta_\eps-\vTheta_{0})  \otimes(\vTheta_\eps-\vTheta_{0})\otimes(\vTheta_\eps-\vTheta_{0})\right]\\
=&\left<\nabla_{\vTheta} g(\vTheta_0),  \int_0^\eps \Exp [\vb (\vTheta_s)]\D s\right> +\frac{1}{2}\nabla_{\vTheta}^2 g({\vTheta}_{0}):\Exp\left[(\vTheta_\eps-\vTheta_{0})\otimes(\vTheta_\eps-\vTheta_{0})\right]\\&~~+\frac{1}{6}\nabla_{\vTheta}^3 g( {\vTheta}_{0}):\Exp\left[(\vTheta_\eps-\vTheta_{0})\otimes(\vTheta_\eps-\vTheta_{0})  \otimes(\vTheta_\eps-\vTheta_{0})\right]\\
&~~+\frac{1}{24}\Exp\left[\nabla_{\vTheta}^4 g( \widetilde{\vTheta}_{0}):(\vTheta_\eps-\vTheta_{0})\otimes(\vTheta_\eps-\vTheta_{0})  \otimes(\vTheta_\eps-\vTheta_{0})\otimes(\vTheta_\eps-\vTheta_{0})\right],
\end{align*}
and  since 
\begin{align*}
\left<\nabla_{\vTheta} g(\vTheta_0), \Exp [\vb (\vTheta_s)]\right>&=\left<\nabla_{\vTheta} g(\vTheta_0),  \Exp [\vb (\vTheta_0)] \right> +\int_{0}^s \fL \left<\nabla_{\vTheta} g(\vTheta_0), \vb  \right>(\vTheta_v)\D v,
\end{align*}
then we obtain that 
\begin{align*}
\Exp g(\vTheta_{\eps})-\Exp g(\vTheta_0)&=\eps\left<\nabla_{\vTheta} g(\vTheta_0),  \Exp [\vb (\vTheta_0)] \right>+\int_0^\eps\int_{0}^s \fL \left<\nabla_{\vTheta} g(\vTheta_0), \vb  \right>(\vTheta_v)\D v \D s\\
&~~+\frac{1}{2}\nabla_{\vTheta}^2 g({\vTheta}_{0}):\Exp\left[(\vTheta_\eps-\vTheta_{0})\otimes(\vTheta_\eps-\vTheta_{0})\right]\\
& ~~+\frac{1}{6}\nabla_{\vTheta}^3 g( {\vTheta}_{0}):\Exp\left[(\vTheta_\eps-\vTheta_{0})\otimes(\vTheta_\eps-\vTheta_{0})  \otimes(\vTheta_\eps-\vTheta_{0})\right]\\
&~~+\frac{1}{24}\Exp\left[\nabla_{\vTheta}^4 g( \widetilde{\vTheta}_{0}):(\vTheta_\eps-\vTheta_{0})\otimes(\vTheta_\eps-\vTheta_{0})  \otimes(\vTheta_\eps-\vTheta_{0})\otimes(\vTheta_\eps-\vTheta_{0})\right],
\end{align*}
and once again since 
\begin{align*}
 \fL \left<\nabla_{\vTheta} g(\vTheta_0), \vb  \right>(\vTheta_v)&= \fL \left<\nabla_{\vTheta} g(\vTheta_0), \vb  \right>(\vTheta_0)+\int_0^v\fL\left(\fL \left<\nabla_{\vTheta} g(\vTheta_0), \vb  \right>\right)(\vTheta_u)\D u,
\end{align*}
then we obtain that
\begin{align*}
\Exp g(\vTheta_{\eps})-\Exp g(\vTheta_0)&=\eps\left<\nabla_{\vTheta} g(\vTheta_0),  \Exp [\vb (\vTheta_0)] \right>+\int_0^\eps\int_{0}^s \fL \left<\nabla_{\vTheta} g(\vTheta_0), \vb  \right>(\vTheta_v)\D v \D s\\
&~~+\frac{1}{2}\nabla_{\vTheta}^2 g({\vTheta}_{0}):\Exp\left[(\vTheta_\eps-\vTheta_{0})\otimes(\vTheta_\eps-\vTheta_{0})\right]\\
& ~~+\frac{1}{6}\nabla_{\vTheta}^3 g( {\vTheta}_{0}):\Exp\left[(\vTheta_\eps-\vTheta_{0})\otimes(\vTheta_\eps-\vTheta_{0})  \otimes(\vTheta_\eps-\vTheta_{0})\right]\\
&~~+\frac{1}{24}\Exp\left[\nabla_{\vTheta}^4 g( \widetilde{\vTheta}_{0}):(\vTheta_\eps-\vTheta_{0})\otimes(\vTheta_\eps-\vTheta_{0})  \otimes(\vTheta_\eps-\vTheta_{0})\otimes(\vTheta_\eps-\vTheta_{0})\right]\\
&=\eps\left<\nabla_{\vTheta} g(\vTheta_0),  \Exp [\vb (\vTheta_0)] \right>+\int_0^\eps\int_{0}^s  \fL \left<\nabla_{\vTheta} g(\vTheta_0), \vb  \right>(\vTheta_0) \D v \D s\\ 
&~~+\int_0^\eps\int_{0}^s\int_0^v \fL\left(\fL \left<\nabla_{\vTheta} g(\vTheta_0), \vb  \right>\right)(\vTheta_u)\D u\D v \D s\\
&~~+\frac{1}{2}\nabla_{\vTheta}^2 g({\vTheta}_{0}):\Exp\left[(\vTheta_\eps-\vTheta_{0})\otimes(\vTheta_\eps-\vTheta_{0})\right]\\
& ~~+\frac{1}{6}\nabla_{\vTheta}^3 g( {\vTheta}_{0}):\Exp\left[(\vTheta_\eps-\vTheta_{0})\otimes(\vTheta_\eps-\vTheta_{0})  \otimes(\vTheta_\eps-\vTheta_{0})\right]\\
&~~+\frac{1}{24}\Exp\left[\nabla_{\vTheta}^4 g( \widetilde{\vTheta}_{0}):(\vTheta_\eps-\vTheta_{0})\otimes(\vTheta_\eps-\vTheta_{0})  \otimes(\vTheta_\eps-\vTheta_{0})\otimes(\vTheta_\eps-\vTheta_{0})\right] \\ 
&=\eps\left<\nabla_{\vTheta} g(\vTheta_0),  \Exp [\vb (\vTheta_0)] \right>+ \frac{\eps^2}{2}   \fL \left<\nabla_{\vTheta} g(\vTheta_0), \vb  \right>(\vTheta_0)  \\ 
&~~+\frac{1}{2}\nabla_{\vTheta}^2 g({\vTheta}_{0}):\Exp\left[(\vTheta_\eps-\vTheta_{0})\otimes(\vTheta_\eps-\vTheta_{0})\right]+\Bar{E}_{\eps}^2(\vTheta_0),
\end{align*}
where the remainder term $\Bar{E}_{\eps}^2(\cdot):\sR^D\to\sR$, whose expression reads
\begin{equation}\label{eq...text...Semigroup...OrderTwo...BarR_2}
\begin{aligned}
\Bar{E}_{\eps}^2({\vTheta}_0)  &:=\int_0^\eps\int_{0}^s\int_0^v \fL\left(\fL \left<\nabla_{\vTheta} g(\vTheta_0), \vb  \right>\right)(\vTheta_u)\D u\D v \D s\\
&~~+\frac{1}{6}\nabla_{\vTheta}^3 g( {\vTheta}_{0}):\Exp\left[(\vTheta_\eps-\vTheta_{0})\otimes(\vTheta_\eps-\vTheta_{0})  \otimes(\vTheta_\eps-\vTheta_{0})\right]\\
&~~+\frac{1}{24}\Exp\left[\nabla_{\vTheta}^4 g( \widetilde{\vTheta}_{0}):(\vTheta_\eps-\vTheta_{0})\otimes(\vTheta_\eps-\vTheta_{0})  \otimes(\vTheta_\eps-\vTheta_{0})\otimes(\vTheta_\eps-\vTheta_{0})\right],
\end{aligned}
\end{equation}
and we remark that $\widetilde{\vTheta}_0$ and $\vTheta_\eps$ are implicitly defined by $\vTheta_0$.  As we choose 
\begin{align*}
\vb\left(\vTheta\right)& =   - \nabla_{\vTheta}\left(L_{\fS}(\vTheta)+\frac{\eps}{4}\Norm{\nabla_{\vTheta}L_S(\vTheta)}_2^2\right),\\     
\vsigma\left(\vTheta\right)& =\sqrt{\eps}\left(\mSigma\left(\vTheta\right)\right)^{\frac{1}{2}},
\end{align*}
 then we carry out the computation for   $\fL\left(\fL \left<\nabla_{\vTheta} g(\vTheta_0), \vb  \right>\right)(\vTheta_u)$,
\begin{align*}
&\fL\left(\fL \left<\nabla_{\vTheta} g(\vTheta_0), \vb  \right>\right)(\vTheta_u)\\
=&\fL\left(\left<\vb, \nabla_{\vTheta}\left(\left<\nabla_{\vTheta} g(\vTheta_0), \vb  \right>\right)  \right>\right)(\vTheta_u)+\fL\left(\frac{\eps}{2} \mSigma:\nabla_{\vTheta}^2\left( \left<\nabla_{\vTheta} g(\vTheta_0), \vb  \right>  \right)\right)(\vTheta_u)\\
=&\left<\vb, \nabla_{\vTheta}\left(\left<\vb, \nabla_{\vTheta}\left(\left<\nabla_{\vTheta} g(\vTheta_0), \vb  \right>\right)  \right>\right)\right>+\frac{\eps}{2} \mSigma:\nabla_{\vTheta}\left(\left<\vb, \nabla_{\vTheta}^2\left(\left<\nabla_{\vTheta} g(\vTheta_0), \vb  \right>\right)  \right>\right)\\
&~+\frac{\eps}{2}\left<\vb, \nabla_{\vTheta}\left( \mSigma:\nabla_{\vTheta}^2\left( \left<\nabla_{\vTheta} g(\vTheta_0), \vb  \right>  \right)\right)\right>+\frac{\eps^2}{4} \mSigma:\nabla^2_{\vTheta}\left( \mSigma:\nabla_{\vTheta}^2\left( \left<\nabla_{\vTheta} g(\vTheta_0), \vb  \right>  \right)\right)\\
=& \vb^\T \nabla_{\vTheta}\left( \vb^\T\nabla_{\vTheta}\vb\nabla_{\vTheta} g(\vTheta_0) \right)(\vTheta_u)+\eps R_\eps(\vTheta_u)\\
=&\left<\nabla_{\vTheta}L_{\fS}(\vTheta_u),\nabla_{\vTheta}\left( \left<\frac{1}{2}\nabla_{\vTheta} \left(\Norm{\nabla_{\vTheta}L_{\fS}(\vTheta_u)}_2^2\right),  \nabla_{\vTheta}   g(\vTheta_0)\right> \right)\right>+\eps R'_\eps(\vTheta_u),
\end{align*}
since $\nabla_{\vTheta}L_{\fS}(\vTheta)$,  $\nabla^2_{\vTheta}L_{\fS}(\vTheta)$, $\nabla^3_{\vTheta}L_{\fS}(\vTheta)$, $\mSigma\left(\vTheta\right)$, $  R_\eps(\vTheta_u)$ and   $  R'_\eps(\vTheta_u)$ can be  bounded above by the second, fourth and sixth moments of the solution to SDE \eqref{eq...text...SME...SME...SDE...Abstract...appendix}. Moreover, we observe that 
\begin{align*} 
&\Exp\left[(\vTheta_\eps-\vTheta_{0})\otimes(\vTheta_\eps-\vTheta_{0})  \otimes(\vTheta_\eps-\vTheta_{0})\right]\\
=&\Exp\Bigg[\left(\int_0^\eps \vb (\vTheta_s)\D s+\int_0^\eps\vsigma(\vTheta_s)\D \mW_s\right)\otimes\left(\int_0^\eps \vb (\vTheta_s)\D s+\int_0^\eps\vsigma(\vTheta_s)\D \mW_s\right)\\
&~\otimes\left(\int_0^\eps \vb (\vTheta_s)\D s+\int_0^\eps\vsigma(\vTheta_s)\D \mW_s\right)\Bigg],
\end{align*}
and its entry can be categorized into four types.  The first one is the pure drift part, i.e., 
\[\int_0^\eps \vb (\vTheta_s)\D s\otimes\int_0^\eps \vb (\vTheta_s)\D s\otimes\int_0^\eps \vb (\vTheta_s)\D s,\]
then by application of the mean value theorem and the fact that  $\nabla_{\vTheta}L_{\fS}(\vTheta)$,  $\nabla^2_{\vTheta}L_{\fS}(\vTheta)$, $\nabla^3_{\vTheta}L_{\fS}(\vTheta)$, and $\mSigma\left(\vTheta\right)$  can be
bounded above by the second, fourth and sixth moments of the solution   to SDE \eqref{eq...text...SME...SME...SDE...Abstract...appendix}, we obtain that 
\begin{align*}
  &\Exp \int_0^\eps \vb (\vTheta_s)\D s\otimes\int_0^\eps \vb (\vTheta_s)\D s\otimes\int_0^\eps \vb (\vTheta_s)\D s \\
=&\eps^3  \Exp\vb (\widetilde{\vTheta}_s)\otimes\vb (\widetilde{\vTheta}_s)\otimes\vb (\widetilde{\vTheta}_s)=\fO(\eps^3).
\end{align*} 
The second one is the pure noise part, i.e., 
\[\left( \int_0^\eps\vsigma(\vTheta_s)\D \mW_s\right) \otimes\left(  \int_0^\eps\vsigma(\vTheta_s)\D \mW_s\right) \otimes\left(  \int_0^\eps\vsigma(\vTheta_s)\D \mW_s\right),\]
and as the odd moments of  zero mean Gaussian variables are zero, hence we have 
\[\Exp\left[\left( \int_0^\eps\vsigma(\vTheta_s)\D \mW_s\right) \otimes\left(  \int_0^\eps\vsigma(\vTheta_s)\D \mW_s\right) \otimes\left(  \int_0^\eps\vsigma(\vTheta_s)\D \mW_s\right)\right]=\mzero,\]
the third and fourth one are both of  the mixed part, for the third one 
\[\int_0^\eps \vb (\vTheta_s)\D s\otimes\int_0^\eps \vb (\vTheta_s)\D s\otimes\left( \int_0^\eps\vsigma(\vTheta_s)\D \mW_s\right) ,\]
whose expectation is of course zero since the drift part and the noise part is independent,  and the fact the odd moments of  zero mean Gaussian variables are zero, 
and for the fourth one 
\[\int_0^\eps \vb (\vTheta_s)\D s\otimes\left( \int_0^\eps\vsigma(\vTheta_s)\D \mW_s\right)\otimes\left( \int_0^\eps\vsigma(\vTheta_s)\D \mW_s\right) ,\]
we obtain that 
\begin{align*}
  &\Exp\left[\int_0^\eps \vb (\vTheta_s)\D s\otimes\left( \int_0^\eps\vsigma(\vTheta_s)\D \mW_s\right)\otimes\left( \int_0^\eps\vsigma(\vTheta_s)\D \mW_s\right) \right]\\
=&\eps   \Exp\vb (\widetilde{\vTheta}_s)\otimes\Exp\left[\left( \int_0^\eps\vsigma(\vTheta_s)\D \mW_s\right)\otimes\left( \int_0^\eps\vsigma(\vTheta_s)\D \mW_s\right)\right]=\fO(\eps^3).
\end{align*} 
As we denote \[\Bar{R}^3({\vTheta}_0):=\Exp\left[(\vTheta_\eps-\vTheta_{0})\otimes(\vTheta_\eps-\vTheta_{0})  \otimes(\vTheta_\eps-\vTheta_{0})\right],\] then we obtain that  \[\Norm{\mathrm{vec}(\Bar{R}^3({\vTheta}_0))}_2\leq \eps^3C(T^{\ast},\vTheta_0). \]
Hence we may apply the mean value theorem to \eqref{eq...text...Semigroup...OrderTwo...BarR_2} and obtain that  
\begin{align*}
\Abs{\Bar{E}_{\eps}^2({\vTheta}_0)}  &=\Big|\int_0^\eps\int_{0}^s v \fL\left(\fL \left<\nabla_{\vTheta} g(\vTheta_0), \vb  \right>\right)(\widetilde{\vTheta}_u)\D v \D s\\
&~~+\frac{1}{6}\nabla_{\vTheta}^3 g( {\vTheta}_{0}):\Exp\left[(\vTheta_\eps-\vTheta_{0})\otimes(\vTheta_\eps-\vTheta_{0})  \otimes(\vTheta_\eps-\vTheta_{0})\right]\\
&~~+\frac{1}{24}\Exp\left[\nabla_{\vTheta}^4 g( \widetilde{\vTheta}_{0}):(\vTheta_\eps-\vTheta_{0})\otimes(\vTheta_\eps-\vTheta_{0})  \otimes(\vTheta_\eps-\vTheta_{0})\otimes(\vTheta_\eps-\vTheta_{0})\right]\Big|\\
&\leq \int_0^\eps\int_{0}^s v \Norm{g}_{C^6} C(T^{\ast},\vTheta_0)\D v \D s+\frac{1}{6}\Norm{g}_{C^6} \eps^3C(T^{\ast},\vTheta_0) \\
&~~+\frac{1}{24}\Norm{g}_{C^6}\Norm{\vTheta_\eps-\vTheta_{0}}_2^4\\
&=\frac{\eps^3}{6} \Norm{g}_{C^6} C(T^{\ast},\vTheta_0)+\frac{1}{6}\Norm{g}_{C^6} \eps^3C(T^{\ast},\vTheta_0)\\&~~+\frac{1}{24}\Norm{g}_{C^6} \Exp\Norm{\int_0^\eps \vb (\vTheta_s)\D s+\int_0^\eps\vsigma(\vTheta_s)\D \mW_s}_2^4\\
&\leq \eps^3\Norm{g}_{C^6} C(T^{\ast},\vTheta_0) +\frac{1}{6}\Norm{g}_{C^6} \eps^3C(T^{\ast},\vTheta_0)\\&~~+\frac{4}{24} \Norm{g}_{C^6} \eps^3 \Exp\Norm{ \nabla_{\vTheta}L_{\fS} (\widetilde{\vTheta}_0) }_2^2+ \frac{4}{24} \Norm{g}_{C^6}    \Exp \Norm{\int_0^\eps\vsigma(\vTheta_s)\D \mW_s}_2^4\\
&\leq \eps^3\Norm{g}_{C^6} C(T^{\ast},\vTheta_0) +\frac{1}{6}\Norm{g}_{C^6} \eps^3C(T^{\ast},\vTheta_0)\\&~~+\frac{4}{24} \Norm{g}_{C^6} \eps^3 \Exp\Norm{ \nabla_{\vTheta}L_{\fS} (\widetilde{\vTheta}_0) }_2^2+ \frac{C}{24} \Norm{g}_{C^6}    \Exp \int_0^\eps\Norm{\vsigma(\vTheta_s)}_{\mathrm{F}}^4\D  s\\
&\leq \eps^3  \Norm{g}_{C^6} C(T^{\ast},\vTheta_0)+\frac{1}{6}\Norm{g}_{C^6} \eps^3C(T^{\ast},\vTheta_0)\\&~~ +\eps^3\Norm{g}_{C^6} \Exp\Norm{ \nabla_{\vTheta}L_{\fS} (\widetilde{\vTheta}_0) }_2^2 +C\Norm{g}_{C^6}   \eps\Exp \left[\eps^2 \Norm{\vSigma(\widetilde{\vTheta}_0)}_{\mathrm{F}}^2 \right] \\&\leq \eps^3\Norm{g}_{C^6}C(T^{\ast},\vTheta_0).
\end{align*}
We remark that for  the last but third line we apply the Burkholder-Davis-Gundy inequality. 

To sum up for now,
\begin{align*}
\Exp g(\vtheta_1)-\Exp g(\vtheta_0)& =-\eps\left<\nabla_{\vtheta} g(\vtheta_0),  \nabla_{\vtheta}L_{\fS}(\vtheta)\big|_{\vtheta=\vtheta_0}\right>\\
&~~+ \frac{\eps^2}{2}\nabla_{\vtheta}^2 g( {\vtheta}_0):\left( \nabla_{\vtheta}L_{\fS}(\vtheta)\big|_{\vtheta=\vtheta_0} \otimes  \nabla_{\vtheta}L_{\fS}(\vtheta)\big|_{\vtheta=\vtheta_0}+\mSigma(\vtheta_{0})\right)+E_{\eps}^2({\vtheta}_0),
\end{align*}
and 
\begin{align*}
\Exp g(\vTheta_{\eps})-\Exp g(\vTheta_0)&= \eps\left<\nabla_{\vTheta} g(\vTheta_0),  \Exp [\vb (\vTheta_0)] \right>+ \frac{\eps^2}{2}   \fL \left<\nabla_{\vTheta} g(\vTheta_0), \vb  \right>(\vTheta_0)  \\ 
&~~+\frac{1}{2}\nabla_{\vTheta}^2 g({\vTheta}_{0}):\Exp\left[(\vTheta_\eps-\vTheta_{0})\otimes(\vTheta_\eps-\vTheta_{0})\right]+\Bar{E}_{\eps}^2(\vTheta_0)\\
 &= \eps\left<\nabla_{\vTheta} g(\vTheta_0),  \Exp [\vb (\vTheta_0)] \right>+ \frac{\eps^2}{2}   \fL \left<\nabla_{\vTheta} g(\vTheta_0), \vb  \right>(\vTheta_0)  \\ 
&~~+\frac{1}{2}\nabla_{\vTheta}^2 g({\vTheta}_{0}):\Exp\Bigg[\left(\int_0^\eps \vb (\vTheta_s)\D s+\int_0^\eps\vsigma(\vTheta_s)\D \mW_s\right)\\&~~~~~~~~~~~~~~~~~~~~~~~~~\otimes\left(\int_0^\eps \vb (\vTheta_s)\D s+\int_0^\eps\vsigma(\vTheta_s)\D \mW_s\right)\Bigg]+\Bar{E}_{\eps}^2(\vTheta_0)\\
&=\eps\left<\nabla_{\vTheta} g(\vTheta_0),  \Exp [\vb (\vTheta_0)] \right>+ \frac{\eps^2}{2}   \fL \left<\nabla_{\vTheta} g(\vTheta_0), \vb  \right>(\vTheta_0)\\&~~+\frac{1}{2}\nabla_{\vTheta}^2 g({\vTheta}_{0}):\Exp\left[\int_0^\eps \vb (\vTheta_s)\D s\otimes\int_0^\eps \vb (\vTheta_s)\D s \right]\\
&~~+\frac{1}{2}\nabla_{\vTheta}^2 g({\vTheta}_{0}):\Exp\left[\int_0^\eps \vsigma (\vTheta_s)\D \mW_s\otimes\int_0^\eps \vsigma (\vTheta_s)\D \mW_s\right]+\Bar{E}_{\eps}^2(\vTheta_0)\\
&=\eps\left<\nabla_{\vTheta} g(\vTheta_0),  \Exp [\vb (\vTheta_0)] \right>+ \frac{\eps^2}{2}   \fL \left<\nabla_{\vTheta} g(\vTheta_0), \vb  \right>(\vTheta_0)\\&~~+\frac{1}{2}\nabla_{\vTheta}^2 g({\vTheta}_{0}):\Exp\left[\int_0^\eps\int_0^\eps \vb (\vTheta_s)\otimes  \vb (\vTheta_u)\D s\D u \right]\\
&~~+\frac{1}{2}\nabla_{\vTheta}^2 g({\vTheta}_{0}):\Exp\left[\int_0^\eps \vsigma (\vTheta_s)\D \mW_s\otimes\int_0^\eps \vsigma (\vTheta_s)\D \mW_s\right]+\Bar{E}_{\eps}^2(\vTheta_0),
\end{align*}
we observe that 
\begin{align*}
&\frac{1}{2}\nabla_{\vTheta}^2 g({\vTheta}_{0}):\Exp\left[\int_0^\eps \vsigma (\vTheta_s)\D \mW_s\otimes\int_0^\eps \vsigma (\vTheta_s)\D \mW_s\right] \\
=&\Exp\left[\int_0^\eps    \frac{1}{2}\nabla_{\vTheta}^2 g({\vTheta}_{0}):\vsigma \vsigma ^\T (\vTheta_s)\D s\right]\\
=&\frac{\eps}{2}\Exp\left[\int_0^\eps    \nabla_{\vTheta}^2 g({\vTheta}_{0}):\mSigma(\vTheta_s)\D s\right],
\end{align*}
thus 
\begin{align*}
\Exp g(\vTheta_{\eps})-\Exp g(\vTheta_0)& =\eps\left<\nabla_{\vTheta} g(\vTheta_0),  \Exp [\vb (\vTheta_0)] \right>+ \frac{\eps^2}{2}   \fL \left<\nabla_{\vTheta} g(\vTheta_0), \vb  \right>(\vTheta_0)\\&~~+\frac{1}{2}\nabla_{\vTheta}^2 g({\vTheta}_{0}):\Exp\left[\int_0^\eps\int_0^\eps \vb (\vTheta_s)\otimes  \vb (\vTheta_u)\D s\D u \right]\\
&~~+\frac{\eps}{2}\Exp\left[\int_0^\eps    \nabla_{\vTheta}^2 g({\vTheta}_{0}):\mSigma(\vTheta_s)\D s\right]+\Bar{E}_{\eps}^2(\vTheta_0).
\end{align*}
Since 
\begin{align*}
&\nabla_{\vTheta}^2 g({\vTheta}_{0}):\Exp\left[ \vb (\vTheta_s)\otimes \vb (\vTheta_u) \right] 
\\
=&\nabla^2_{\vTheta} g(\vTheta_0): \Exp [\vb (\vTheta_s)\otimes\vb (\vTheta_0)]   + \int_{0}^u \fL \left( \nabla^2_{\vTheta} g(\vTheta_0):\vb (\vTheta_s)\otimes\vb (\vTheta_v)\right) \D v \\
=&\nabla^2_{\vTheta} g(\vTheta_0): \Exp [\vb (\vTheta_0)\otimes\vb (\vTheta_0)] + \int_{0}^s \fL \left( \nabla^2_{\vTheta} g(\vTheta_0):\vb (\vTheta_w)\otimes\vb (\vTheta_0)\right) \D w \\
&~+\int_{0}^u \fL \left( \nabla^2_{\vTheta} g(\vTheta_0):\vb (\vTheta_s)\otimes\vb (\vTheta_v)\right) \D v, 
\end{align*}
and  since 
\begin{align*}
&\nabla_{\vTheta}^2 g({\vTheta}_{0}):\Exp\left[   \mSigma(\vTheta_s) \right] 
\\
=&\nabla^2_{\vTheta} g(\vTheta_0): \Exp\left[   \mSigma(\vTheta_0) \right]  + \int_{0}^s \fL \left( \nabla^2_{\vTheta} g(\vTheta_0):\mSigma(\vTheta_s) \right) \D v, 
\end{align*}
we are one step away to finish our proof,  
\begin{align*}
\Exp g(\vTheta_{\eps})-\Exp g(\vTheta_0)& =\eps\left<\nabla_{\vTheta} g(\vTheta_0),  \Exp [\vb (\vTheta_0)] \right>+ \frac{\eps^2}{2}   \fL \left<\nabla_{\vTheta} g(\vTheta_0), \vb  \right>(\vTheta_0)\\&~~+\frac{1}{2}\nabla_{\vTheta}^2 g({\vTheta}_{0}):\Exp\left[\int_0^\eps\int_0^\eps \vb (\vTheta_0)\otimes  \vb (\vTheta_0)\D s\D u \right]\\
&~~+\frac{\eps}{2}\Exp\left[\int_0^\eps    \nabla_{\vTheta}^2 g({\vTheta}_{0}):\mSigma(\vTheta_0)\D s\right]+\Bar{E}_{\eps}^2(\vTheta_0),
\end{align*}
where we misuse our notations for $\Bar{E}_{\eps}^2(\vTheta_0)$, and the term  \begin{align*}
&\int_0^\eps\int_0^\eps\int_{0}^s \fL \left( \nabla^2_{\vTheta} g(\vTheta_0):\vb (\vTheta_w)\otimes\vb (\vTheta_0)\right) \D w \D s\D u\\
 &~~+\int_0^\eps\int_0^\eps\int_{0}^u \fL \left( \nabla^2_{\vTheta} g(\vTheta_0):\vb (\vTheta_s)\otimes\vb (\vTheta_v)\right) \D v  \D s\D u   \\
 &~~+\int_0^\eps \int_{0}^s \fL \left( \nabla^2_{\vTheta} g(\vTheta_0):\mSigma(\vTheta_s) \right) \D v \D s,
\end{align*} 
is included,  and $\Bar{E}_{\eps}^2(\vTheta_0)$ is still of order $\fO(\eps^3)$ by similar reasoning and we omit its demonstration. Thus
\begin{align*}
\Exp g(\vTheta_{\eps})-\Exp g(\vTheta_0)& =\eps\left<\nabla_{\vTheta} g(\vTheta_0),  \Exp [\vb (\vTheta_0)] \right>+ \frac{\eps^2}{2}   \left<\vb (\vTheta_0), \nabla_{\vTheta}\left<\nabla_{\vTheta} g(\vTheta_0), \vb  \right>(\vTheta_0)\right>\\
&~~+\frac{\eps^3}{2}\mSigma(\vTheta_0):  \nabla_{\vTheta}^2\left<\nabla_{\vTheta} g(\vTheta_0), \vb  \right>(\vTheta_0)\\
&~~+\frac{\eps^2}{2}\nabla_{\vTheta}^2 g({\vTheta}_{0}):\Exp\left[  \vb (\vTheta_0)\otimes  \vb (\vTheta_0) \right]\\
&~~+\frac{\eps^2}{2}\Exp\left[    \nabla_{\vTheta}^2 g({\vTheta}_{0}):\mSigma(\vTheta_0)\right]+\Bar{E}_{\eps}^2(\vTheta_0),
\end{align*}
and recall that since we choose 
\begin{align*}
\vb\left(\vTheta\right)& =   - \nabla_{\vTheta}\left(L_{\fS}(\vTheta)+\frac{\eps}{4}\Norm{\nabla_{\vTheta}L_S(\vTheta)}_2^2\right),\\     
\vsigma\left(\vTheta\right)& =\sqrt{\eps}\left(\mSigma\left(\vTheta\right)\right)^{\frac{1}{2}},
\end{align*}
then
\begin{align*}
\Exp g(\vTheta_{\eps})-\Exp g(\vTheta_0)& =-\eps\left<\nabla_{\vTheta} g(\vTheta_0),  \nabla_{\vTheta}\left(L_{\fS}(\vTheta)\right)\mid_{\vTheta=\vTheta_0}\right>\\
&~~-\frac{\eps^2}{4}\left<\nabla_{\vTheta} g(\vTheta_0),  \nabla_{\vTheta}\left(\Norm{\nabla_{\vTheta}L_S(\vTheta)}_2^2\right)\mid_{\vTheta=\vTheta_0}\right>\\
&~~+ \frac{\eps^2}{2}   \left< \nabla_{\vTheta}\left(L_{\fS}(\vTheta)\right)\mid_{\vTheta=\vTheta_0}, \nabla_{\vTheta}\left<\nabla_{\vTheta} g(\vTheta_0),  \nabla_{\vTheta}\left(L_{\fS}(\vTheta)\right) \right>\mid_{\vTheta=\vTheta_0}\right>\\
&~~+\frac{\eps^2}{2}\nabla_{\vTheta}^2 g({\vTheta}_{0}):\left(\nabla_{\vTheta}\left(L_{\fS}(\vTheta)\right)\mid_{\vTheta=\vTheta_0})\otimes \nabla_{\vTheta}\left(L_{\fS}(\vTheta)\right)\mid_{\vTheta=\vTheta_0}\right)\\
&~~+\frac{\eps^2}{2}      \nabla_{\vTheta}^2 g({\vTheta}_{0}):\mSigma(\vTheta_0) +\Bar{E}_{\eps}^2(\vTheta_0)\\
&=-\eps\left<\nabla_{\vTheta} g(\vTheta_0),  \nabla_{\vTheta}\left(L_{\fS}(\vTheta)\right)\mid_{\vTheta=\vTheta_0}\right>\\
&~~+ \frac{\eps^2}{2}\nabla_{\vTheta}^2 g({\vTheta}_{0}):\left(\nabla_{\vTheta}\left(L_{\fS}(\vTheta)\right)\mid_{\vTheta=\vTheta_0})\otimes \nabla_{\vTheta}\left(L_{\fS}(\vTheta)\right)\mid_{\vTheta=\vTheta_0}\right)\\
&~~+\frac{\eps^2}{2}      \nabla_{\vTheta}^2 g({\vTheta}_{0}):\mSigma(\vTheta_0) +\Bar{E}_{\eps}^2(\vTheta_0),
\end{align*}
thus, we have 
\begin{align*}
\Abs{\Exp g(\vtheta_1)-\Exp g(\vTheta_\eps)}&=\Big|\Exp g(\vtheta_0)-\eps\left<\nabla_{\vtheta} g(\vtheta_0),  \nabla_{\vtheta}L_{\fS}(\vtheta)\big|_{\vtheta=\vtheta_0}\right>\\
&~~+ \frac{\eps^2}{2}\nabla_{\vtheta}^2 g( {\vtheta}_0):\left( \nabla_{\vtheta}L_{\fS}(\vtheta)\big|_{\vtheta=\vtheta_0} \otimes  \nabla_{\vtheta}L_{\fS}(\vtheta)\big|_{\vtheta=\vtheta_0}+\mSigma(\vtheta_{0})\right)\\
&~~+E_{\eps}^2({\vtheta}_0) \\
&~~-\Exp g(\vTheta_0)+\eps\left<\nabla_{\vTheta} g(\vTheta_0),  \nabla_{\vTheta}\left(L_{\fS}(\vTheta)\right)\mid_{\vTheta=\vTheta_0}\right>\\
&~~- \frac{\eps^2}{2}\nabla_{\vTheta}^2 g({\vTheta}_{0}):\left(\nabla_{\vTheta}\left(L_{\fS}(\vTheta)\right)\mid_{\vTheta=\vTheta_0})\otimes \nabla_{\vTheta}\left(L_{\fS}(\vTheta)\right)\mid_{\vTheta=\vTheta_0}\right)\\
&~~-\frac{\eps^2}{2}      \nabla_{\vTheta}^2 g({\vTheta}_{0}):\mSigma(\vTheta_0) +\Bar{E}_{\eps}^2(\vTheta_0) \Big|\\
&\leq \Abs{E_{\eps}^2({\vtheta}_0)}+\Abs{\Bar{E}_{\eps}^2(\vTheta_0) }\\
&\leq  \eps^3\Norm{g}_{C^6} C(T^{\ast},\vtheta_0,\eps_0)+\eps^3\Norm{g}_{C^6}C(T^{\ast},\vTheta_0)\\
&=\fO(\eps^3).
\end{align*}

For the $N$-th step iteration,  since
\begin{align*}
\Abs{\Exp g(\vtheta_{N})-\Exp g(\vTheta_{\eps N})} &=\Abs{\fP_\eps^N g(\vtheta_0)-\fP_{ \eps N } g(\vTheta_0)},  
\end{align*}
and the RHS of the above equation can be written into a telescoping sum as 
\begin{align*}
\fP_\eps^N g(\vtheta_0)-\fP_{ \eps N } g(\vTheta_0)&=\sum_{l=1}^{N}\left(\fP_\eps^{N-l+1}\circ\fP_{(l-1)\eps}  g(\vtheta_0)-\fP_\eps^{N-l}\circ\fP_{l\eps} g(\vTheta_0)  \right),
\end{align*}
hence by application of Proposition \ref{prop...Contractive}, we obtain that 
\begin{align*}
\Abs{\Exp g(\vtheta_{N})-\Exp g(\vTheta_{\eps N})}  &\leq  \sum_{l=1}^{N}\Abs{\fP_\eps^{N-l+1}\circ\fP_{(l-1)\eps}  g(\vtheta_0)-\fP_\eps^{N-l}\circ\fP_{l\eps} g(\vTheta_0)  }\\
&\leq \sum_{l=1}^{N}\Abs{\fP_\eps^{N-l}\circ\left(\fP_\eps^{1}\circ\fP_{(l-1)\eps}  - \fP_\eps\circ\fP_{(l-1)\eps}   \right) g(\vTheta_0)},
\end{align*}
since $\left(\fP_\eps^{1}\circ\fP_{(l-1)\eps}  - \fP_\eps\circ\fP_{(l-1)\eps}   \right) g(\vTheta_0)$ can be regarded as $\fL^1(\sR^D)$ if we choose measure $\mu$  to be the delta measure concentrated on $\vTheta_0$. i.e., \[\mu:=\delta_{\vTheta_0},\] hence by the conctration property of Markov operators, we obtain further that 
\begin{align*}
\Abs{\Exp g(\vtheta_{N})-\Exp g(\vTheta_{\eps N})} 
&\leq\sum_{l=1}^{N}\Abs{ \left(\fP_\eps^{1}\circ\fP_{(l-1)\eps}  - \fP_\eps\circ\fP_{(l-1)\eps}   \right) g(\vTheta_0)}\\
&\leq   \sum_{l=1}^{N}\Abs{\fP_\eps^{1}   g(\vTheta_{(l-1)\eps} )-\fP_\eps g(\vTheta_{(l-1)\eps} )}.
\end{align*}
By taking   expectation conditioned on $\vTheta_{(l-1)\eps}$, then similar to the relation \eqref{eq...proof...OrderOne...Crutial}, the following holds 
\begin{equation*}
\begin{aligned}
\Abs{ \fP_\eps^1 g(\vTheta_{(l-1)\eps})-  \fP_\eps g(\vTheta_{(l-1)\eps}) }&=\Exp\left[\left[\Abs{\Exp g(\vtheta_l)-\Exp g(\vTheta_\eps l)}\Big| \vTheta_{(l-1)\eps}\right]\right]\\
&\leq \Exp\Abs{E_{\eps}^2(\vTheta_{(l-1)\eps})}+\Exp\Abs{\Bar{E}_{\eps}^2(\vTheta_{(l-1)\eps})}\\
&\leq  \eps^3\Norm{g}_{C^6} C(T^{\ast},\vtheta_0,\eps_0)+\eps^3\Norm{g}_{C^6}C(T^{\ast},\vTheta_0)\\
&=\fO(\eps^3).
\end{aligned}
\end{equation*}
We remark that the last line of the above relation is essentially based on Assumption \ref{Append...theOnlyAssumption}, since  $ \Exp\Abs{E_{\eps}^2(\vTheta_{(l-1)\eps})}$ and $\Exp\Abs{\Bar{E}_{\eps}^2(\vTheta_{(l-1)\eps})}$ can be  bounded above by the second, fourth and sixth moments of the solution to SDE \eqref{eq...text...SME...SME...SDE...Abstract...appendix}, hence we may apply dominated convergence theorem to obtain the  last line of the above relation.

To sum up, as
\begin{align*}
 \Abs{\fP_\eps^N g(\vtheta_0)-\fP_{ \eps N } g(\vTheta_0)}  &\leq  \sum_{l=1}^{N}\Abs{\fP_\eps^{N-l+1}\circ\fP_{(l-1)\eps}  g(\vtheta_0)-\fP_\eps^{N-l}\circ\fP_{l\eps} g(\vTheta_0)  } =N\fO(\eps^3),
\end{align*}
hence for  $N=N_{T,\eps}$, 
\[
 \Abs{\fP_\eps^N g(\vtheta_0)-\fP_{ \eps N } g(\vTheta_0)}   =N\fO(\eps^3)=N\eps \fO(\eps)\leq T\fO(\eps^2)=\fO(\eps^2).
\]
\end{proof}

\newpage

\section{Validation for Assumption 1}\label{section...ValidationforAssump}
In this section, we endeavor to demonstrate the validity of Assumption 1. We begin this section by making some estimates on the modified loss $L_{\fS}$ and covariance $\mSigma$.
\subsection{Estimates on Modified Loss and Covariance}\label{subsection...EstimatesonModified+Covariance}
For the modified loss, recall that  $\vtheta=\mathrm{vec}(\{\vq_r\}_{r=1}^m)=\mathrm{vec}\left(\{(a_r, \vw_r)\}_{r=1}^m\right)$,   as we have
\begin{align*}
 \nabla_{\vq_k}L_S(\vTheta)&= \frac{1}{n}\sum_{i=1}^ne_i\nabla_{\vq_k}\left(a_k\sigma(\vw_k^{\T}\vx_i)\right)+ \frac{1-p}{np}\sum_{i=1}^na_{k}\sigma(\vw_{k}^{\T}\vx_i) \nabla_{\vq_k}\left(a_k\sigma(\vw_k^{\T}\vx_i)\right),
\end{align*}
and under the usual  convention that for all $i\in[n]$, 
\[  \frac{1}{c}\leq \Norm{\vx_{i}}_2, \quad\Abs{y_{i}}\leq c,\] 
where $c$ is some universal constant, and that $\sigma(0)=0$,  we obtain that
\begin{align*}
\Abs{e_i}&=\Abs{\sum_{r=1}^m a_{r}\sigma(\vw_{r}^{\T}\vx_i)-y_i}\\
&\leq 1+\sum_{r=1}^m \Abs{a_{r}}\Norm{\vw_{r}}_2\\
&\leq 1+\frac{1}{2}\sum_{r=1}^m \left(\Abs{a_{r}}^2+\Norm{\vw_{r}}_2^2\right)\\
&\leq 1+\Norm{\vTheta}_2^2,
\end{align*}
hence 
\begin{align*}
 \Norm{\nabla_{\vq_k}L_S(\vTheta)}_2&\leq  \left( 1+\Norm{\vTheta}_2^2\right) \Norm{\vq_k}_2 +\frac{1-p}{p} \Norm{\vq_k}_2^3,
\end{align*}
thus we have  
\begin{align*}
 \Norm{\nabla_{\vTheta}L_S(\vTheta)}_2&\leq  \left( 1+\Norm{\vTheta}_2^2\right) \Norm{\vTheta}_2 +\frac{1-p}{p} \Norm{\vTheta}_2^3\\
&\leq C_p(1+\Norm{\vTheta}_2^3).
\end{align*}
Moreover,   
since
\begin{align*}
\nabla_{\vTheta}^2L_S(\vTheta) &=\frac{1}{n}\sum_{i=1}^n \left(\nabla_{\vTheta}e_i\otimes \nabla_{\vTheta}e_i +
e_i\nabla_{\vTheta}^2e_i\right)\\
&~~+\frac{1-p}{np}\sum_{i=1}^n \mathrm{diag}\left\{\nabla_{\vq_k}^2\left(a_k^2\sigma(\vw_k^{\T}\vx_i)^2\right)\right\},
\end{align*}
as  we denote only for now $\times$ as matrix multiplication, 
\begin{align*} 
& \nabla_{\vTheta}^2L_S(\vTheta)\nabla_{\vTheta}L_S(\vTheta) \\
=&\left(\frac{1}{n}\sum_{i=1}^n \left(\nabla_{\vTheta}e_i\otimes \nabla_{\vTheta}e_i +
e_i\nabla_{\vTheta}^2e_i\right) +\frac{1-p}{np}\sum_{i=1}^n \mathrm{diag}\left\{\nabla_{\vq_k}^2\left(a_k^2\sigma(\vw_k^{\T}\vx_i)^2\right)\right\}\right)\\
&~~\times\left(\frac{1}{n}\sum_{i=1}^ne_i\nabla_{\vTheta}e_i+ \frac{1-p}{np}\sum_{i=1}^n  \nabla_{\vTheta}\left(a_k^2\sigma(\vw_k^{\T}\vx_i)^2\right)\right),
\end{align*}
then the components in $ \nabla_{\vTheta}^2L_S(\vTheta)\nabla_{\vTheta}L_S(\vTheta)  $  can be categorized into six different types: Firstly,
\begin{align*}
&\Norm{\left(\nabla_{\vTheta}e_i\otimes \nabla_{\vTheta}e_i\right)e_j\nabla_{\vTheta}e_j}_2\\
\leq& \Abs{e_j}\Norm{\nabla_{\vTheta}e_i}_2^2\Norm{\nabla_{\vTheta}e_j}_2\\
 \leq &\left( 1+\Norm{\vTheta}_2^2\right)\Norm{\vTheta}_2^3\\
\leq &\left( 1+\Norm{\vTheta}_2^5\right).
\end{align*}
Secondly,
\begin{align*}
&\Norm{\left(e_i\nabla_{\vTheta}^2e_i\right)e_j\nabla_{\vTheta}e_j}_2\\
\leq &\left( 1+\Norm{\vTheta}_2^2\right)^2\Norm{\nabla_{\vTheta}^2e_i}_{2\to 2}\Norm{\nabla_{\vTheta}e_j}_2\\
\leq & \left( 1+\Norm{\vTheta}_2^4\right)\Norm{\vTheta}_2^2\\
\leq &\left( 1+\Norm{\vTheta}_2^6\right). 
\end{align*}
Thirdly,
\begin{align*}
&\Norm{\left(\mathrm{diag}\left\{\nabla_{\vq_k}^2\left(a_k^2\sigma(\vw_k^{\T}\vx_i)^2\right)\right\}\right)e_j\nabla_{\vTheta}e_j}_2\\
\leq &\left( 1+\Norm{\vTheta}_2^2\right)\Norm{\mathrm{diag}\left\{\nabla_{\vq_k}^2\left(a_k^2\sigma(\vw_k^{\T}\vx_i)^2\right)\right\}}_{2\to 2}\Norm{\vTheta}_2\\
\leq & \left( 1+\Norm{\vTheta}_2^2\right)\left( 1+\Norm{\vTheta}_2^3\right)\Norm{\vTheta}_2\\
\leq &\left( 1+\Norm{\vTheta}_2^6\right).
\end{align*}
Fourthly, 
\begin{align*}
&\Norm{\left(\nabla_{\vTheta}e_i\otimes \nabla_{\vTheta}e_i\right)\nabla_{\vTheta}\left(a_k^2\sigma(\vw_k^{\T}\vx_j)^2\right)}_2\\
\leq &\Norm{\nabla_{\vTheta}e_i}_2^2\Norm{\vTheta}_2^3 \\ 
\leq &\left( 1+\Norm{\vTheta}_2^5\right).
\end{align*}
Fifthly,
\begin{align*}
&\Norm{\left(e_i\nabla_{\vTheta}^2e_i\right)\nabla_{\vTheta}\left(a_k^2\sigma(\vw_k^{\T}\vx_j)^2\right)}_2\\
\leq &\left( 1+\Norm{\vTheta}_2^2\right) \Norm{\nabla_{\vTheta}^2e_i}_{2\to 2}\Norm{\vTheta}_2^3\\ 
\leq &\left( 1+\Norm{\vTheta}_2^2\right)\Norm{\vTheta}_2^4\\ 
\leq &\left( 1+\Norm{\vTheta}_2^6\right). 
\end{align*}
Finally, 
\begin{align*}
&\Norm{\left(\mathrm{diag}\left\{\nabla_{\vq_k}^2\left(a_k^2\sigma(\vw_k^{\T}\vx_i)^2\right)\right\}\right)\nabla_{\vTheta}\left(a_k^2\sigma(\vw_k^{\T}\vx_j)^2\right)}_2\\
\leq & \Norm{\mathrm{diag}\left\{\nabla_{\vq_k}^2\left(a_k^2\sigma(\vw_k^{\T}\vx_i)^2\right)\right\}}_{2\to 2}\Norm{\vTheta}_2^3\\
\leq & \left( 1+\Norm{\vTheta}_2^3\right)\Norm{\vTheta}_2^3\\
\leq &\left( 1+\Norm{\vTheta}_2^6\right).
\end{align*}
To sum up, for the drift term  $\vb(\vTheta)$, regardless of the choice of first order or second order accuracy, we obtain that 
\[\Norm{\vb(\vTheta)}_2\leq 1+\Norm{\vTheta}_2^6.\]

As for the covariance $\mSigma$,  recall that  $\vtheta=\mathrm{vec}(\{\vq_r\}_{r=1}^m)=\mathrm{vec}\left(\{(a_r, \vw_r)\}_{r=1}^m\right)$, then   we obtain   that the covariance $\mSigma$ reads
 \[
\mSigma=\left[\begin{array}{cccc}
 \mSigma_{11} &  \mSigma_{12} &  \cdots & \mSigma_{1m}  \\
\mSigma_{21} &  \mSigma_{22} &  \cdots & \mSigma_{2m}\\ 
\vdots& \vdots&\vdots&\vdots\\
\mSigma_{m1} &  \mSigma_{m2} &  \cdots & \mSigma_{mm}
\end{array}\right].
\]
 For each  $  k \in [m]$, we obtain that 
\begin{align*}
 \mSigma_{kk}(\vTheta )  =&\left(\frac{1}{p}-1\right)\left(\frac{1}{n}\sum_{i=1}^n\left( e_{i,\backslash k}+\frac{1}{p}a_{k}\sigma(\vw_{k}^{\T}\vx_i)\right)\nabla_{\vq_k}\left(a_k\sigma(\vw_k^{\T}\vx_i)\right) \right)\\
&~~~~~~~~~~~~~~~~~~~~~~~~~~~~~~{\otimes}\left(\frac{1}{n}\sum_{i=1}^n\left( e_{i,\backslash k}+\frac{1}{p}a_{k}\sigma(\vw_{k}^{\T}\vx_i)\right)\nabla_{\vq_k}\left(a_k\sigma(\vw_k^{\T}\vx_i)\right) \right)\\
&~+\left(\frac{1}{p^2}-\frac{1}{p}\right)\sum_{l=1, l\neq k}^m\left(\frac{1}{n}\sum_{i=1}^na_{l}\sigma(\vw_{l}^{\T}\vx_i)\nabla_{\vq_k}\left(a_k\sigma(\vw_k^{\T}\vx_i)\right)\right)\\
&~~~~~~~~~~~~~~~~~~~~~~~~~~~~~~{\otimes}\left(\frac{1}{n}\sum_{i=1}^na_{l}\sigma(\vw_{l}^{\T}\vx_i)\nabla_{\vq_k}\left(a_k\sigma(\vw_k^{\T}\vx_i)\right)\right),
\end{align*}
and for each $  k, r \in [m]$ with $k \neq r$,  
\begin{align*}
  \mSigma_{kr}(\vTheta)  
=&\left(\frac{1}{p}-1\right)\left(\frac{1}{n}\sum_{i=1}^n\left(e_{i,\backslash k, \backslash r}+\frac{1}{p}a_k\sigma(\vw_k^\T\vx_i)+\frac{1}{p}a_r\sigma(\vw_r^\T\vx_i)\right)\nabla_{\vq_k}\left(a_k\sigma(\vw_k^{\T}\vx_i)\right)\right)\\
&~~~~~~~~~~~~~~~~{\otimes}\left(\frac{1}{n}\sum_{i=1}^na_k\sigma(\vw_k^\T\vx_i)\nabla_{\vq_r}\left(a_r\sigma(\vw_r^{\T}\vx_i)\right)\right)\\
&+\left(\frac{1}{p}-1\right)\left(\frac{1}{n}\sum_{i=1}^na_r\sigma(\vw_r^\T\vx_i)\nabla_{\vq_k}\left(a_k\sigma(\vw_k^{\T}\vx_i)\right)\right) \\
&~~~~~~~~~~~~~~~~{\otimes}\left(\frac{1}{n}\sum_{i=1}^n\left(e_{i,\backslash k, \backslash r}+a_k\sigma(\vw_k^\T\vx_i)+\frac{1}{p}a_r\sigma(\vw_r^\T\vx_i)\right)\nabla_{\vq_r}\left(a_r\sigma(\vw_r^{\T}\vx_i)\right)\right),
\end{align*}
hence we obtain that 
\begin{align*}
\Norm{\mSigma_{kk}(\vTheta)}_{\mathrm{F}}^2&\leq C_p\left(\Abs{e_{i,\backslash k}+\frac{1}{p}a_{k}\sigma(\vw_{k}^{\T}\vx_i)}^2+\sum_{l=1, l\neq k}^m a_l^2\sigma(\vw_l^\T\vx_i)^2\right)\Norm{\nabla_{\vTheta} e_i}_2^2\\
&\leq C_p(1+\Norm{\vTheta}_2^2)^2\Norm{\vTheta}_2^2\\
&\leq (1+\Norm{\vTheta}_2^6),
\end{align*}
and by similar reasoning 
\begin{align*}
\Norm{\mSigma_{kr}(\vTheta)}_{\mathrm{F}}^2 \leq (1+\Norm{\vTheta}_2^6).
\end{align*}
\subsection{Existence, Uniqueness and Moment Estimates of  the Solution to SDE}\label{subsection...Existence+Uniqueness}

Existence of the solution to SDE \eqref{eq...text...SME...SME...SDE...Abstract...appendix} is proved by a truncation procedure:
For each $M \geq 1$, define the truncation function
$$
\vb_M(\vTheta):= \begin{cases}\vb (\vTheta) & \text { if }\Norm{\vTheta}_2 \leq M, \\ 
\vb(M \frac{\vTheta}{\Norm{\vTheta}_2}) & \text { if }\Norm{\vTheta}_2 >M.\end{cases}
$$
We also perform similar truncation to $\vsigma(\vTheta)$ and obtain its truncation $\vsigma_M(\vTheta)$. Then $\vb_M$ and $\vsigma_M$ satisfy the Lipschitz condition   and the linear growth condition, hence by application of the classical results~\citep[Theorem 5.2.1]{oksendal2013stochastic} in SDE, there exists a unique solution $\vTheta_M(\cdot)$   to the truncated  SDE
\begin{equation}\label{eq...truncated}\D \vTheta_t=\vb_M \left(\vTheta_t\right)\D t+\vsigma_M \left(\vTheta_t\right) \D \vW_t, \quad  \vTheta_0=\vTheta(0).\end{equation}
We may choose $M$ large enough, such that 
\[
\Norm{\vTheta_0}_2< M,
\]
and the solution to SDE \eqref{eq...text...SME...SME...SDE...Abstract...appendix} coincides with the solution to SDE \eqref{eq...truncated} at least for a period of time $T^{\ast}>0$ since $\Norm{\vTheta_0}_2< M$. We remark that $T^{\ast}$ is the desired time in Assumption \ref{Append...theOnlyAssumption}. We also remark that not only for any time $t\in\left[0, T^{\ast} \right]$, the second, fourth and sixth  moments of  the solution  to SDE \eqref{eq...text...SME...SME...SDE...Abstract...appendix}  are uniformly bounded with respect to time $t$, but also that for any time $t\in\left[0, T^{\ast} \right]$,  all  moments of  the solution  to SDE \eqref{eq...text...SME...SME...SDE...Abstract...appendix}  are uniformly bounded with respect to time $t$. 

At this point, it is important to discuss that we prove   is that for fixed time $T$, we can take the learning rate $\eps>0$ small
enough so that the SME   is a good approximation of the
distribution of the dropout iterates. What we did not prove is that for fixed $\eps$, the approximations hold for arbitrary time $T$. In particular, it is not hard to construct systems where
for fixed $\eps$, both the SME and the asymptotic expansion fails when time $T$ is large enough.
\subsection{ Moment Estimates of  the Dropout Iteration}\label{subsection...Moments}
Recall that the dropout iteration reads
\[
\vtheta_N= \vtheta_{N-1}-\eps \nabla_{\vtheta}R_{\fS}^\mathrm{drop}\left(\vtheta_{N-1}; \veta_{N}\right),
\]
then we obtain that 
\begin{align*}
\Exp\Norm{\vtheta_N}_2^{2l}&=   \Exp\Norm{\vtheta_{N-1}}_2^{2l}-2l\eps \Exp\left[\Norm{\vtheta_{N-1}}_2^{2l-2}\left<\vtheta_{N-1},\nabla_{\vtheta}R_{\fS}^\mathrm{drop}\left(\vtheta_{N-1}; \veta_{N}\right)\right>\right]   +\fO(\eps^2),
\end{align*}
then for learning rate $\eps$ small enough, we observe that $\{\Exp\Norm{\vtheta_N}_2^{2l}\}_{N\geq 0}$ follows close to the trajectory of a ordinary differential equation~(ODE). Moreover, from the estimates obtained in Section \ref{subsection...EstimatesonModified+Covariance}, 
\begin{align*}
 &\Norm{\vtheta_{N-1}}_2^{2l-2}\left<\vtheta_{N-1},\nabla_{\vtheta}R_{\fS}^\mathrm{drop}\left(\vtheta_{N-1}; \veta_{N}\right)\right> \\
 \leq & \Norm{\vtheta_{N-1}}_2^{2l-1}\Norm{\nabla_{\vtheta}R_{\fS}^\mathrm{drop}\left(\vtheta_{N-1}; \veta_{N}\right)}_2\\
=&\Norm{\vtheta_{N-1}}_2^{2l-1}\Abs{e_{i}^N} \Norm{\nabla_{\vtheta}e_{i}^N}_2\\
\leq &\Norm{\vtheta_{N-1}}_2^{2l-1} C_p(1+\Norm{\vtheta_{N-1}}_2^2)\Norm{\vtheta_{N-1}}_2\\
\leq & C_p(1+\Norm{\vtheta_{N-1}}_2^{2l+2}),
\end{align*}
we remark that as the above estimates hold almost surely, 
then  for learning rate $\eps$ small enough, we may apply Gronwall inequality to 
$ \{\Exp\Norm{\vtheta_N}_2^{2l}\}_{N\geq 0}$ and shows that for some $N^{\ast}$,  all moments of  the dropout iterations  are uniformly
bounded with respect to $N$, since for   the  ODE
\begin{equation}\label{eq:ode}
\frac{\D u}{\D t}=1+u^{1+\lambda},\quad u_0:=u(0),
\end{equation}
with $\lambda >0$.   There exists time $T^{\ast}>0$, such that  for any time $t\in\left[0, T^{\ast} \right]$, its solution $\{u_t\}_{t\geq 0}$ is  uniformly
bounded with respect to time $t$. And since for small enough learning rate,  all moments of  the dropout iterations  $ \{\Exp\Norm{\vtheta_N}_2^{2l}\}_{N\geq 0}$ follows close to the trajectory of ODEs of \eqref{eq:ode} type, hence all these moments are also uniformly
bounded with respect to $N$.

\newpage

\section{Some Computations on the Covariance}\label{section...ComputationalDetails}
Once again, since $\vtheta=\mathrm{vec}(\{\vq_r\}_{r=1}^m)=\mathrm{vec}\left(\{(a_r, \vw_r)\}_{r=1}^m\right)$, then the covariance of $\nabla_{\vtheta}\RS^\mathrm{drop}\left(\vtheta_{N-1}; \veta_{N}\right)$ equals to the matrix  $ \mSigma(\vtheta_{N-1})$, and as we denote for any $k,r\in[m]$,
 \[\mSigma_{kr}(\vtheta_{N-1}):=\mathrm{Cov}\left( \nabla_{\vq_k}\RS^\mathrm{drop}\left(\vtheta_{N-1}; \veta_{N}\right), \nabla_{\vq_r}\RS^\mathrm{drop}\left(\vtheta_{N-1}; \veta_{N}\right)\right), \]
then 
 \[
\mSigma=\left[\begin{array}{cccc}
 \mSigma_{11} &  \mSigma_{12} &  \cdots & \mSigma_{1m}  \\
\mSigma_{21} &  \mSigma_{22} &  \cdots & \mSigma_{2m}\\ 
\vdots& \vdots&\vdots&\vdots\\
\mSigma_{m1} &  \mSigma_{m2} &  \cdots & \mSigma_{mm}
\end{array}\right].
\]
\subsection{Elements on the Diagonal}
In this part, we compute $\mSigma_{kk}$ for all $k\in[m]$.
\begin{align*}
 \mSigma_{kk}(\vtheta_{N-1})&=\mathrm{Cov}\left( \nabla_{\vq_k}\RS^\mathrm{drop}\left(\vtheta_{N-1}; \veta_{N}\right), \nabla_{\vq_k}\RS^\mathrm{drop}\left(\vtheta_{N-1}; \veta_{N}\right)\right)   \\ 
&=\frac{1}{n^2}\sum_{i,j=1}^n\mathrm{Cov}\left(e_{i}^N (\veta_N)_{k},  e_{j}^N (\veta_N)_{k}\right)\nabla_{\vq_k}\left(a_k\sigma(\vw_k^{\T}\vx_i)\right) \otimes  \nabla_{\vq_k}\left(a_k\sigma(\vw_k^{\T}\vx_j)\right),
\end{align*}
in order to compute $\mathrm{Cov}\left(e_{i}^N (\veta_N)_{k},  e_{j}^N (\veta_N)_{k}\right)$, we need to compute firstly   $\Exp\left[e_{i}^Ne_{j}^N(\veta_N)_{k}^2\right]$,
and since $\Exp\left[e_{i}^Ne_{j}^N(\veta_N)_{k}^2\right]$ consists of four parts, one of which is 
\begin{align*}
&\Exp\left[ \left(\sum_{k'=1, k'\neq k
 }^m (\veta_N)_{k'}a_{k'}\sigma(\vw_{k'}^{\T}\vx_i)-y_i\right)\left(\sum_{l=1, l\neq k
 }^m (\veta_N)_{l}a_{l}\sigma(\vw_{l}^{\T}\vx_j)-y_j\right)(\veta_N)_{k}^2\right]\\
=&\Exp\left[ \left(\sum_{k'=1, k'\neq k
 }^m (\veta_N)_{k'}a_{k'}\sigma(\vw_{k'}^{\T}\vx_i)-y_i\right)\left(\sum_{l=1, l\neq k
 }^m (\veta_N)_{l}a_{l}\sigma(\vw_{l}^{\T}\vx_j)-y_j\right)\right]\Exp\left[ (\veta_N)_{k}^2\right]\\
=&\frac{1}{p}\Bigg(\Exp \left[\sum_{k'=1, k'\neq k}^m  (\veta_N)_{k'}^2a_{k'}^2\sigma(\vw_{k'}^{\T}\vx_i)\sigma(\vw_{k'}^{\T}\vx_j)\right]+ \Exp \left[\sum_{k'\neq l ,\ k',l\neq k}(\veta_N)_{k'}(\veta_N)_{l} a_{k'}a_l\sigma(\vw_{k'}^{\T}\vx_i) \sigma(\vw_{l}^{\T}\vx_j)\right]\\
&~-y_i\Exp \left[\sum_{k'=1, k'\neq k
 }^m (\veta_N)_{k'}a_{k'}\sigma(\vw_{k'}^{\T}\vx_j)\right]-y_j\Exp \left[\sum_{k'=1, k'\neq k
 }^m (\veta_N)_{k'}a_{k'}\sigma(\vw_{k'}^{\T}\vx_i)\right]+y_iy_j\Bigg)\\
=&\frac{1}{p^2}\sum_{k'=1, k'\neq k}^ma_{k'}^2\sigma(\vw_{k'}^{\T}\vx_i)\sigma(\vw_{k'}^{\T}\vx_j)+\frac{1}{p}\sum_{k'\neq l ,\ k',l\neq k} a_{k'}a_l\sigma(\vw_{k'}^{\T}\vx_i) \sigma(\vw_{l}^{\T}\vx_j)\\
&~-\frac{y_i}{p}\sum_{k'=1, k'\neq k
 }^m a_{k'}\sigma(\vw_{k'}^{\T}\vx_j)-\frac{y_j}{p}\sum_{k'=1, k'\neq k
 }^m a_{k'}\sigma(\vw_{k'}^{\T}\vx_i)+\frac{y_iy_j}{p}\\
=&\frac{1}{p}\left[\sum_{k'=1, k'\neq k
 }^m a_{k'}\sigma(\vw_{k'}^{\T}\vx_i)-y_i \right]\left[\sum_{k'=1, k'\neq k
 }^m a_{k'}\sigma(\vw_{k'}^{\T}\vx_j)-y_j \right]\\
&~+\left(\frac{1}{p^2}-\frac{1}{p}\right)\left(\sum_{k'=1, k'\neq k}^ma_{k'}^2\sigma(\vw_{k'}^{\T}\vx_i)\sigma(\vw_{k'}^{\T}\vx_j)\right),
\end{align*}
and the second part reads
\begin{align*}
&\Exp\left[ (\veta_N)_{k}a_{k}\sigma(\vw_{k}^{\T}\vx_i)\left(\sum_{l=1, l\neq k
 }^m (\veta_N)_{l}a_{l}\sigma(\vw_{l}^{\T}\vx_j)-y_j\right)(\veta_N)_{k}^2\right]\\
=&\frac{a_k \sigma(\vw_{k}^{\T}\vx_i)}{p^2}\left(\sum_{k'=1, k'\neq k
 }^m a_{k'} \sigma(\vw_{k'}^{\T}\vx_j)-y_j\right),
\end{align*}
and by symmetry, the third part reads
\begin{align*}
&\Exp\left[ (\veta_N)_{k}a_{k}\sigma(\vw_{k}^{\T}\vx_j)\left(\sum_{l=1, l\neq k
 }^m (\veta_N)_{l}a_{l}\sigma(\vw_{l}^{\T}\vx_i)-y_i\right)(\veta_N)_{k}^2\right]\\
=&\frac{a_k \sigma(\vw_{k}^{\T}\vx_j)}{p^2}\left(\sum_{k'=1, k'\neq k
 }^m a_{k'} \sigma(\vw_{k'}^{\T}\vx_i)-y_i\right),
\end{align*}
and finally, the fourth part reads
\begin{align*}
&\Exp\left[ (\veta_N)_{k}a_{k}\sigma(\vw_{k}^{\T}\vx_i)(\veta_N)_{k}a_{k}\sigma(\vw_{k}^{\T}\vx_j)(\veta_N)_{k}^2\right]=\frac{1}{p^3}a_k^2\sigma(\vw_{k}^{\T}\vx_i)\sigma(\vw_{k}^{\T}\vx_j).
\end{align*}
To sum up,  
\begin{align*}
   \Exp\left[e_{i}^Ne_{j}^N(\veta_N)_{k}^2\right] &= \left(\frac{1}{p^2}-\frac{1}{p}\right)\left(\sum_{k'=1, k'\neq k}^ma_{k'}^2\sigma(\vw_{k'}^{\T}\vx_i)\sigma(\vw_{k'}^{\T}\vx_j)\right)\\
&~~+\frac{1}{p} e_{i,\backslash k}e_{j,\backslash k}+\frac{a_k \sigma(\vw_{k}^{\T}\vx_j)}{p^2}e_{i,\backslash k}+\frac{a_k \sigma(\vw_{k}^{\T}\vx_i)}{p^2}e_{j,\backslash k}\\
&~~+\frac{1}{p^3}a_k^2\sigma(\vw_{k}^{\T}\vx_i)\sigma(\vw_{k}^{\T}\vx_j),
\end{align*}
and 
\begin{align*}
&\Exp\left[ e_{i}^N(\veta_N)_{k}\right]\Exp\left[ e_{j}^N(\veta_N)_{k}\right]\\
=&\left(e_{i,\backslash k}+\frac{1}{p}a_k \sigma(\vw_{k}^{\T}\vx_i)\right)\left(e_{j,\backslash k}+\frac{1}{p}a_k \sigma(\vw_{k}^{\T}\vx_j)\right)\\
=& e_{i,\backslash k}e_{j,\backslash k}+\frac{a_k \sigma(\vw_{k}^{\T}\vx_j)}{p}e_{i,\backslash k}+\frac{a_k \sigma(\vw_{k}^{\T}\vx_i)}{p}e_{j,\backslash k}+\frac{1}{p^2}a_k^2\sigma(\vw_{k}^{\T}\vx_i)\sigma(\vw_{k}^{\T}\vx_j),
\end{align*}
hence
\begin{align*}
&\mathrm{Cov}\left(e_{i}^N (\veta_N)_{k},  e_{j}^N (\veta_N)_{k}\right)\\
=&    \Exp\left[e_{i}^Ne_{j}^N(\veta_N)_{k}^2\right]- \Exp\left[ e_{i}^N(\veta_N)_{k}\right ]\Exp\left[e_{i}^N(\veta_N)_{k}\right]\\
=&\left(\frac{1}{p^2}-\frac{1}{p}\right)\left(\sum_{k'=1, k'\neq k}^ma_{k'}^2\sigma(\vw_{k'}^{\T}\vx_i)\sigma(\vw_{k'}^{\T}\vx_j)\right)\\
&~+\left(\frac{1}{p}-1\right)e_{i,\backslash k}e_{j,\backslash k} +\left(\frac{1}{p^2}-\frac{1}{p}\right) {a_k \sigma(\vw_{k}^{\T}\vx_i)} e_{j,\backslash k} \\
&~+\left(\frac{1}{p^2}-\frac{1}{p}\right){a_k \sigma(\vw_{k}^{\T}\vx_j)} e_{i,\backslash k}   +\left(\frac{1}{p^3}-\frac{1}{p^2}\right)a_k^2\sigma(\vw_{k}^{\T}\vx_i)\sigma(\vw_{k}^{\T}\vx_j)\\
=&\left(\frac{1}{p}-1\right)\Exp\left( e_{i}^N(\veta_N)_{k}\right)\Exp\left( e_{j}^N(\veta_N)_{k}\right)+\left(\frac{1}{p^2}-\frac{1}{p}\right)\left(\sum_{k'=1, k'\neq k}^ma_{k'}^2\sigma(\vw_{k'}^{\T}\vx_i)\sigma(\vw_{k'}^{\T}\vx_j)\right),
\end{align*}
by summation over the indices $i$ and $j$, for each  $  k \in [m]$,  the covariance matrix reads:  
\begin{align*}
&\mSigma_{kk}(\vtheta_{N-1})=\mathrm{Cov}\left( \nabla_{\vq_k}R_{\fS}^\mathrm{drop}\left(\vtheta_{N-1}; \veta_{N}\right), \nabla_{\vq_k}R_{\fS}^\mathrm{drop}\left(\vtheta_{N-1}; \veta_{N}\right)\right) \\
=&\left(\frac{1}{p}-1\right)\left(\frac{1}{n}\sum_{i=1}^n\left( e_{i,\backslash k}+\frac{1}{p}a_{k}\sigma(\vw_{k}^{\T}\vx_i)\right)\nabla_{\vq_k}\left(a_k\sigma(\vw_k^{\T}\vx_i)\right) \right)\\
&~~~~~~~~~~~~~~~~~~~~~~~~~~~~~~{\otimes}\left(\frac{1}{n}\sum_{i=1}^n\left( e_{i,\backslash k}+\frac{1}{p}a_{k}\sigma(\vw_{k}^{\T}\vx_i)\right)\nabla_{\vq_k}\left(a_k\sigma(\vw_k^{\T}\vx_i)\right) \right)\\
&~+\left(\frac{1}{p^2}-\frac{1}{p}\right)\sum_{l=1, l\neq k}^m\left(\frac{1}{n}\sum_{i=1}^na_{l}\sigma(\vw_{l}^{\T}\vx_i)\nabla_{\vq_k}\left(a_k\sigma(\vw_k^{\T}\vx_i)\right)\right)\\
&~~~~~~~~~~~~~~~~~~~~~~~~~~~~~~{\otimes}\left(\frac{1}{n}\sum_{i=1}^na_{l}\sigma(\vw_{l}^{\T}\vx_i)\nabla_{\vq_k}\left(a_k\sigma(\vw_k^{\T}\vx_i)\right)\right).
\end{align*}
\subsection{Elements off the Diagonal}
In this part, we compute $\mSigma_{kr}$ for all $k, r\in[m]$, where $k\neq r$.
\begin{align*}
 \mSigma_{kr}(\vtheta_{N-1})&=\mathrm{Cov}\left( \nabla_{\vq_k}\RS^\mathrm{drop}\left(\vtheta_{N-1}; \veta_{N}\right), \nabla_{\vq_r}\RS^\mathrm{drop}\left(\vtheta_{N-1}; \veta_{N}\right)\right)   \\ 
&=\frac{1}{n^2}\sum_{i,j=1}^n\mathrm{Cov}\left(e_{i}^N (\veta_N)_{k},  e_{j}^N (\veta_N)_{r}\right)\nabla_{\vq_k}\left(a_k\sigma(\vw_k^{\T}\vx_i)\right) \otimes  \nabla_{\vq_r}\left(a_k\sigma(\vw_k^{\T}\vx_j)\right),
\end{align*}
in order to compute $\mathrm{Cov}\left(e_{i}^N (\veta_N)_{k},  e_{j}^N (\veta_N)_{r}\right)$, we need to compute firstly   $\Exp\left[e_{i}^Ne_{j}^N(\veta_N)_{k}(\veta_N)_{r} \right]$,
and since   $\Exp\left[e_{i}^Ne_{j}^N(\veta_N)_{k}(\veta_N)_{r} \right]$ consists of nine parts, one of which is  
\begin{align*}
&\Exp\left[ \left(\sum_{k'=1, k'\neq k, k'\neq r
 }^m (\veta_N)_{k'}a_{k'}\sigma(\vw_{k'}^{\T}\vx_i)-y_i\right)\left(\sum_{l=1, l\neq k,l\neq r
 }^m (\veta_N)_{l}a_{l}\sigma(\vw_{l}^{\T}\vx_j)-y_j\right)(\veta_N)_{k}(\veta_N)_{r}\right]\\
=&\Exp\left[ \left(\sum_{k'=1, k'\neq k, k'\neq r
 }^m (\veta_N)_{k'}a_{k'}\sigma(\vw_{k'}^{\T}\vx_i)-y_i\right)\left(\sum_{l=1, l\neq k, l\neq r
 }^m (\veta_N)_{l}a_{l}\sigma(\vw_{l}^{\T}\vx_j)-y_j\right)\right]\Exp\left[ (\veta_N)_{k}(\veta_N)_{r}\right]\\
=&\frac{1}{p}\sum_{k'=1, k'\neq k, k'\neq r}^ma_{k'}^2\sigma(\vw_{k'}^{\T}\vx_i)\sigma(\vw_{k'}^{\T}\vx_j)+\sum_{k'\neq l  \ \text{and} \ k',l\neq k, r} a_{k'}a_l\sigma(\vw_{k'}^{\T}\vx_i) \sigma(\vw_{l}^{\T}\vx_j)\\
&~-{y_i}\sum_{k'=1, k'\neq k, k'\neq r
 }^m a_{k'}\sigma(\vw_{k'}^{\T}\vx_j)-{y_j}\sum_{k'=1, k'\neq k, k'\neq r
 }^m a_{k'}\sigma(\vw_{k'}^{\T}\vx_i)+{y_iy_j}\\
=& \left[\sum_{k'=1, k'\neq k, k'\neq r
 }^m a_{k'}\sigma(\vw_{k'}^{\T}\vx_i)-y_i \right]\left[\sum_{k'=1, k'\neq k, k'\neq r
 }^m a_{k'}\sigma(\vw_{k'}^{\T}\vx_j)-y_j \right]\\
&~+\left(\frac{1}{p}-1\right)\left(\sum_{k'=1, k'\neq k, k'\neq r}^ma_{k'}^2\sigma(\vw_{k'}^{\T}\vx_i)\sigma(\vw_{k'}^{\T}\vx_j)\right)\\
&=e_{i,\backslash k, \backslash r}e_{j,\backslash k, \backslash r}+\left(\frac{1}{p}-1\right)\left(\sum_{k'=1, k'\neq k, k'\neq r}^ma_{k'}^2\sigma(\vw_{k'}^{\T}\vx_i)\sigma(\vw_{k'}^{\T}\vx_j)\right),
\end{align*}
and the second part reads
\begin{align*}
&\Exp\left[ \left(\sum_{k'=1, k'\neq k, k'\neq r
 }^m (\veta_N)_{k'}a_{k'}\sigma(\vw_{k'}^{\T}\vx_i)-y_i\right)(\veta_N)_{k}a_k\sigma(\vw_k^\T\vx_j)(\veta_N)_{k}(\veta_N)_{r}\right]\\
=&\Exp\left[\sum_{k'=1, k'\neq k, k'\neq r
 }^m (\veta_N)_{k'}a_{k'}\sigma(\vw_{k'}^{\T}\vx_i)-y_i 
 \right]a_k\sigma(\vw_k^\T\vx_j)\Exp\left[ (\veta_N)_{k}^2(\veta_N)_{r}\right] =\frac{a_k\sigma(\vw_k^\T\vx_j)}{p}e_{i,\backslash k, \backslash r} ,
\end{align*}
by similar reasoning and symmetry, the third part reads 
\begin{align*}
&\Exp\left[ \left(\sum_{k'=1, k'\neq k, k'\neq r
 }^m (\veta_N)_{k'}a_{k'}\sigma(\vw_{k'}^{\T}\vx_i)-y_i\right)(\veta_N)_{r}a_r\sigma(\vw_r^\T\vx_j)(\veta_N)_{k}(\veta_N)_{r}\right]\\
=&\Exp\left[\sum_{k'=1, k'\neq k, k'\neq r
 }^m (\veta_N)_{k'}a_{k'}\sigma(\vw_{k'}^{\T}\vx_i)-y_i 
 \right]a_r\sigma(\vw_r^\T\vx_j)\Exp\left[ (\veta_N)_{k}(\veta_N)_{r}^2\right] =\frac{a_r\sigma(\vw_r^\T\vx_j)}{p}e_{i,\backslash k, \backslash r} ,
\end{align*}
also by similar reasoning and symmetry, the  fourth part reads 
\begin{align*}
&\Exp\left[ \left(\sum_{k'=1, k'\neq k, k'\neq r
 }^m (\veta_N)_{k'}a_{k'}\sigma(\vw_{k'}^{\T}\vx_j)-y_j\right)(\veta_N)_{k}a_k\sigma(\vw_k^\T\vx_i)(\veta_N)_{k}(\veta_N)_{r}\right]\\
=&\Exp\left[\sum_{k'=1, k'\neq k, k'\neq r
 }^m (\veta_N)_{k'}a_{k'}\sigma(\vw_{k'}^{\T}\vx_j)-y_j 
 \right]a_k\sigma(\vw_k^\T\vx_i)\Exp\left[ (\veta_N)_{k}^2(\veta_N)_{r}\right] =\frac{a_k\sigma(\vw_k^\T\vx_i)}{p}e_{j,\backslash k, \backslash r} ,
\end{align*}
and  the  fifth part reads 
\begin{align*}
 \Exp\left[  (\veta_N)_{k}a_k\sigma(\vw_k^\T\vx_i)(\veta_N)_{k}a_k\sigma(\vw_k^\T\vx_j)(\veta_N)_{k}(\veta_N)_{r}\right]&=\Exp\left[(\veta_N)_{k}^3(\veta_N)_{r }a_k^2\sigma(\vw_k^\T\vx_i)\sigma(\vw_k^\T\vx_j)\right]\\
&=\frac{1}{p^2}a_k^2\sigma(\vw_k^\T\vx_i)\sigma(\vw_k^\T\vx_j),
\end{align*}
and  the sixth part reads 
\begin{align*}
& \Exp\left[  (\veta_N)_{k}a_k\sigma(\vw_k^\T\vx_i)(\veta_N)_{r}a_r\sigma(\vw_r^\T\vx_j)(\veta_N)_{k}(\veta_N)_{r}\right]\\
=&\Exp\left[(\veta_N)_{k}^2(\veta_N)_{r}^2a_k a_r\sigma(\vw_k^\T\vx_i)\sigma(\vw_r^\T\vx_j)\right] =\frac{1}{p^2}a_k a_r\sigma(\vw_k^\T\vx_i)\sigma(\vw_r^\T\vx_j),
\end{align*}
also by similar reasoning and symmetry, the  seventh  part reads 
\begin{align*}
&\Exp\left[ \left(\sum_{k'=1, k'\neq k, k'\neq r
 }^m (\veta_N)_{k'}a_{k'}\sigma(\vw_{k'}^{\T}\vx_j)-y_j\right)(\veta_N)_{r}a_r\sigma(\vw_r^\T\vx_i)(\veta_N)_{k}(\veta_N)_{r}\right]\\
=&\Exp\left[\sum_{k'=1, k'\neq k, k'\neq r
 }^m (\veta_N)_{k'}a_{k'}\sigma(\vw_{k'}^{\T}\vx_j)-y_j 
 \right]a_r\sigma(\vw_r^\T\vx_i)\Exp\left[ (\veta_N)_{k}(\veta_N)_{r}^2\right] =\frac{a_r\sigma(\vw_r^\T\vx_i)}{p}e_{j,\backslash k, \backslash r} ,
\end{align*}
and  the  eighth part reads 
\begin{align*}
 \Exp\left[  (\veta_N)_{r}a_r\sigma(\vw_r^\T\vx_i)(\veta_N)_{k}a_k\sigma(\vw_k^\T\vx_j)(\veta_N)_{k}(\veta_N)_{r}\right]&=\Exp\left[(\veta_N)_{k}^2(\veta_N)_{r }^2a_k a_r\sigma(\vw_k^\T\vx_i)\sigma(\vw_r^\T\vx_j)\right]\\
&=\frac{1}{p^2}a_k a_r\sigma(\vw_k^\T\vx_i)\sigma(\vw_r^\T\vx_j),
\end{align*}
and  the ninth part reads 
\begin{align*}
& \Exp\left[  (\veta_N)_{r}a_r\sigma(\vw_r^\T\vx_i)(\veta_N)_{r}a_r\sigma(\vw_r^\T\vx_j)(\veta_N)_{k}(\veta_N)_{r}\right]\\
=&\Exp\left[(\veta_N)_{k}(\veta_N)_{r}^3a_r^2\sigma(\vw_r^\T\vx_i)\sigma(\vw_r^\T\vx_j)\right] =\frac{1}{p^2}a_r^2\sigma(\vw_r^\T\vx_i)\sigma(\vw_r^\T\vx_j).
\end{align*}
To sum up,  
\begin{align*}
&\Exp\left[e_{i}^Ne_{j}^N(\veta_N)_{k}(\veta_N)_{r}\right] \\
=&e_{i,\backslash k, \backslash r}e_{j,\backslash k, \backslash r}+\left(\frac{1}{p}-1\right)\left(\sum_{k'=1, k'\neq k, k'\neq r}^ma_{k'}^2\sigma(\vw_{k'}^{\T}\vx_i)\sigma(\vw_{k'}^{\T}\vx_j)\right)+\frac{a_k\sigma(\vw_k^\T\vx_j)}{p}e_{i,\backslash k, \backslash r}\\
&~+\frac{a_r\sigma(\vw_r^\T\vx_j)}{p}e_{i,\backslash k, \backslash r}+\frac{a_k\sigma(\vw_k^\T\vx_i)}{p}e_{j,\backslash k, \backslash r}+\frac{1}{p^2}a_k^2\sigma(\vw_k^\T\vx_i)\sigma(\vw_k^\T\vx_j)\\
&~+\frac{1}{p^2}a_k a_r\sigma(\vw_k^\T\vx_i)\sigma(\vw_r^\T\vx_j)+\frac{a_r\sigma(\vw_r^\T\vx_i)}{p}e_{j,\backslash k, \backslash r}\\
&~+\frac{1}{p^2}a_k a_r\sigma(\vw_k^\T\vx_i)\sigma(\vw_r^\T\vx_j)+\frac{1}{p^2}a_r^2\sigma(\vw_r^\T\vx_i)\sigma(\vw_r^\T\vx_j),
\end{align*}
and 
\begin{align*}
&\Exp\left[e_{i}^N(\veta_N)_{k}\right]\Exp\left[e_{j}^N(\veta_N)_{r}\right] \\
=&\left(e_{i,\backslash k, \backslash r}+  a_r \sigma(\vw_r^\T\vx_i)+\frac{1}{p}a_k \sigma(\vw_k^\T\vx_i)\right)\left(e_{j,\backslash k, \backslash r}+  a_k \sigma(\vw_k^\T\vx_j)+\frac{1}{p}a_r \sigma(\vw_r^\T\vx_j)\right) \\
=&e_{i,\backslash k, \backslash r}e_{j,\backslash k, \backslash r}+ e_{i,\backslash k, \backslash r}a_k \sigma(\vw_k^\T\vx_j)+\frac{1}{p}e_{i,\backslash k, \backslash r}a_r \sigma(\vw_r^\T\vx_j)+a_r \sigma(\vw_r^\T\vx_i)e_{j,\backslash k, \backslash r}\\
&~+a_ra_k\sigma(\vw_r^\T\vx_i)\sigma(\vw_k^\T\vx_j)+\frac{1}{p}a_r^2\sigma(\vw_r^\T\vx_i)\sigma(\vw_r^\T\vx_j)+\frac{1}{p}a_k\sigma(\vw_k^\T\vx_i)e_{j,\backslash k, \backslash r}\\
&~+\frac{1}{p}a_k^2\sigma(\vw_k^\T\vx_i)\sigma(\vw_k^\T\vx_j)+\frac{1}{p^2}a_ra_k\sigma(\vw_k^\T\vx_i)\sigma(\vw_r^\T\vx_j),
\end{align*}
hence
\begin{align*}
&\mathrm{Cov}\left(e_{i}^N (\veta_N)_{k},  e_{j}^N (\veta_N)_{r}\right)\\
=&    \Exp\left[e_{i}^Ne_{j}^N(\veta_N)_{k}(\veta_N)_{r} \right]- \Exp\left[e_{i}^N(\veta_N)_{k}\right]\Exp\left[ e_{i}^N(\veta_N)_{r}\right]\\
=&\left(\frac{1}{p}-1\right)\left(\sum_{k'=1, k'\neq k, k'\neq r}^ma_{k'}^2\sigma(\vw_{k'}^{\T}\vx_i)\sigma(\vw_{k'}^{\T}\vx_j)\right)+\left(\frac{1}{p}-1\right)a_k\sigma(\vw_k^\T\vx_j)e_{i,\backslash k, \backslash r}\\
&~+\left(\frac{1}{p}-1\right)a_r\sigma(\vw_r^\T\vx_i)e_{j,\backslash k, \backslash r}+\left(\frac{1}{p^2}-\frac{1}{p}\right)a_r^2\sigma(\vw_r^\T\vx_i)\sigma(\vw_r^\T\vx_j)\\
&~+\left(\frac{1}{p^2}-\frac{1}{p}\right)a_k^2\sigma(\vw_k^\T\vx_i)\sigma(\vw_k^\T\vx_j)+\left(\frac{1}{p^2}-1\right)a_ra_k\sigma(\vw_r^\T\vx_i)\sigma(\vw_k^\T\vx_j),
\end{align*}
by summation over the indices $i$ and $j$, the covariance matrix reads
\begin{align*}
&\mSigma_{kr}(\vtheta_{N-1}) = \mathrm{Cov}\left( \nabla_{\vq_k}R_{\fS}^\mathrm{drop}\left(\vtheta_{N-1}; \veta_{N}\right), \nabla_{\vq_r}R_{\fS}^\mathrm{drop}\left(\vtheta_{N-1}; \veta_{N}\right)\right) \\
=&\left(\frac{1}{p}-1\right)\left(\frac{1}{n}\sum_{i=1}^n\left(e_{i,\backslash k, \backslash r}+\frac{1}{p}a_k\sigma(\vw_k^\T\vx_i)+\frac{1}{p}a_r\sigma(\vw_r^\T\vx_i)\right)\nabla_{\vq_k}\left(a_k\sigma(\vw_k^{\T}\vx_i)\right)\right)\\
&~~~~~~~~~~~~~~~~{\otimes}\left(\frac{1}{n}\sum_{i=1}^na_k\sigma(\vw_k^\T\vx_i)\nabla_{\vq_r}\left(a_r\sigma(\vw_r^{\T}\vx_i)\right)\right)\\
&+\left(\frac{1}{p}-1\right)\left(\frac{1}{n}\sum_{i=1}^na_r\sigma(\vw_r^\T\vx_i)\nabla_{\vq_k}\left(a_k\sigma(\vw_k^{\T}\vx_i)\right)\right) \\
&~~~~~~~~~~~~~~~~{\otimes}\left(\frac{1}{n}\sum_{i=1}^n\left(e_{i,\backslash k, \backslash r}+a_k\sigma(\vw_k^\T\vx_i)+\frac{1}{p}a_r\sigma(\vw_r^\T\vx_i)\right)\nabla_{\vq_r}\left(a_r\sigma(\vw_r^{\T}\vx_i)\right)\right),
\end{align*}

\newpage

\section{The structural similarity between Hessian and covariance}
We can derive the Hessian of the loss landscape in the expectation sense with respect to the dropout noise $\veta$ and the covariance matrix of dropout noise under intuitive approximations. We first show our assumptions as follows:

\begin{assump}\label{ass:1}
 The NN piece-wise linear activation.
\end{assump}

\begin{assump}\label{ass:2}
The parameters of NN's output layer are fixed during training.
\end{assump}

\begin{assump}\label{ass:3}
    We study the loss landscape after training reaches a stable stage, i.e., the loss function in the sense of expectation is small enough,
    \begin{equation*}
        \Exp_{\veta}\nabla_{\vtheta} R_{S}^{\mathrm{drop}}(\vtheta;\veta)\approx \mzero.
    \end{equation*}
\end{assump}

\textbf{Hessian matrix with dropout regularization }Based on the Assumption \ref{ass:1}, \ref{ass:2}, the Hessian matrix of the loss function with respect to $f_{\vtheta,\veta}^{\mathrm{drop}}(\vx)$ can be written in the mean sense as:
\begin{equation*}
        \mH(\vtheta)\approx\frac{1}{n} \sum_{i=1}^{n}\left[ \nabla_{\vtheta} f_{\vtheta}\left(\vx_{i}\right){\otimes} \nabla_{\vtheta} f_{\vtheta}\left(\vx_{i}\right)+ \frac{1-p}{p} \sum_{r=1}^{m} \nabla_{\vq_r}\left(a_r\sigma(\vw_r^{\T}\vx_i)\right)   {\otimes}\nabla_{\vq_r}\left(a_r\sigma(\vw_r^{\T}\vx_i)\right) \right],
\end{equation*}
where  $\mH(\vtheta):=\nabla^2_{\vtheta}L_{\fS}(\vtheta)$.


\label{thm:1}

\begin{proof}

We first compute the Hessian matrix after taking expectations with respect to the dropout variable, 
\begin{equation}
    \nabla^2_{\vtheta}L_{\fS}(\vtheta)=\nabla_{\vtheta}^{2}R_{S}(\vtheta)+\frac{1-p}{2np}\sum_{i=1}^{n}\sum_{r=1}^{m}\nabla_{\vq_r}^{2}\left(a_r\sigma(\vw_r^{\T}\vx_i)\right)  ^2.
    \label{equ:hess}
\end{equation}

The first and second terms on the RHS of the Eq. (\ref{equ:hess}) are as follows, 

\begin{equation*}
    \nabla_{\vtheta}^{2}R_{S}(\vtheta)=\frac{1}{n} \sum_{i=1}^{n}\left( \nabla_{\vtheta} f_{\vtheta}\left(\vx_{i}\right) \otimes \nabla_{\vtheta} f_{\vtheta}\left(\vx_{i}\right)+(f_{\vtheta}\left(\vx_{i}\right)-y_{i})\cdot \nabla_{\vtheta}^{2}f_{\vtheta}\left(\vx_{i}\right) \right)
\end{equation*}
\begin{equation*}
\begin{aligned}
        &\frac{1-p}{2np}\sum_{i=1}^{n}\sum_{r=1}^{m}\nabla_{\vq_r}^{2}\left(a_r\sigma(\vw_r^{\T}\vx_i)\right)  ^2 \\
        ~&= \frac{1-p}{np} \sum_{i=1}^{n} \sum_{r=1}^{m}\left( \nabla_{\vq_r}\left(a_r\sigma(\vw_r^{\T}\vx_i)\right) \otimes \nabla_{\vq_r}\left(a_r\sigma(\vw_r^{\T}\vx_i)\right)+\left(a_r\sigma(\vw_r^{\T}\vx_i)\right) \cdot\nabla_{\vq_r}^{2}\left(a_r\sigma(\vw_r^{\T}\vx_i)\right)  ^2\right). 
\end{aligned}
\end{equation*}
Note that for linear activate function, $\nabla_{\vtheta}^{2}f_{\vtheta}\left(\vx_{i}\right)=\nabla_{\vq_r}^{2}\left(a_r\sigma(\vw_r^{\T}\vx_i)\right)  ^2=\mzero,\ a.e. \ \forall i \in [n], \forall r \in [m]$, we have

\begin{equation*}
    \nabla_{\vtheta}^{2}R_{S}(\vtheta)=\frac{1}{n} \sum_{i=1}^{n} \nabla_{\vtheta} f_{\vtheta}\left(\vx_{i}\right) \otimes \nabla_{\vtheta} f_{\vtheta}\left(\vx_{i}\right)
\end{equation*}
\begin{equation*}
    \frac{1-p}{2np}\sum_{i=1}^{n}\sum_{r=1}^{m}\nabla_{\vq_r}^{2}\left(a_r\sigma(\vw_r^{\T}\vx_i)\right)  ^2 = \frac{1-p}{np} \sum_{i=1}^{n} \sum_{r=1}^{m} \nabla_{\vq_r}\left(a_r\sigma(\vw_r^{\T}\vx_i)\right) \otimes  \nabla_{\vq_r}\left(a_r\sigma(\vw_r^{\T}\vx_i)\right). 
\end{equation*}

Thus the Eq. (\ref{equ:hess}) can be rewritten as
\begin{equation*}
    \mH(\vtheta)=\frac{1}{n} \sum_{i=1}^{n}\left( \nabla_{\vtheta} f_{\vtheta}\left(\vx_{i}\right) \otimes \nabla_{\vtheta} f_{\vtheta}\left(\vx_{i}\right)+ \frac{1-p}{p} \sum_{r=1}^{m} \nabla_{\vq_r}\left(a_r\sigma(\vw_r^{\T}\vx_i)\right) \otimes  \nabla_{\vq_r}\left(a_r\sigma(\vw_r^{\T}\vx_i)\right) \right).
\end{equation*}
\end{proof}

\textbf{Covariance matrix with dropout regularization} Based on the Assumption \ref{ass:3}, the covariance matrix of the loss function under the randomness of dropout variable $\veta$ and data $\vx$ can be written as:
\begin{equation*}
    \begin{aligned}
    \mSigma(\vtheta)&\approx\frac{1}{n}\sum_{i=1}^{n} \left[ l_{i,1} \nabla_{\vtheta}f_{\vtheta}(\vx_i){\otimes}\nabla_{\vtheta}f_{\vtheta}(\vx_i)+ l_{i,2} \frac{1-p}{p}\sum_{r=1}^{m} \nabla_{\vq_r}\left(a_r\sigma(\vw_r^{\T}\vx_i)\right)   {\otimes}\nabla_{\vq_r}\left(a_r\sigma(\vw_r^{\T}\vx_i)\right)\right],
    \end{aligned}
\end{equation*}
where $l_{i,1}:=  (e_{i})^2+\frac{1-p}{p}\sum_{r=1}^m a_r^2\sigma(\vw_r^\T\vx_i)^2$, $l_{i,2}:=  (e_{i})^2 $ .

\label{thm:2}

\begin{proof}
For simplicity, we approximate the loss function through Taylor expansion, which is also used in \cite{wei2020implicit},  
\begin{equation*}
    \ell(f_{\vtheta}(\vx_i;\veta),y_{i}) \approx \ell(f_{\vtheta}(\vx_i),y_{i})+(f_{\vtheta}(\vx_i)-y_{i})\sum_{r=1}^{m}a_r (\veta-\vone)_{r} \sigma(\vw_r^{\T}\vx_i),
\end{equation*}
where $\ell(f_{\vtheta}(\vx_i;\veta),y_{i})=\frac{1}{2}\left(f_{\vtheta}(\vx_i;\veta)-y_i\right)^2$, $\ell(f_{\vtheta}(\vx_i),y_{i})=\frac{1}{2}\left(f_{\vtheta}(\vx_i)-y_i\right)^2$. The covariance matrix under dropout regularization is
\begin{equation*}
\begin{aligned}
    \mSigma(\vtheta)
    &\approx\frac{1}{n}\sum_{i=1}^{n}\Exp_{\veta}\left(\nabla_{\vtheta}\ell(f_{\vtheta}(\vx_i;\veta),y_{i}) \otimes \nabla_{\vtheta}\ell(f_{\vtheta}(\vx_i;\veta),y_{i}) \right) - \nabla_{\vtheta} \Exp_{\veta}R_{S}^{\mathrm{drop}}(\vtheta;\veta) \otimes \nabla_{\vtheta} \Exp_{\veta}R_{S}^{\mathrm{drop}}(\vtheta;\veta)\\
    &\approx\frac{1}{n}\sum_{i=1}^{n}\Exp_{\veta}\left(\nabla_{\vtheta}\ell(f_{\vtheta}(\vx_i;\veta),y_{i}) \otimes \nabla_{\vtheta}\ell(f_{\vtheta}(\vx_i;\veta),y_{i}) \right) .
\end{aligned}
\end{equation*}
Combining the properties of the dropout variable $\veta$, we have, 
\begin{equation}
\begin{aligned}
    \mSigma(\vtheta)&\approx\frac{1}{n}\sum_{i=1}^{n}\nabla_{\vtheta}\ell(f_{\vtheta}(\vx_i),y_{i}) \otimes \nabla_{\vtheta}\ell(f_{\vtheta}(\vx_i),y_{i}) \\
    &+\frac{1}{n}\sum_{i=1}^{n} \Exp_{\veta} \left(\sum_{r=1}^{m}(\veta-\vone)_{r}\nabla_{\vq_r}(a_r\sigma(\vw_r^{\T}\vx_i)e_{i}) \otimes \sum_{r=1}^{m}(\veta-\vone)_{r}\nabla_{\vq_r}(a_r\sigma(\vw_r^{\T}\vx_i)e_{i}) \right) \\
    &=\frac{1}{n}\sum_{i=1}^{n}\left(\nabla_{\vtheta}\ell(f_{\vtheta}(\vx_i),y_{i}) \otimes \nabla_{\vtheta}\ell(f_{\vtheta}(\vx_i),y_{i})+\frac{1-p}{p} \sum_{r=1}^{m}\nabla_{\vq_r}(a_r\sigma(\vw_r^{\T}\vx_i)e_{i})\otimes \nabla_{\vq_r}(a_r\sigma(\vw_r^{\T}\vx_i)e_{i}) \right)\\
    &:= \frac{1}{n}\sum_{i=1}^{n}\left(\mSigma_{1}(\vx_{i},y_{i})+\frac{1-p}{p}\mSigma_{2}(\vx_{i},y_{i})\right).
\end{aligned}    
\label{equ:cov}
\end{equation}

We calculate the two terms on the RHS of the Eq. (\ref{equ:cov}) separately:

\begin{equation*}
    \mSigma_{1}(\vx_{i},y_{i})=(e_i)^2 \cdot \nabla_{\vtheta}f_{\vtheta}(\vx_i)\otimes \nabla_{\vtheta}f_{\vtheta}(\vx_i),
\end{equation*}

\begin{equation*}
\begin{aligned}
    \mSigma_{2}(\vx_{i},y_{i})&=(e_{i})^2\sum_{r=1}^{m}\nabla_{\vq_r}(a_r\sigma(\vw_r^{\T}\vx_i))\otimes \nabla_{\vq_r}(a_r\sigma(\vw_r^{\T}\vx_i))+\nabla_{\vtheta}f_{\vtheta}(\vx_i)\otimes \nabla_{\vtheta}f_{\vtheta}(\vx_i)\sum_{r=1}^{m}(a_r\sigma(\vw_r^{\T}\vx_i))^2\\
    &+2\sum_{r=1}^{m}e_{i}a_r\sigma(\vw_r^{\T}\vx_i)\cdot\nabla_{\vtheta}e_{i} \otimes
    \nabla_{\vq_r}(a_r\sigma(\vw_r^{\T}\vx_i))\\
    &=(e_{i})^2\sum_{r=1}^{m}\nabla_{\vq_r}(a_r\sigma(\vw_r^{\T}\vx_i))\otimes \nabla_{\vq_r}(a_r\sigma(\vw_r^{\T}\vx_i))+\nabla_{\vtheta}f_{\vtheta}(\vx_i)\otimes \nabla_{\vtheta}f_{\vtheta}(\vx_i)\sum_{r=1}^{m}(a_r\sigma(\vw_r^{\T}\vx_i))^2\\
    &+\frac{1}{2}\sum_{r=1}^{m}\nabla_{\vtheta}(e_{i})^2 \otimes \nabla_{\vq_r}(a_r\sigma(\vw_r^{\T}\vx_i))^2.
\end{aligned}
\end{equation*}
Under the assumption that $\nabla_{\vtheta}(e_{i})^2=2\cdot \nabla_{\vtheta}\ell(f_{\vtheta}(\vx_i),y_{i})=\mzero$, $\forall i \in [n]$, we have 
\begin{equation*}
    \mSigma_{2}(\vx_{i},y_{i})=(e_{i})^2\sum_{r=1}^{m}\nabla_{\vq_r}(a_r\sigma(\vw_r^{\T}\vx_i))\otimes \nabla_{\vq_r}(a_r\sigma(\vw_r^{\T}\vx_i))+\nabla_{\vtheta}f_{\vtheta}(\vx_i)\otimes \nabla_{\vtheta}f_{\vtheta}(\vx_i)\sum_{r=1}^{m}(a_r\sigma(\vw_r^{\T}\vx_i))^2.
\end{equation*}
Thus the Eq. (\ref{equ:cov}) can be rewritten as
\begin{equation*}
\begin{aligned}
    \mSigma(\vtheta)&=\frac{1}{n}\sum_{i=1}^{n}\nabla_{\vtheta}f_{\vtheta}(\vx_i)\otimes \nabla_{\vtheta}f_{\vtheta}(\vx_i)\left((e_i)^2+ \frac{1-p}{p}\sum_{r=1}^{m}(a_r\sigma(\vw_r^{\T}\vx_i))^2\right)\\
    &+\frac{1-p}{np}\sum_{i=1}^{n}\sum_{r=1}^{m}(e_{i})^2\cdot\nabla_{\vq_r}(a_r\sigma(\vw_r^{\T}\vx_i))\otimes \nabla_{\vq_r}(a_r\sigma(\vw_r^{\T}\vx_i)).
\end{aligned}
\end{equation*}
Note that 
\begin{equation*}
    (e_i)^2+ \frac{1-p}{p}\sum_{r=1}^{m}(a_r\sigma(\vw_r^{\T}\vx_i))^2=\Exp_{\veta}2\ell(f_{\vtheta}(\vx_i;\veta),y_{i}), 
\end{equation*}
we have
\begin{equation*}
    \begin{aligned}
    \mSigma(\vtheta)&=\frac{2}{n}\sum_{i=1}^{n}\Exp_{\veta}\ell(f_{\vtheta}(\vx_i;\veta),y_{i})\cdot \nabla_{\vtheta}f_{\vtheta}(\vx_i)\otimes \nabla_{\vtheta}f_{\vtheta}(\vx_i)\\
    &+\frac{2(1-p)}{np}\sum_{i=1}^{n}\sum_{r=1}^{m}(\ell(f_{\vtheta}(\vx_i),y_{i}))\cdot\nabla_{\vq_r}(a_r\sigma(\vw_r^{\T}\vx_i))\otimes \nabla_{\vq_r}(a_r\sigma(\vw_r^{\T}\vx_i)).
    \end{aligned}
\end{equation*}
\end{proof}


\end{document}